%% file: 17-313.tex
\documentclass[twoside,11pt]{article}

\usepackage{jmlr2e}

\usepackage{epsf}
\usepackage{epsfig}
\usepackage{amsmath}
\usepackage{subfigure}
\usepackage{amssymb}
\usepackage{algorithm}
\usepackage{multirow}
\usepackage{multicol}
\usepackage{array}
\usepackage{enumitem}

\usepackage[pagebackref=true,breaklinks=true,letterpaper=true,colorlinks,bookmarks=yes]{hyperref}

\def\A{{\bf A}}
\def\a{{\bf a}}
\def\B{{\bf B}}
\def\bb{{\bf b}}
\def\C{{\bf C}}

\def\D{{\bf D}}

\def\G{{\bf G}}
\def\H{{\bf H}}
\def\I{{\bf I}}

\def\M{{\bf M}}

\def\N{{\bf N}}

\def\PP{{\bf P}}

\def\S{{\bf S}}
\def\s{{\bf s}}

\def\U{{\bf U}}
\def\u{{\bf u}}
\def\V{{\bf V}}

\def\W{{\bf W}}
\def\w{{\bf w}}
\def\X{{\bf X}}
\def\x{{\bf x}}
\def\Y{{\bf Y}}
\def\y{{\bf y}}

\def\0{{\bf 0}}
\def\1{{\bf 1}}

\def\NM{{\mathcal N}}
\def\OM{{\mathcal O}}

\def\RB{{\mathbb R}}
\def\EB{{\mathbb E}}
\def\PB{{\mathbb P}}

\def\varepsi{\mbox{\boldmath$\varepsilon$\unboldmath}}

\def\Ps{\mbox{\boldmath$\Psi$\unboldmath}}

\def\Si{\mbox{\boldmath$\Sigma$\unboldmath}}
\def\si{\mbox{\boldmath$\sigma$\unboldmath}}

\def\De{\mbox{\boldmath$\Delta$\unboldmath}}

\def\Xii{\mbox{\boldmath$\Xi$\unboldmath}}

\def\argmin{\mathop{\rm argmin}}

\def\bias{\mathsf{bias}}
\def\var{\mathsf{var}}
\def\nnz{\mathrm{nnz}}
\def\poly{\mathrm{poly}}

\def\tr{\mathrm{tr}}
\def\rk{\mathrm{rank}}

\def\diag{\mathsf{diag}}

\newtheorem{Remark}{Remark}

\newtheorem{assumption}{Assumption}

\makeatletter
\newcommand{\mylabel}[2]{#2\def\@currentlabel{#2}\label{#1}}
\makeatother

\jmlrheading{19}{2018}{1-50}{6/17}{5/18}{17-313}{Shusen~Wang, Alex~Gittens, and Michael~W.~Mahoney}

\ShortHeadings{Sketched Ridge Regression}{Wang, Gittens, and Mahoney}
\firstpageno{1}

\begin{document}

\title{Sketched Ridge Regression: Optimization Perspective,\\ Statistical Perspective, and Model Averaging}

\author{\name Shusen Wang \email wssatzju@gmail.com \\
       \addr International Computer Science Institute and Department of Statistics \\
       University of California at Berkeley\\
       Berkeley, CA 94720, USA
       \AND
       \name Alex Gittens \email gittea@rpi.edu \\
       \addr Computer Science Department \\
       Rensselaer Polytechnic Institute \\
       Troy, NY 12180, USA 
       \AND
       \name Michael W.\ Mahoney \email mmahoney@stat.berkeley.edu \\
       \addr International Computer Science Institute and Department of Statistics \\
       University of California at Berkeley \\
       Berkeley, CA 94720, USA }

\editor{Mehryar Mohri}

\maketitle

\begin{abstract}
We address the statistical and optimization impacts of the classical sketch and Hessian sketch used to approximately solve the Matrix Ridge Regression (MRR) problem. Prior research has quantified the effects of classical sketch on the strictly simpler least squares regression (LSR) problem. We establish that classical sketch has a similar effect upon the optimization properties of MRR as it does on those of LSR: namely, it recovers nearly optimal solutions. By contrast, Hessian sketch does not have this guarantee; instead, the approximation error is governed by a subtle interplay between the ``mass'' in the responses and the optimal objective value. 

For both types of approximation, the regularization in the sketched MRR problem results in significantly different statistical properties from those of the sketched LSR problem. In particular, there is a bias-variance trade-off in sketched MRR that is not present in sketched LSR. We provide upper and lower bounds on the bias and variance of sketched MRR; these bounds show that classical sketch significantly increases the variance, while Hessian sketch significantly increases the bias. Empirically, sketched MRR solutions can have risks that are higher by an order-of-magnitude than those of the optimal MRR solutions.

We establish theoretically and empirically that model averaging greatly decreases the gap between the risks of the true and sketched solutions to the MRR problem. Thus, in parallel or distributed settings, sketching combined with model averaging is a powerful technique that quickly obtains near-optimal solutions to the MRR problem while greatly mitigating the increased statistical risk incurred by sketching. 
\end{abstract}

\begin{keywords}
Randomized Linear Algebra, Matrix Sketching, Ridge Regression
\end{keywords}

\section{Introduction}
\label{sec:introduction}

Regression is one of the most fundamental problems in machine learning.  The
simplest and most thoroughly studied regression model is least squares
regression (LSR).  Given features $\X = [\x_1^T; \ldots , \x_n^T] \in
\RB^{n\times d}$ and responses $\y = [y_1, \ldots , y_n]^T \in \RB^{n}$, the LSR
problem $\min_{\w} \|\X \w - \y \|_2^2$ can be solved in $\OM (nd^2 )$ time
using the QR decomposition or in $\OM (n d t)$ time using accelerated gradient
descent algorithms.  Here, $t$ is the number of iterations, which depends on
the initialization, the condition number of $\X^T \X$, and the stopping criterion.

This paper considers the $n\gg d$ problem, where there is much redundancy in $\X$.  
Matrix sketching, as used in the paradigm of Randomized Linear Algebra (RLA) \citep{mahoney2011ramdomized,woodruff2014sketching,drineas2016randnla}, aims to reduce the
size of $\X$ while limiting information loss; the sketching operation can consist of 
sampling a subset of the rows of $\X$, or forming linear combinations
of the rows of $\X$. Either operation is modeled mathematically by multiplication with a sketching matrix $\S$ to form
the sketch $\S^T \X$. The sketching matrix $\S \in \RB^{n \times s}$ satisfies $d < s \ll n$
so that $\S^T \X$ generically has the same rank but much fewer rows as $\X$. 
Sketching has been used to speed up LSR \citep{drineas2006sampling,drineas2011faster,clarkson2013low,meng2013low,nelson2013osnap}
by solving the sketched LSR problem $\min_{\w} \|\S^T\X \w - \S^T \y\|_2^2$ instead of the
original LSR problem.  Solving sketched LSR costs either $\OM (s d^2 +
T_{{s}})$ time using the QR decomposition or $\OM (s d t + T_{{s}})$ time using
accelerated gradient descent algorithms, where $t$ is as defined
previously\footnote{The condition number of $\X^T \S \S^T \X$
	is very close to that of $\X^T \X$,
	and thus the number of iterations $t$ is almost unchanged.} and
$T_{{s}}$ is the time cost of sketching.  For example, $T_{\textrm{s}} = \OM
(nd \log s)$ when $\S$ is the subsampled randomized Hadamard
transform~\citep{drineas2011faster}, and $T_{{s}} = \OM (nd)$ when $\S$ is a
CountSketch matrix~\citep{clarkson2013low}.

There has been much work in RLA on analyzing the quality of sketched LSR with
different sketching methods and different objectives; see the reviews
\citep{mahoney2011ramdomized,woodruff2014sketching,drineas2016randnla} and the references therein.
The concept of sketched LSR originated in the theoretical computer science literature, e.g.,
~\citet{drineas2006sampling,drineas2011faster}, where the 
behavior of sketched LSR was first studied from an optimization perspective.  Let $\w^\star$ be the
optimal LSR solution and $\tilde{\w}$ be the solution to sketched LSR.
This line of work established that if $s = \OM (d/\epsilon + \poly (d))$, then
the objective value $\|\X \tilde{\w} - \y \big\|_2^2$ is at most (1+$\epsilon$) 
times greater than $\|\X {\w^\star} - \y \big\|_2^2$.
These works also bounded $\| \tilde{\w} - \w^\star \|_2^2$ in terms of
the difference in the objective function values at $\tilde{\w}$ and $\w^\star$ and the condition number of $\X^T \X$.

A more recent line of work has studied sketched LSR from a statistical perspective: 
~\cite{ma2014statistical,raskutti2015statistical,pilanci2015iterative,wang2016computationally}
considered statistical properties of sketched LSR such as the bias and variance.
In particular,~\citet{pilanci2015iterative} showed that the solutions to sketched LSR have
much higher variance than the optimal solutions.

Both of these perspectives are important and of practical interest. 
The optimization perspective is relevant when the approximate solution is used to initialize an (expensive) iterative optimization algorithm;
the statistical perspective is relevant in machine learning and statistics applications where the approximate solution is directly used in lieu of the optimal solution.

In practice, regularized regression, e.g., ridge regression and LASSO,
exhibit more attractive bias-variance trade-offs and generalization errors than
vanilla LSR. Furthermore, the matrix generalization of LSR, where multiple
responses are to be predicted, is often more useful than LSR. 
However, the properties of sketched regularized matrix
regression are largely unknown.  Hence, we consider the question: {\it how does our understanding of
	the optimization and statistical properties of sketched LSR generalize to
  sketched regularized regression problems?} We answer this question for
the sketched matrix ridge regression (MRR) problem. 

Recall that $\X$ is ${n \times d}$. Let $\Y \in \RB^{n \times m}$ denote
a matrix of corresponding responses.  We study the MRR problem
\begin{eqnarray} \label{eq:def_f}
\min_\W \; \Big\{ f (\W )
\; \triangleq \; \tfrac{1}{n} \big\| \X \W - \Y \big\|_F^2 + \gamma \|\W\|_F^2 \Big\},
\end{eqnarray}
which has optimal solution
\begin{eqnarray} \label{eq:def_W_optimal}
\W^\star
& = & (\X^T \X + n \gamma \I_d )^\dag \X^T \Y.
\end{eqnarray}
Here, $(\cdot)^\dagger$ denotes the Moore-Penrose inversion operation.
LSR is a special case of MRR, with $m=1$ and $\gamma = 0$.  The optimal
solution $\W^\star$ can be obtained in $\OM (n d^2 + nmd)$ time using a QR
decomposition of $\X$.  Sketching can be applied to MRR in two ways:
\begin{eqnarray}
{\W}^{\textrm{c}}
& = & (\X^T \S \S^T \X + n \gamma \I_d )^\dag (\X^T \S \S^T \Y) ,  \label{eq:def_W_tilde} \\
{\W}^{\textrm{h}}
& = & (\X^T \S \S^T \X + n \gamma \I_d )^\dag \X^T \Y . \label{eq:def_W_hat}
\end{eqnarray}
Following the convention of \citet{pilanci2015iterative,wang2016sketching},
we call ${\W}^{\textrm{c}}$ the {\bf classical sketch}
and ${\W}^{\textrm{h}}$ the {\bf Hessian sketch}.
Table~\ref{tab:timecost} lists the time costs of the three solutions to MRR.

\begin{table}[!h]\setlength{\tabcolsep}{0.3pt}
	\caption{The time cost of the solutions to MRR.
		Here $T_s (\X)$ and $T_s (\Y)$ denote the time cost of forming
		the sketches $\S^T \X \in \RB^{s\times d}$ and $\S^T \Y \in \RB^{s\times m}$.}
	\label{tab:timecost}
	\begin{center}
		\begin{small}
			\begin{tabular}{c c c}
				\hline
				~~~{\bf Solution}~~~&~~~{\bf Definition}~~~&~~~{\bf Time Complexity}~~~\\
				\hline
				~Optimal Solution~ & ~~~\eqref{eq:def_W_optimal}~~~
				& ~~~$\OM (n d^2 + nmd)$~~~ \\
				~Classical Sketch~ & ~~~\eqref{eq:def_W_tilde}~~~
				& ~~~$\OM (s d^2 + smd) + T_s (\X) + T_s (\Y)$~~~ \\
				~Hessian Sketch~ & ~~~\eqref{eq:def_W_hat}~~~
				& ~~~$\OM (s d^2 + nmd)+ T_s (\X)$~~~ \\
				\hline
			\end{tabular}
		\end{small}
		\vspace{-5mm}
	\end{center}
\end{table}

\subsection{Main Results and Contributions} \label{sec:intro:main}

We summarize all of our upper bounds in Table~\ref{tab:main}.
Our optimization analysis bounds the gap between the objective function values at the 
sketched and optimal solutions,
while our statistical analysis quantifies the behavior of the bias and variance of the sketched solutions
relative to those of the true solutions.

\begin{table}[!h]\setlength{\tabcolsep}{0.3pt}
	\def\arraystretch{1.4}
  \caption{A summary of our main results. In the table, $\W$ is the solution of classical/Hessian sketch 
    with or without model averaging (mod.\ avg.);
		$\W^\star$ is the optimal solution;
    $g$ is the number of models used in model averaging; and
		$\beta = \frac{\|\X\|_2^2 }{\|\X\|_2^2 + n \gamma } \leq 1$,
		where $\gamma$ is the regularization parameter.
    For conciseness, we take the sketching matrix $\S \in \RB^{n\times s}$ to correspond to 
		Gaussian projection, SRHT, or shrinkage leverage score sampling.
		Similar but more complex expressions hold for 
    uniform sampling (with or without model averaging)
		and CountSketch (only without model averaging.)
    All the bounds hold with constant probability.
    The notation $\tilde{\OM}$ conceals logarithmic factors. 
	}
	\label{tab:main}
	\begin{center}
		\begin{footnotesize}
			\begin{tabular}{c | c c | c c}
				\hline
				& \multicolumn{2}{c|}{\bf Classical Sketch}
				& \multicolumn{2}{c}{\bf Hessian Sketch}  \\ 
				& ~{\bf \scriptsize w/o mod.\ avg.}~~ & ~~{\bf \scriptsize w/ mod.\ avg.}~
				& ~~{\bf \scriptsize w/o mod.\ avg.}~~ & ~~{\bf \scriptsize w/ mod.\ avg.}~~ \\ 
				\hline
				$s = $
				& \multicolumn{2}{c|}{ $\tilde\OM ({ d} / {\epsilon } )$ }
				& \multicolumn{2}{c}{ $\tilde\OM ( {d} / {\epsilon })$ } \\
				{\tiny $  f (\W ) - f (\W^\star)  \leq $}
				& $\beta \epsilon f (\W^\star) $ 
				& $\beta (\tfrac{\epsilon}{g} + \beta^2 \epsilon^2 ) f (\W^\star)$
				& ~~~~$\beta^2 \epsilon \big[ \tfrac{\|\Y\|_F^2}{n} - f (\W^\star) \big] $ ~~~~
				& ~$\beta^2(\tfrac{\epsilon}{g} + {\epsilon^2 } ) 
				\big[ \tfrac{\|\Y\|_F^2}{n} - f (\W^\star) \big] $~ \\
				Theorems
				& Theorem~\ref{thm:optimization:classical}
				& Theorem~\ref{thm:optimization:classical_avg}
				& Theorem~\ref{thm:optimization:hessian}
				& Theorem~\ref{thm:optimization:hessian_avg}\\
				\hline
				$s = $
				& \multicolumn{2}{c|}{ $\tilde\OM ({ d} / {\epsilon^2 } )$ }
				& \multicolumn{2}{c}{ $\tilde\OM ( { d} / {\epsilon^2 })$ } \\
				$\frac{\bias (\W) }{ \bias (\W^\star )} \leq $
				& $1+\epsilon$ 
				& $1+\epsilon$
				& $(1+\epsilon) (1 + \frac{\epsilon \|\X\|_2^2}{n\gamma} )$ 
				& $1 + \epsilon + \big( \frac{\epsilon }{ \sqrt{g} } + \epsilon^2  \big) \frac{\| \X \|_2^2 }{n \gamma }$ \\
				$\frac{\var (\W) }{ \var (\W^\star )} \leq $
				&$(1+\epsilon) \frac{n}{s}$ 
				&$\frac{n}{s} \Big( \sqrt{\tfrac{1 + \epsilon / g}{g} } + \epsilon \Big)^2$
				& $1+\epsilon$ 
				& $1+\epsilon$	\\
				Theorems
				& Theorem~\ref{thm:biasvariance:classical}
				& Theorem~\ref{thm:biasvariance:classical_avg}
				& Theorem~\ref{thm:biasvariance:hessian}
				& Theorem~\ref{thm:biasvariance:hessian_avg}\\
				\hline
			\end{tabular}
		\end{footnotesize}
	\end{center}
\end{table}

We first study classical and Hessian sketches from the {\bf optimization perspective}.
Theorems~\ref{thm:optimization:classical} and \ref{thm:optimization:hessian} show:
\begin{itemize}
	\item
    {Classical sketch} achieves relative error in the objective value.
	With sketch size $s = \tilde\OM (d/\epsilon   )$, the
  sketched solution satisfies $f ({\W}^{\textrm{c}}) \leq (1+\epsilon) f (\W^\star)$.
	\vspace{-1mm}
	\item
    {Hessian sketch} does not achieve relative error in the objective value. In particular,
	if $\frac{1}{n} \| \Y\|_F^2$ is much larger than $f (\W^\star)$,
	then $f ({\W}^{\textrm{h}})$ can be far larger than $f (\W^\star)$.
	\vspace{-1mm}
	\item
    For both classical and Hessian sketch, the relative quality of approximation often improves as
	the regularization parameter $\gamma$ increases (because $\beta$ decreases).
\end{itemize}

We then study classical and Hessian sketch from the {\bf statistical perspective}, by modeling
$\Y = \X \W_0 + \Xii$ as the sum of a true linear model and random noise, 
decomposing the risk $R (\W ) = \EB \|\X \W - \X \W_0 \|_F^2$ into bias and variance terms,
and bounding these terms.
We draw the following conclusions
(see Theorems~\ref{thm:bias_var_decomp}, \ref{thm:biasvariance:classical}, 
\ref{thm:biasvariance:hessian} for the details):
\begin{itemize}
	\item
    The bias of {classical sketch} can be nearly as small as that of the optimal solution.
	The variance is $\Theta \big(\frac{n}{s} \big)$ times that of the optimal solution;
	this bound is optimal. 
  Therefore over-regularization\footnote{For example,
    using a larger value of the regularization parameter $\gamma$ 
    than one would optimally choose for the unsketched problem.}
	should be used to supress the variance.
	(As $\gamma$ increases, the bias increases, and the variance decreases.)
	\item
    Since Hessian sketch uses the whole of $\Y$, the variance of {Hessian sketch} can be close to that of the optimal solution.
    However, Hessian sketch incurs a high bias, especially when $n \gamma$ is small compared to $\|\X\|_2^2$.
    This indicates that over-regularization is necessary for Hessian sketch to deliver solutions with low bias.
\end{itemize}
Our empirical evaluations bear out these theoretical results.
In particular, in Section~\ref{sec:experiment1}, we show in Figure~\ref{fig:risk_nb2}
that even when the regularization parameter $\gamma$ is fine-tuned,
the risks of classical and Hessian sketch are worse than 
that of the optimal solution by an order of magnitude.
This is an empirical demonstration of the fact that the near-optimal properties of sketch from the optimization
perspective are much less relevant in a statistical setting than its sub-optimal statistical properties.

We propose to use {\bf model averaging},
which averages the solutions of $g$ sketched MRR problems,
to attain lower optimization and statistical errors.
Without ambiguity, we denote model-averaged classical and Hessian sketches by
$\W^\textrm{c}$ and $\W^\textrm{h}$, respectively.
Theorems \ref{thm:optimization:classical_avg}, \ref{thm:optimization:hessian_avg},
\ref{thm:biasvariance:classical_avg}, \ref{thm:biasvariance:hessian_avg} establish the following results:
\begin{itemize}
	\item
    {Classical Sketch}.
    Model averaging decreases the objective function value and the variance
	and does not increase the bias.
	Specifically, with the same sketch size $s$, 
  model averaging ensures
	$\frac{ f(\W^{\textrm{c}}) - f (\W^\star) }{f(\W^{\star})}$ and
	$\frac{ \var (\W^{\textrm{c}}) }{ \var (\W^{\star}) }$ respectively decrease to 
	almost $\tfrac{1}{g}$ of those of classical sketch without model averaging,
	provided that $s \gg d$.
	See Table~\ref{tab:main} for the details.
	\item
    {Hessian Sketch}.
	Model averaging decreases the objective function value and the bias
	and does not increase the variance.
\end{itemize}
In the distributed setting, the feature-response pairs $(\x_1, \y_1) , \cdots , (\x_n, \y_n ) \in \RB^{d} \times \RB^m$
are divided among $g$ machines. Assuming that the data have been shuffled randomly, each machine contains a sketch 
of the MRR constructed by uniformly sampling rows from the data set without replacement. 
We illustrate this procedure in Figure~\ref{fig:partition}.
In this setting, the model averaging
procedure communicates the $g$ local models only once to return the final estimate; this process has very low communication and latency costs,
and suggests two further applications of classical sketch with model averaging:
\begin{itemize}
	\item
	{Model Averaging for Machine Learning.}
  When a low-precision solution is acceptable, model averaging can be used in lieu of
	distributed numerical optimization algorithms requiring multiple rounds of communication.
  If $\frac{n}{g}$ is large enough compared to $d$ and the row coherence of $\X$ is small,
	then ``one-shot'' model averaging
	has bias and variance comparable to the optimal solution.
	\vspace{-1mm}
	\item
	{Model Averaging for Optimization}.
	If a high-precision solution to MRR is required,
	then an iterative numerical optimization algorithm must be used.
	The cost of such algorithms heavily depends on the quality of the initialization.\footnote{For example, the conjugate gradient method satisfies
		$\tfrac{\|\W^{(t)} - \W^\star \|_F^2}{\|\W^{(0)} - \W^\star \|_F^2} \leq
		\theta_1^t$ 
    and stochastic block coordinate descent \citep{tu2016large} satisfies
		$\tfrac{ \EB f (\W^{(t)}) - f (\W^\star)}{ f (\W^{(0)})  - f (\W^\star) }
		\leq \theta_2^t$. 
		Here $\W^{(t)}$ is the output of the $t$-th iteration;
		$\theta_1, \theta_2 \in (0, 1)$ depend on the condition number of $\X^T \X + n \gamma \I_d$
		and some other factors.} 
  A good initialization reduces the number of iterations needed to reach convergence.
	The averaged model is provably close to the optimal solution, so 
	model averaging provides a high-quality initialization for more expensive algorithms.
\end{itemize}

\begin{figure}
	\begin{center}
		\includegraphics[width=0.8\textwidth]{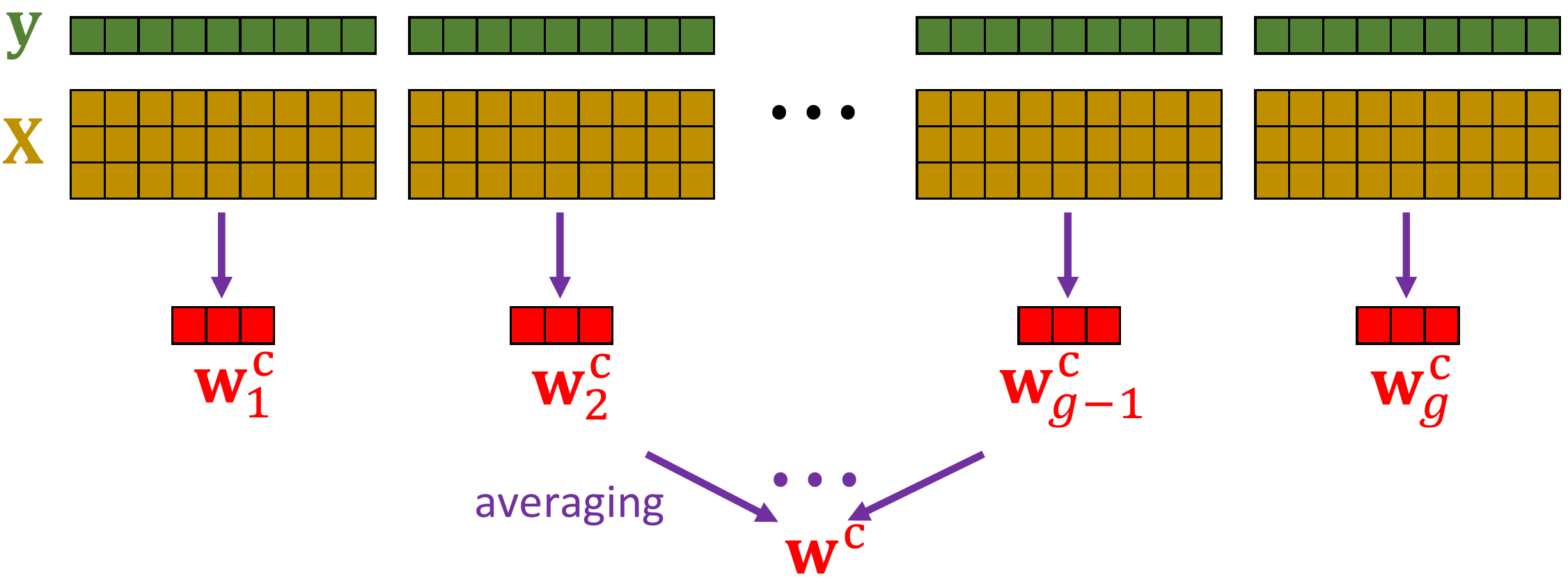}
	\end{center}
  \caption{Using model averaging with the classical sketch
    in the distributed setting to approximately solve LSR.}
	\label{fig:partition}
\end{figure}

\subsection{Prior Work}

The body of work on sketched LSR mentioned 
earlier~\citep{drineas2006sampling,drineas2011faster,clarkson2013low,meng2013low,nelson2013osnap} 
shares many similarities with our results.
However, the theories of sketched LSR developed from the optimization
perspective do not obviously extend to MRR, and the statistical analysis of LSR
and MRR differ: among other differences, 
LSR is unbiased while MRR is biased and therefore
has a bias-variance tradeoff that must be considered.

\citet{lu2013faster}~has considered a different application of sketching to ridge regression: 
they assume $d \gg n$, reduce the number of features in $\X$ using sketching,
and conduct statistical analysis.
Our setting differs in that we consider $n \gg d$, reduce the number of samples by sketching,
and allow for multiple responses.

The model averaging analyzed in this paper is similar in spirit to the \textsc{AvgM} algorithm of~\citep{zhang2013communication}.
When classical sketch is used with uniform row sampling without replacement, our model averaging procedure is a special case 
of \textsc{AvgM}. However, our results do not follow from those of~\citep{zhang2013communication}.
First, we make no assumption on the data, $\X$ and $\Y$, and the model (parameters), $\W$.
Second, we study both the optimization objective, $\|\X \W^{\textrm{c}} - \X \W^\star\|_F^2$, and the statistical objective, $\EB \|\X \W^{\textrm{c}} - \X \W_0\|_F^2$,
where $\W^{\textrm{c}}$ is the average of the approximate solutions obtained used classical sketch,
$\W_0$ is the unknown ground truth, and $\W^\star$ is the optimal solution based on the observed data;
they studied solely the optimization objective.
Third, our results apply to many other sketching ensembles than uniform sampling without replacement.
Our results clearly indicate that the performance critically depends on the row coherence of $\X$;
this dependence has not been explicitly captured in \citep{zhang2013communication}.
\citet{zhang2014divide} studied a different statistical objective and their resulting bound has a higher-order of dependence on $d$ and other parameters. 

Iterative Hessian sketch has been studied in~\citet{pilanci2015iterative,wang2016sketching,wang2017giant}.
By way of comparison, all the algorithms in this paper are ``one-shot'' rather than iterative.
This work has connections to the contemporary works~\citep{avron2017sharper,thanei2017random,derezinski2017unbiased,derezinski2018subsampling}.
\citet{avron2017sharper}~studied classical sketch from the optimization perspective;
\citet{thanei2017random} studied LSR with model averaging;
\citet{derezinski2017unbiased,derezinski2018subsampling} studied linear regression with volume sampling for experimental design.

\subsection{Paper Organization}

Section~\ref{sec:preliminary} defines our notation and introduces the sketching schemes we consider.
Section~\ref{sec:main} presents our theoretical results. 
Sections~\ref{sec:experiment1} and~\ref{sec:experiments_real} conduct experiments to 
verify our theories and demonstrates the efficacy of model averaging.
Section~\ref{sec:proofsketch} sketches the proofs of our main results.
Complete proofs are provided in the appendix.


\section{Preliminaries} \label{sec:preliminary}

Throughout, we take $\I_n$ to be the $n\times n$ identity matrix and $\0$ to be a vector or matrix of all zeroes of the appropriate size.
Given a matrix $\A = [a_{i j}]$, the $i$-th row is denoted by $\a_{i:}$, 
and the $j$-th column is denoted by $\a_{:j}$.
The Frobenius and spectral norms of $\A$ are written as, respectively, $\|\A \|_F$ and $\|\A \|_2$.
The set $\{1, 2, \cdots, n\}$ is written $[n]$.
Let $\OM$, $\Omega$, and $\Theta$ be the standard asymptotic notation, and let
$\tilde{\OM}$ conceal logarithmic factors.

Throughout, we fix $\X \in \RB^{n \times d}$ as our matrix of features. We set $\rho = \rk (\X)$ and
write the SVD of $\X$ as
$\X =\U \Si \V^T $,
where $\U$, $\Si$, $\V$ are respectively $n\times \rho$,
$\rho\times \rho$, and $d\times \rho$ matrices.
We let $\sigma_1 \geq \cdots \geq \sigma_\rho > 0$ be the singular values of $\X$.
The Moore-Penrose inverse of $\X$ is defined by
$\X^\dag = \V \Si^{-1} \U^T$.
The row leverage scores of $\X$ are $l_i = \|\u_{:i}\|_2^2$ for $i \in [n]$.
The row coherence of $\X$ is $\mu (\X) = \frac{n}{\rho} \max_{i } \|\u_{:i}\|_2^2$.
Throughout, we let $\mu$ be shorthand for $\mu (\X)$.
The notation defined in Table~\ref{tab:notation} is used throughout this paper.

\begin{table}[t]\setlength{\tabcolsep}{0.3pt}
	\def\arraystretch{1.1}
	\caption{The commonly used notation.}
	\label{tab:notation}
	\begin{center}
		\begin{small}
			\begin{tabular}{c l}
				\hline
				~~~{\bf Notation}~~~&~~~{\bf Definition}~~~\\
				\hline
				~~~$\X \in \RB^{n\times d}$~~~ & ~~~each row is a data sample (feature vector)~~~\\
				~~~$\Y \in \RB^{n\times m}$~~~ & ~~~each row contains the corresponding responses~~~\\
				~~~$\U \Si \V^T$~~~ & ~~~the SVD of $\X$~~~\\
				~~~$\rho$~~~ & ~~~the rank of $\X$~~~ \\
				~~~$\mu $~~~ & ~~~the row coherence of $\X$~~~\\
				~~~$\sigma_i$~~~ & ~~~the $i$-th largest singular value of $\X$~~~\\
				~~~$\gamma$~~~ & ~~~the regularization parameter~~~\\
				~~~$\beta $~~~ & ~~~$\beta=\frac{\|\X\|_2^2}{\|\X\|_2^2 + n \gamma} \leq 1$~~~ \\
				~~~$\S \in \RB^{n\times s}$~~~ & ~~~the sketching matrix~~~ \\
				~~~${\W^\star} \in \RB^{d\times m}$~~~ & ~~~the optimal solution \eqref{eq:def_W_optimal}~~~\\
        ~~~${\W}^{\textrm{c}} \in \RB^{d\times m}$~~~ & ~~~approximate solution obtained using the classical sketch \eqref{eq:def_W_tilde}~~~\\
        ~~~${\W}^{\textrm{h}} \in \RB^{d\times m}$~~~ & ~~~approximate solution obtained using the Hessian sketch \eqref{eq:def_W_hat}~~~\\
		        ~~~${\W_0} \in \RB^{d\times m}$~~~ & ~~~the unknown ground truth (in the statistical setting)~~~\\
				\hline
			\end{tabular}
		\end{small}
	\end{center}
\end{table}

Matrix sketching attempts to reduce the size of large matrices while minimizing the loss of spectral information that is 
useful in tasks like linear regression.
We denote the process of sketching a matrix $\X \in \RB^{n \times d}$ by $\X' = \S^T \X$.
Here, $\S \in \RB^{n\times s}$ is called a sketching matrix
and $\X' \in \RB^{s\times d}$ is called a sketch of $\X$.
In practice, except for Gaussian projection (where the entries of $\S$ are i.i.d.\ sampled from $\NM (0, 1/s)$),
the sketching matrix $\S$ is not formed explicitly.

Matrix sketching can be accomplished by random sampling or random projection.
{\bf Random sampling} corresponds to sampling rows of $\X$ i.i.d.\ with replacement according to
given row sampling probabilities
$p_1 , \cdots , p_m \in (0, 1)$.
The corresponding (random) sketching matrix $\S \in \RB^{n\times s}$ 
has exactly one non-zero entry, whose position indicates the index of the selected row in each column;
in practice, this $\S$ is not explicitly formed.
{\bf Uniform sampling} fixes $p_1 = \cdots  = p_n = \frac{1}{n}$.
{\bf Leverage score sampling} sets $p_i$ proportional to the (exact or approximate~\citep{drineas2012fast}) row leverage scores $l_i$ of $\X$.
In practice {\bf shrinked leverage score sampling}
can be a better choice than leverage score sampling~\citep{ma2014statistical}.
The sampling probabilities of shrinked leverage score sampling are defined by
$p_i = \frac{1}{2} \big(\frac{{l}_i}{\sum_{j=1}^n {l}_j } + \frac{1}{n} \big)$.\footnote{In fact,
	$p_i$ can be any convex combination of $\frac{{l}_i}{\sum_{j=1}^n {l}_j }$
	and $\frac{1}{n}$ \citep{ma2014statistical}. We use the weight $\frac{1}{2}$ for convenience; 
	our conclusions extend in a straightforward manner to other weightings.}

The exact leverage scores are unnecessary in practice;
constant-factor approximation to the leverage scores is sufficient.
Leverage scores can be efficiently approximated by the algorithms of \citep{drineas2012fast}.
Let $l_1 , \cdots , l_n$ be the true leverage scores.
We denote the approximate leverages by $\widetilde{l}_1 , \cdots , \widetilde{l}_n$ and require that they satisfy
\begin{equation} \label{eq:def:appro_lev}
\tilde{l}_q \in [l_q, \tau l_q ]
\quad \textrm{ for all } \;  q\in [n] ,
\end{equation}
where $\tau \geq 1$ indicates the quality of approximation.
We then use $p_q = \tilde{l}_q / \sum_j \tilde{l}_j$ as the sampling probabilities.
One can obtain the same accuracies when using approximate leverage scores in place of the true leverage scores by 
increasing $s$ by a factor of $\tau$, so as long as $\tau$ is a small constant, 
the orders of the sketch sizes when using exact or approximate leverage score sampling are the same.
Thus we do not distinguish between exact and approximate leverage scores in this paper.
For {shrinked leverage score sampling},
we define the sampling probabilities
\begin{eqnarray} \label{eq:rr_mixed_prob}
  p_i = \tfrac{1}{2} \left(\tfrac{\tilde{l}_i}{\sum_{j=1}^n \tilde{l}_j } + \tfrac{1}{n} \right)
\quad \textrm{ for } \; i = 1 , \dots , n .
\end{eqnarray}

{\bf Gaussian projection} is also well-known as the prototypical Johnson-Lindenstrauss transform \citep{johnson1984extensions}.
Let $\G \in \RB^{n\times s}$ be a standard Gaussian matrix,
i.e., each entry is sampled independently from $\NM (0, 1)$.
The matrix $\S = \frac{1}{\sqrt{s}} \G$ is a Gaussian projection matrix.
It takes $\OM (nds)$ time to apply $\S\in \RB^{n\times s}$ to any $n\times d$ dense matrix,
which makes Gaussian projection computationally inefficient relative to other forms of sketching.

The {\bf Subsampled randomized Hadamard transform (SRHT)} \citep{drineas2011faster,lu2013faster,tropp2011improved}
is a more efficient alternative to Gaussian projection.
Let $\H_n \in \RB^{n\times n}$ be the Walsh-Hadamard matrix with $+1$ and $-1$ entries,
$\D \in \RB^{n\times n}$ be a diagonal matrix with diagonal entries sampled uniformly from $\{+1, -1\}$,
and $\PP \in \RB^{n\times s}$ be the uniform row sampling matrix defined above.
The matrix
$\S = \frac{1}{\sqrt{n}} \D \H_n \PP \in \RB^{n\times s}$
is an SRHT matrix,
and can be applied to any $n\times d$ matrix in $\OM(n d \log s)$ time.
In practice, the subsampled randomized Fourier transform (SRFT) \citep{woolfe2008fast} is often used in lieu
of the SRHT, because the SRFT exists for all values of $n$, whereas $\H_n$ exists only for some values of $n$.
Their performance and theoretical analyses are very similar.

{\bf CountSketch} can be applied to any $\X \in \RB^{n\times d}$ in $\OM (nd)$ time \citep{charikar2004finding,clarkson2013low,meng2013low,nelson2013osnap,pham2013fast,weinberger2009feature}.
Though more efficient to apply, CountSketch requires a larger sketch size than Gaussian projections, 
SRHT, and leverage score sampling to attain the same theoretical guarantees.
Interested readers can refer to~\citep{woodruff2014sketching} for a detailed description of CountSketch.
Unlike the other sketching methods mentioned here, model averaging with CountSketch may not be theoretically sound.
See Remark~\ref{remark:cs_avg} for further discussion.


\section{Main Results} \label{sec:main}

Sections~\ref{sec:main:optimization} and \ref{sec:main:statistical} analyze sketched MRR from, respectively,  the
optimization and statistical perspectives.
Sections~\ref{sec:main:optimization_avg} and \ref{sec:main:statistical_avg} capture the 
impacts of model averaging on, respectively, the optimization and statistical properties of sketched MRR.

We described six sketching methods in Section~\ref{sec:preliminary}.
For simplicity, in this section, we refer to leverage score sampling,
shrinked leverage score sampling, Gaussian projection, and SRHT
as {\bf the four sketching methods} while we refer to uniform sampling and CountSketch by name.
Throughout, let $\mu$ be the row coherence of $\X$ and
$\beta = \frac{\|\X\|_2^2}{\|\X\|_2^2 + n \gamma} \leq 1$.


\subsection{Sketched MRR: Optimization Perspective}  \label{sec:main:optimization}

Theorem~\ref{thm:optimization:classical} shows that 
$f ({\W}^{\textrm{c}})$, the objective value of classical sketch, 
is close to the optimal objective value $f (\W^\star)$, and that the
approximation quality improves as the regularization parameter $\gamma$ increases.

\begin{theorem} [Classical Sketch]\label{thm:optimization:classical}
	Let $\beta = \frac{\|\X\|_2^2}{\|\X\|_2^2 + n \gamma} \leq 1$.
	For the four sketching methods with
	$s = \tilde\OM \big( \frac{ d}{\epsilon}  \big)$,
	uniform sampling with
	$s = \OM \big( \frac{\mu  d \log d}{\epsilon}  \big)$,
	and CountSketch with
	$s = \OM \big( \frac{ d^2}{\epsilon}  \big)$,
	the inequality
	\begin{eqnarray*}
		f ({\W}^{\textrm{c}}) -f (\W^\star )
		& \leq & \epsilon \beta \, f (\W^\star )
	\end{eqnarray*}
	holds with probability at least 0.9.
  The uncertainty is with respect to the random choice of sketching matrix.
\end{theorem}

The corresponding guarantee for the performance of Hessian sketch is given in 
Theorem~\ref{thm:optimization:hessian}. It is weaker than the guarantee for classical sketch,
especially when $\frac{1}{n} \|\Y\|_F^2$ is far larger than $f (\W^\star )$.
If $\Y$ is nearly noiseless---$\Y$ is well-explained by a linear combination of 
the columns of $\X$---and $\gamma$ is small,
then $f (\W^\star )$ is close to zero,
and consequently $f (\W^\star )$ can be far smaller than $\frac{1}{n} \|\Y\|_F^2$.
Therefore, in this case which is ideal for MRR, $f ({\W}^{\textrm{h}})$ is not close to $f(\W^\star )$ and our 
theory suggests Hessian sketch does not perform as well as classical sketch.
This is verified by our experiments (see Figure~\ref{fig:obj_nb2}), which show that
unless $\gamma$ is large or a large portion of $\Y$ is outside the column space of $\X$,
the ratio $\frac{ f (\W^\textrm{h})}{f (\W^\star )}$ can be large.

\begin{theorem} [Hessian Sketch] \label{thm:optimization:hessian}
	Let $\beta = \frac{\|\X\|_2^2}{\|\X\|_2^2 + n \gamma} \leq 1$.
	For the four sketching methods with
	$s = \tilde\OM \big( \frac{ d}{\epsilon} \big)$,
	uniform sampling with $s =  {\OM} \big(  \frac{\mu  d \log d}{\epsilon} \big)$,
	and CountSketch with $s =  {\OM} (\frac{ d^2}{\epsilon})$,
	the inequality
	\vspace{-1mm}
	\begin{small}
		\begin{eqnarray*}
			f ({\W}^{\textrm{h}}) - f (\W^\star )
			& \leq & \epsilon \beta^2 \, \left(\tfrac{\|\Y\|_F^2 }{n }   - f (\W^\star ) \right) .
			\vspace{-1mm}
		\end{eqnarray*}
	\end{small}
	holds with probability at least 0.9.
  The uncertainty is with respect to the random choice of sketching matrix.
\end{theorem}

These two results imply that $f ({\W}^{\textrm{c}})$ and $f ({\W}^{\textrm{h}})$ can be close to $f (\W^\star)$.
When this is the case, curvature of the objective function ensures that
the sketched solutions ${\W}^{\textrm{c}}$ and ${\W}^{\textrm{h}}$ are close to the optimal solution $\W^\star$.
Lemma~\ref{lem:optimization:mdistance} bounds the Mahalanobis distance $\| \M ({\W} - \W^\star ) \|_F^2$.
Here $\M$ is any non-singular matrix; in particular, it can be the identity matrix or $(\X^T \X)^{1/2}$.
Lemma~\ref{lem:optimization:mdistance} is a consequence of Lemma~\ref{lem:optimization}.

\begin{lemma}[Mahalanobis Distance] \label{lem:optimization:mdistance}
	Let $f $ be the objective function of MRR defined in \eqref{eq:def_f},
	$\W \in \RB^{d\times m}$ be arbitrary,
	and $\W^\star$ be the optimal solution defined in \eqref{eq:def_W_optimal}.
	For any non-singular matrix $\M$, the Mahalanobis distance satisfies
	\begin{small}
		\begin{eqnarray*}
			\frac{1}{n} \big\| \M ({\W} - \W^\star )  \big\|_F^2
			\; \leq \;
			\frac{f (\W) - f (\W^\star ) }{
				\sigma_{\min}^{2} \big[ (\X^T  \S \S^T \X + n \gamma \I_d)^{1/2} \M^{-1} \big]} .
		\end{eqnarray*}
	\end{small}%
\end{lemma}

By choosing $\M = (\X^T \X)^{1/2}$, we can bound $ \tfrac{1}{n} \| \X \W - \X \W^\star \|_F^2$ 
in terms of the difference in the objective values:
\begin{small}
	\begin{eqnarray*}
		\tfrac{1}{n} \big\| \X {\W} - \X \W^\star   \big\|_F^2
		& \leq & \beta \big[ f (\W) - f (\W^\star ) \big] ,
	\end{eqnarray*}
\end{small}%
where $\beta = \tfrac{\|\X\|_2^2}{\|\X\|_2^2 + n \gamma} \leq 1$.
With Lemma~\ref{lem:optimization:mdistance}, we can directly apply 
Theorems~\ref{thm:optimization:classical} or \ref{thm:optimization:hessian}
to bound $\tfrac{1}{n} \| \X {\W^\textrm{c}} - \X \W^\star   \|_F^2$
or $\tfrac{1}{n} \| \X {\W^\textrm{h}} - \X \W^\star   \|_F^2$.


\subsection{Sketched MRR: Statistical Perspective}  \label{sec:main:statistical}

We consider the following fixed design model.
Let $\X \in \RB^{n\times d}$ be the observed feature matrix,
$\W_0 \in \RB^{d\times m}$ be the true and unknown model,
$\Xii \in \RB^{n\times m}$ contain unknown random noise, and
\begin{eqnarray}\label{eq:rr_datamodel}
\Y \; = \; \X \W_0 + \Xii
\end{eqnarray}
be the observed responses.
We make the following standard weak assumptions on the noise:
\[
\EB [\Xii ] = \0
\quad \textrm{ and } \quad
\EB [ \Xii \Xii^T ] = \xi^2 \I_n .
\]
We observe $\X$ and $\Y$ and seek to estimate $\W_0$.

We can evaluate the quality of the estimate by the risk:
\begin{eqnarray} \label{eq:def_risk}
R (\W)
\; = \; \tfrac{1}{n} \EB \big\|\X \W - \X \W_0 \big\|_F^2 ,
\end{eqnarray}
where the expectation is taken w.r.t.\ the noise $\Xii$.
We study the risk functions $R (\W^\star)$, $R ({\W}^{\textrm{c}})$, 
and $R ({\W}^{\textrm{h}})$ in the following.

\begin{theorem} [Bias-Variance Decomposition] \label{thm:bias_var_decomp}
	We consider the data model described in this subsection.
	Let $\W$ be $\W^\star$, ${\W}^{\textrm{c}}$, or ${\W}^{\textrm{h}}$,
	as defined in \eqref{eq:def_W_optimal}, \eqref{eq:def_W_tilde}, or \eqref{eq:def_W_hat},
	respectively;
	then the risk function can be decomposed as
	\begin{eqnarray*}
		R (\W )
		& = & \bias^2 ( \W ) + \var ( \W ) .
	\end{eqnarray*}
	Recall the SVD of $\X$ defined in Section~\ref{sec:preliminary}: $\X = \U \Si \V^T$.
	The bias and variance terms can be written as
	\begin{small}
		\begin{eqnarray*}
			\bias \big( \W^\star \big)
			& = & \gamma\sqrt{n}
			\Big\| (\Si^{2} + n\gamma \I_\rho )^{-1} \Si \V^T \W_0  \Big\|_F , \\
			\var \big( \W^\star \big)
			& = & \tfrac{\xi^2}{n}   \Big\| \big( \I_\rho + n\gamma \Si^{-2} \big)^{-1} \Big\|_F^2 , \\
			\bias \big( \W^{\textrm{c}} \big)
			& = &  \gamma\sqrt{n}
			\Big\|\big( \U^T \S \S^T \U + n\gamma \Si^{-2} \big)^{\dag} \Si^{-1} \V^T \W_0\Big\|_F , \\
			\var \big( \W^{\textrm{c}} \big)
			& = & \tfrac{\xi^2}{n}   \Big\| \big( \U^T \S \S^T \U + n\gamma \Si^{-2} \big)^{\dag} \U^T \S \S^T  \Big\|_F^2 ,\\
			\bias \big({\W}^{\textrm{h}}\big)
			& = & \gamma \sqrt{n} \Big\|
			\Big(  \Si^{-2} + \tfrac{\U^T \S \S^T \U - \I_\rho}{n \gamma} \Big) 
			\big(\U^T \S \S^T \U + n \gamma \Si^{-2} \big)^{\dag} \Si \V^T \W_0 \Big\|_F  , \\
			\var \big( \W^{\textrm{h}} \big)
			& = & \tfrac{\xi^2 }{n} \Big\| \big(\U^T \S \S^T \U + n \gamma \Si^{-2} \big)^{\dag } \Big\|_F^2 .
		\end{eqnarray*}
	\end{small}%
\end{theorem}

The functions $\bias (\W^\star )$ and $\var (\W^\star )$ are deterministic.
The randomness in $\bias (\W^\textrm{c} )$, $\var (\W^\textrm{c})$, $\bias (\W^\textrm{h} )$, and $\var (\W^\textrm{h})$ all arises from the sketching matrix $\S$.

Throughout this paper, we compare the bias and variance of classical sketch and Hessian sketch to 
those of the optimal solution $\W^\star$.
We first study the bias, variance, and risk of $\W^\star$,
which will help us understand the subsequent comparisons.
We can assume that $\Si^2 = \V^T \X^T \X \V$ is linear in $n$;
this is reasonable because $\X^T \X = \sum_{i=1}^n \x_i \x_i^T$
and $\V$ is an orthogonal matrix.
\begin{itemize}
	\item 
	{\bf Bias.}
	The bias of $\W^\star$ is independent of $n$ and is increasing with $\gamma$.
  The bias is the price paid for using regularization to decrease the variance;
	for least squares regression, $\gamma$ is zero, and the bias is zero.
	\item
	{\bf Variance.}
	The variance of $\W^\star$ is inversely proportional to $n$.
  As $n$ grows, the variance decreases to zero, and we must also decrease $\gamma$ to ensure that the sum of the squared bias and variance decreases to zero.
	\item
	{\bf Risk.}
	Note that $\W^\star$ is not the minimizer of $R (\cdot )$;
	$\W_0$ is the minimizer because $R (\W_0 ) =0$.
	Nevertheless, because $\W_0 $ is unknown,
  $\W^\star$ for a carefully chosen $\gamma$ is a standard proxy for the exact minimizer in practice.
  It is thus highly interesting to compare the risk of MRR solutions obtained using sketching to to $R (\W^\star )$.
\end{itemize}

Theorem~\ref{thm:biasvariance:classical} provides upper and lower bounds on
the bias and variance of solutions obtained using classical sketch. In particular, we see that 
that $\bias ({\W}^{\textrm{c}})$ is within a factor of $(1 \pm \epsilon)$ of $\bias (\W^\star)$.
However, $\var ({\W}^{\textrm{c}})$ can be $\Theta (\frac{n}{s})$ times worse than $\var ({\W^\star})$.
The absolute value of $\var ({\W}^{\textrm{c}})$ is inversely proportional to $s$,
whereas the absolute value of $\bias ({\W}^{\textrm{c}})$ is almost independent of $s$.

\begin{theorem} [Classical Sketch] \label{thm:biasvariance:classical}
	Assume $ s \leq n$.
	For Gaussian projection and SRHT sketching
	with $s = \tilde{\OM} (\frac{d}{\epsilon^2})$,
	uniform sampling with $s =  {\OM} (\frac{ \mu  d \log d}{\epsilon^2})$,
	or CountSketch with $s =  {\OM} (\frac{d^2}{\epsilon^2})$,
	the inequalities
	\begin{eqnarray*}
		1-\epsilon
		& \leq & \tfrac{\bias ( \W^{\textrm{c}} )}{\bias (\W^\star ) }
		\; \leq \; 1+\epsilon ,\\
		( 1- \epsilon ) \tfrac{n}{s}
		& \leq & \tfrac{\var ( \W^{\textrm{c}} ) }{ \var (\W^\star ) }
		\; \leq \; ( 1+ \epsilon ) \tfrac{n}{s}
	\end{eqnarray*}
	hold with probability at least 0.9.
	For shrinked leverage score sampling with $s = {\OM} (\frac{d \log d}{\epsilon^2})$,
	these inequalities, except for the lower bound on the variance,
	hold with probability at least 0.9.
  Here the randomness comes from the sketching matrix $\S$.
\end{theorem}

\begin{Remark}
  To establish an upper (lower) bound on the variance,
  we need an upper (lower) bound on $\|\S\|_2^2$.
  There is no nontrivial upper nor lower bound on $\|\S\|_2^2$ for leverage score sampling,
	so the variance of leverage score sampling cannot be bounded.
  Shrinked leverage score sampling satisfies the upper bound $\|\S \|_2^2 \leq \frac{2n}{s}$;
	but $\|\S\|_2^2$ does not have a nontrivial lower bound,
  so there is no nontrivial lower bound on the variance of shrinked leverage score. 
  Remark~\ref{remark:sketch:spectral} explains the nonexistence of the relevant bounds on $\|\S\|_2^2$ for both variants of leverage score sampling.
\end{Remark}

Theorem~\ref{thm:biasvariance:hessian} establishes similar upper and lower bounds
on the bias and variance of solutions obtained using Hessian sketch. The situation is the reverse of that
with classical sketch: 
the variance of $\W^{\textrm{h}}$ is close to that of $\W^\star$ if $s$ is large enough,
but as the regularization parameter $\gamma$ goes to zero, $\bias ( \W^{\textrm{h}} )$ becomes much larger than $\bias (\W^\star )$.
The theory suggest that Hessian sketch should be preferred over classical
sketch when $\Y$ is very noisy, because Hessian sketch does not
magnify the variance.

\begin{theorem} [Hessian Sketch] \label{thm:biasvariance:hessian}
	For the four sketching methods
	with $s = \tilde{\OM} (\frac{d}{\epsilon^2})$,
	uniform sampling with $s =  {\OM} (  \frac{\mu d \log d}{\epsilon^2})$,
	and CountSketch with $s =  {\OM} (\frac{d^2}{\epsilon^2})$,
	the inequalities
	\begin{eqnarray*}
		\tfrac{\bias ( \W^{\textrm{h}} )}{\bias (\W^\star ) }
		& \leq & (1+\epsilon ) \,
		\big( 1 + \tfrac{\epsilon \|\X\|_2^2 }{n \gamma} \big)  ,\\
		1 - \epsilon
		& \leq &
		\tfrac{\var ( \W^{\textrm{h}} ) }{ \var (\W^\star ) }
		\; \leq \; 1 + \epsilon
	\end{eqnarray*}
	hold with probability at least 0.9.
  Further assume that the $\rho$-th singular value of $\X$ satisfies $\sigma_\rho^2 \geq \frac{n \gamma}{\epsilon }$, then
	\begin{align*}
	\tfrac{\bias ( \W^{\textrm{h}} )}{\bias (\W^\star ) }
	\; \geq \; \tfrac{1}{1 + \epsilon} \,
	\big(  \tfrac{\epsilon \sigma_{\rho}^2 }{n \gamma} - 1 \big)
	\end{align*}
	with probability at least 0.9.
  Here the randomness is in the choice of sketching matrix $\S$.
\end{theorem}

The lower bound on the bias shows that the solution from Hessian sketch can exhibit a much higher bias than the optimal solution. 
The gap between $\bias({\W}^{\textrm{h}})$ and $\bias(\W^\star)$ can be
lessened by increasing the regularization parameter $\gamma$, but such
over-regularization increases the baseline $\bias(\W^\star)$ itself.
It is also worth mentioning that unlike $\bias (\W^\star )$ and $\bias({\W}^{\textrm{c}})$,
$\bias({\W}^{\textrm{h}})$ is not monotonically increasing with $\gamma$,
as is empirically verified in Figure~\ref{fig:risk_nb2}.

In sum, our theory shows that the classical and Hessian sketches are not
statistically comparable to the optimal solutions: classical sketch has too
high a variance, and Hessian sketch has too high a bias for reasonable
amounts of regularization.
In practice, the regularization parameter $\gamma$ should be tuned to optimize the prediction accuracy.
Our experiments in Figure~\ref{fig:risk_nb2} show that even with carefully chosen $\gamma$, 
the risks of classical and Hessian sketch can be 
higher than the risk of the optimal solution by an order of magnitude.
Formally speaking, $\min_{\gamma} R({\W}^{\textrm{c}}) \gg \min_{\gamma} R (\W^\star)$
and $\min_{\gamma} R({\W}^{\textrm{h}}) \gg \min_{\gamma} R (\W^\star)$ hold in practice.

Our empirical study in Figure~\ref{fig:risk_nb2} suggests classical and Hessian
sketch both require over-regularization, i.e., setting $\gamma$ larger than is
best for the optimal solution $\W^\star$.
Formally speaking, $\argmin_{\gamma} R({\W}^{\textrm{c}}) > \argmin_{\gamma} R (\W^\star)$
and $\argmin_{\gamma} R({\W}^{\textrm{h}}) > \argmin_{\gamma} R (\W^\star)$.
Although this is the case for both types of sketching, the underlying explanations are different.
Classical sketches have a high variance, so a large $\gamma$ is required to supress their variance 
(the variance is non-increasing with $\gamma$).
Hessian sketches magnify the bias when $\gamma$ is small,
so a reasonably large $\gamma$ is necessary to lower their bias.


\subsection{Model Averaging: Optimization Perspective}  \label{sec:main:optimization_avg}

We consider model averaging as a method to increase the accuracy of sketched MRR solutions.
The model averaging procedure is straightforward:
one independently draws $g$ sketching matrices $\S_1 , \cdots, \S_g \in \RB^{n\times s}$,
uses these to form $g$ sketched MRR solutions, denoted by 
$\{{\W}^{\textrm{c}}_i\}_{i=1}^g$ or $\{{\W}^{\textrm{h}}_i\}_{i=1}^g$, and 
averages these solutions to obtain the final estimate ${\W}^{\textrm{c}} = \frac{1}{g} \sum_{i=1}^g {\W}^{\textrm{c}}_i$
or ${\W}^{\textrm{h}} = \frac{1}{g} \sum_{i=1}^g {\W}^{\textrm{h}}_i$.
Practical applications of model averaging are enumerated in Section~\ref{sec:intro:main}.

Theorems~\ref{thm:optimization:classical_avg} and~\ref{thm:optimization:hessian_avg}
present guarantees on the optimization accuracy of using model averaging on
classical/Hessian sketch solutions.
We can contrast these with the guarantees provided for sketched MRR in 
Theorems~\ref{thm:optimization:classical} and~\ref{thm:optimization:hessian}.
For classical sketch with model averaging, we see that 
when $\epsilon \leq \frac{1}{g}$, 
the bound on $f ({\W}^{\textrm{h}}) - f (\W^\star )$ is proportional to $\epsilon/g$.
From Lemma~\ref{lem:optimization:mdistance} we see that the distance
between ${\W}^{\textrm{c}}$ and $\W^\star$ also decreases accordingly.

\begin{theorem} [Classical Sketch with Model Averaging] \label{thm:optimization:classical_avg}
	Let $\beta = \frac{\|\X\|_2^2}{\|\X\|_2^2 + n \gamma} \leq 1$.
  For the four methods, let $s = \tilde{\OM} \big( \frac{ d }{\epsilon}  \big)$, and
  for uniform sampling, let $s = \OM \big( \frac{\mu  d \log d}{ \epsilon}  \big)$, then
  the inequality
	\begin{small}
	\begin{eqnarray*}
		f ({\W}^{\textrm{c}}) - f (\W^\star )
		& \leq & \beta \Big( \frac{\epsilon }{g} 
		+ \beta^2  \epsilon^2 \Big) \, f (\W^\star )
	\end{eqnarray*}
\end{small}%
	holds with probability at least 0.8.
  Here the randomness comes from the choice of sketching matrices.
\end{theorem}

For Hessian sketch with model averaging, if $\epsilon < \frac{1}{g}$,
then the bound on $f ({\W}^{\textrm{h}}) - f (\W^\star )$ is proportional to $\frac{\epsilon}{g}$.

\begin{theorem} [Hessian Sketch with Model Averaging]\label{thm:optimization:hessian_avg}
	Let $\beta = \frac{\|\X\|_2^2}{\|\X\|_2^2 + n \gamma} \leq 1$.
	For the four methods let $s = \tilde{\OM} \big( \frac{d }{\epsilon }  \big)$, and
	for uniform sampling let $s = \OM \big( \frac{\mu  d \log d }{\epsilon }  \big)$,
	then the inequality
	\begin{small}
		\begin{eqnarray*}
			f ({\W}^{\textrm{h}}) - f (\W^\star )
			& \leq & \beta^2  \, \Big( \frac{\epsilon}{g} + \epsilon^2 \Big) \,
			\Big( \frac{\|\Y\|_F^2}{n}  -  f (\W^\star ) \Big)  
		\end{eqnarray*}
	\end{small}%
	holds with probability at least 0.8.
  Here the randomness comes from the choice of sketching matrices.
\end{theorem}


\subsection{Model Averaging: Statistical Perspective}  \label{sec:main:statistical_avg}

Model averaging has the salutatory property of reducing the risks of the classical and Hessian sketches.
Our first result conducts a bias-variance decomposition for the averaged solution of the sketched MRR problem.

\begin{theorem} [Bias-Variance Decomposition] \label{thm:bias_var_decomp_avg}
	We consider the fixed design model \eqref{eq:rr_datamodel}.
  Decompose the risk function defined in \eqref{eq:def_risk} as
	\begin{small}
	\begin{eqnarray*}
		R (\W )
		& = & \bias^2 ( \W ) + \var ( \W ) .
	\end{eqnarray*}
\end{small}%
	The bias and variance terms are
	\begin{small}
		\begin{eqnarray*}
			\bias \big( \W^{\textrm{c}} \big)
			& = &  \gamma\sqrt{n}
			\bigg\| \frac{1}{g} \sum_{i=1}^g
			\big( \U^T \S_i \S_i^T \U + n\gamma \Si^{-2} \big)^{\dag} \Si^{-1} \V^T \W_0 \bigg\|_F , \\
			\var \big( \W^{\textrm{c}} \big)
			& = & \frac{\xi^2}{n}   \bigg\| \frac{1}{g} \sum_{i=1}^g
			\big( \U^T \S_i \S_i^T \U + n\gamma \Si^{-2} \big)^{\dag} \U^T \S_i \S_i^T  \bigg\|_F^2 ,\\
			\bias \big({\W}^{\textrm{h}}\big)
			& = & \gamma \sqrt{n} \bigg\| \frac{1}{g} \sum_{i=1}^g
			\big(  \Si^{-2} + \tfrac{\U^T \S_i \S_i^T \U - \I_\rho}{n \gamma} \big) 
			\big(\U^T \S_i \S_i^T \U + n \gamma \Si^{-2} \big)^{\dag} \Si \V^T \W_0 \bigg\|_F  , \\
			\var \big( {\W}^{\textrm{h}} \big)
			& = & \frac{\xi^2 }{n} \bigg\| \frac{1}{g} \sum_{i=1}^g
			\big(\U^T \S_i \S_i^T \U + n \gamma \Si^{-2} \big)^{\dag } \bigg\|_F^2 .
		\end{eqnarray*}
	\end{small}
\end{theorem}

Theorems~\ref{thm:biasvariance:classical_avg} and~\ref{thm:biasvariance:hessian_avg}
provide upper bounds on the bias and variance of averaged sketched MRR solutions for, respectively,
classical sketch and Hessian sketch.
We can contrast them with Theorems~\ref{thm:biasvariance:classical} and~\ref{thm:biasvariance:hessian}
to see the statistical benefits of model averaging.
Theorem~\ref{thm:biasvariance:classical_avg} shows that when $g \approx
\frac{n}{s}$, classical sketch with model averaging yields a solution with comparable bias and
variance to the optimal solution.

\begin{theorem} [Classical Sketch with Model Averaging] \label{thm:biasvariance:classical_avg}
  For the four sketching methods
	with $s = \tilde\OM \big( \frac{d}{\epsilon^2} \big)$,
	or uniform sampling with $s = \OM \big( \frac{\mu d \log d}{\epsilon^2} \big)$,
	the inequalities
	\begin{small}
		\begin{eqnarray*}
			\frac{\bias ( \W^{\textrm{c}} )}{\bias (\W^\star )}
			\; \leq \; 1 + \epsilon
			\qquad \textrm{ and } \qquad
			\frac{\var ( \W^{\textrm{c}} ) }{ \var (\W^\star ) }
			\; \leq \; \frac{n}{s} \bigg( \frac{ \sqrt{ 1 + \epsilon  }}{ \sqrt{h}} 
			+ \epsilon \bigg)^2 ,
		\end{eqnarray*}
	\end{small}%
	where $h = \min \{ g, \, \Theta (\frac{n}{s}) \}$,
	hold with probability at least 0.8.
  The randomness comes from the choice of sketching matrices.
\end{theorem}

Theorem~\ref{thm:biasvariance:hessian_avg} shows that model averaging
decreases the bias of Hessian sketch without increasing the variance.
For Hessian sketch without model averaging, recall that $\bias({\W}^{\textrm{h}})$ is larger than $\bias(\W^\star)$ by a factor of $\OM(\|\X\|_2^2/(n\gamma))$.
Theorem~\ref{thm:biasvariance:hessian_avg} shows that model averaging significantly reduces the bias.

\begin{theorem} [Hessian Sketch with Model Averaging] \label{thm:biasvariance:hessian_avg}
  For the four sketching methods with $s = \tilde\OM \big( \frac{d}{\epsilon^2} \big)$,
	or uniform sampling with $s = \OM \big( \frac{\mu d \log d}{\epsilon^2} \big)$,
	the inequalities
	\begin{small}
	\begin{eqnarray*}
		\frac{\bias ( {\W}^{\textrm{h}} )}{\bias (\W^\star ) }
		\; \leq \;
		1 + \epsilon + \Big( \frac{\epsilon }{ \sqrt{g} } + \epsilon^2  \Big)
		\frac{\| \X \|_2^2 }{n \gamma }  
		\qquad \textrm{ and } \qquad
		\frac{\var ( {\W}^{\textrm{h}} ) }{ \var (\W^\star ) }
		\; \leq \;
		1+\epsilon
	\end{eqnarray*}
	\end{small}%
	hold with probability at least 0.8.
  Here the randomness comes from the choice of sketching matrices.
\end{theorem}


\section{Experiments on Synthetic Data} \label{sec:experiment1}

We conduct experiments on synthetic data to verify our theory.
Section~\ref{sec:experiment1:setting} describes the data model and experiment settings.
Sections~\ref{sec:experiment1:opt} and~\ref{sec:experiment1:stats} empirically study
classical and Hessian sketch from the optimization and statistical perspectives, respectively,
to verify Theorems~\ref{thm:optimization:classical}, \ref{thm:optimization:hessian},
\ref{thm:biasvariance:classical}, and~\ref{thm:biasvariance:hessian}.
Sections~\ref{sec:experiment1:opt_avg} and~\ref{sec:experiment1:stats_avg} study
model averaging from the optimization and statistical perspectives, respectively,
to corroborate Theorems~\ref{thm:optimization:classical_avg}, \ref{thm:optimization:hessian_avg},
\ref{thm:biasvariance:classical_avg}, and~\ref{thm:biasvariance:hessian_avg}.


\subsection{Settings} \label{sec:experiment1:setting}

Following \citep{ma2014statistical,yang2015implementing},
we construct $\X = \U \diag (\si ) \V^T \in \RB^{n\times d}$ 
and $\y = \X \w_0 + \varepsi \in \RB^n$ in the following way.
\begin{itemize}
	\item 
	We take $\U$ be the matrix of left singular vectors of $\A \in \RB^{n\times d}$ which is constructed in the following way. (Note that $\A$ and $\X$ are different.)
	Let the rows of $\A$ be i.i.d.\ sampled from a multivariate $t$-distribution
	with covariance matrix $\C$ and $v=2$ degree of freedom, 
	where the $(i,j)$-th entry of $\C \in \RB^{d\times d}$ is $2\times 0.5^{|i-j|}$.
	Constructing $\A$ in this manner ensures that it has high row coherence. 
	\vspace{-1mm}
	\item
    Let the entries of $\bb\in \RB^{d}$ be equally spaced between $0$ and $-6$ and take
	$\sigma_i = 10^{b_i}$ for all $i\in [d]$.
	\vspace{-1mm}
	\item
    Let $\V \in \RB^{d\times d}$ be an orthonormal basis for the column range of a $d\times d$ standard Gaussian matrix.
	\vspace{-1mm}
	\item
	Let $\w_0 = [\1_{0.2d} ; 0.1 \, \1_{0.6d} ; \1_{0.2d}]$.
	\vspace{-1mm}
	\item
    Take the entries of $\varepsi \in \RB^n$ to be i.i.d.\ samples from the $\NM (0, \xi^2)$ distribution.
\end{itemize}
This construction ensures that $\X$ has high row coherence, and its
condition number is $\kappa (\X^T \X) = 10^{12}$.
Let $\S \in \RB^{n\times s}$ be any of the six sketching methods considered in this paper.
We fix $n = 10^5$, $d=500$, and $s=5,000$.
Since the sketching methods are randomized, we repeat each trial 10 times with idependent sketches and report averaged results.


\begin{figure}[t]
	\begin{center}
		\includegraphics[width=0.9\textwidth]{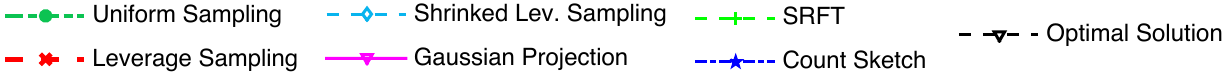}\\
		\includegraphics[width=0.96\textwidth]{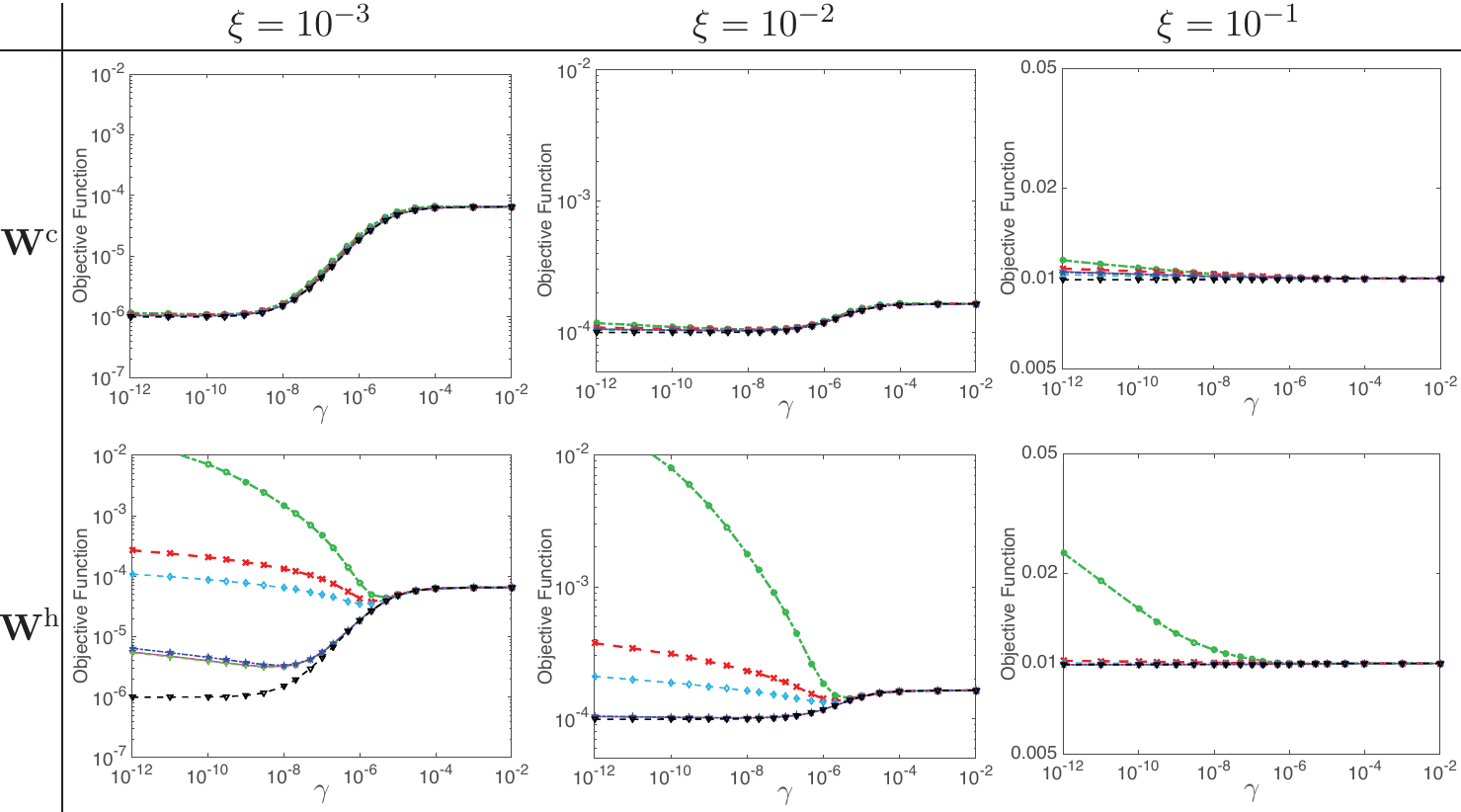}
		\vspace{-4mm}
	\end{center}
	\caption{An empirical study of classical and Hessian sketch from the optimization perspective.
		The $x$-axis is the regularization parameter $\gamma$ (log scale);
		the $y$-axis is the objective function values (log scale).
		Here $\xi$ is the standard deviation of the Gaussian noise added to the response.}
	\label{fig:obj_nb2}
\end{figure}

\subsection{Sketched MRR: Optimization Perspective} \label{sec:experiment1:opt}

We seek to empirically verify Theorems~\ref{thm:optimization:classical} and \ref{thm:optimization:hessian}
which study classical and Hessian sketches, respective, from the optimization perspective.
In Figure~\ref{fig:obj_nb2}, we plot the objective function value
$f (\w) = \frac{1}{n} \|\X \w - \y\|_2^2 + \gamma \|\w\|_2^2$
against $\gamma$, under different settings of $\xi$ (the standard deviation of the Gaussian noise added to the response).
The black curves correspond to the optimal solution $\w^\star$;
the color curves correspond to classical or Hessian sketch with different sketching methods.
The results verify our theory: 
the objective value of the solution from the classical sketch, $\w^{\textrm{c}}$, is always close to optimal; and
the objective value of the solution from the Hessian sketch, $\w^{\textrm{h}}$, is much worse than the optimal value when $\gamma$ is small 
and $\y$ is mostly in the column space of $\X$.


\begin{figure*}[t]
	\begin{center}
		\includegraphics[width=0.9\textwidth]{figure/legend1.pdf}\\
		\includegraphics[width=0.96\textwidth]{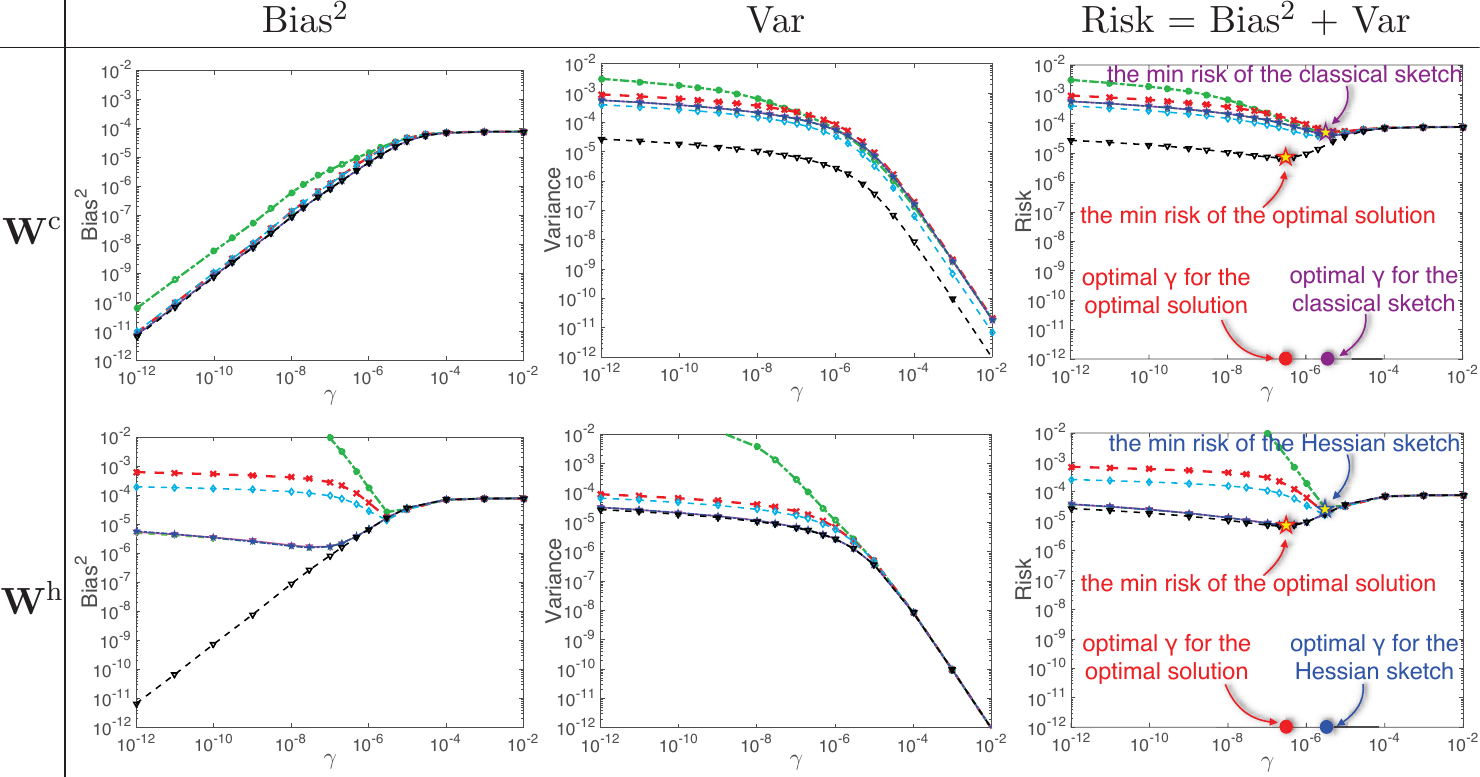}
		\vspace{-4mm}
	\end{center}
  \caption{An empirical study of classical sketch and Hessian sketch from the statistical perspective.
		The $x$-axis is the regularization parameter $\gamma$ (log-scale);
		the $y$-axes are respectively bias$^2$, variance, and risk (log-scale).
    We indicate the minimum risks and optimal choice of $\gamma$ in the plots.}
	\label{fig:risk_nb2}
\end{figure*}

\subsection{Sketched MRR: Statistical Perspective} \label{sec:experiment1:stats}

In Figure~\ref{fig:risk_nb2}, we plot the analytical expressions for the squared bias, variance, and risk
stated in Theorem~\ref{thm:bias_var_decomp} against the regularization parameter $\gamma$. 
Because these expressions involve the random sketching matrix $\S$, 
we randomly generate $\S$, repeat this procedure 10 times, 
and report the average of the computed squared biases, variances, and risks.
We fix $\xi = 0.1$ (the standard deviation of the Gaussian noise). The results of this experiment match our theory:
classical sketch magnified the variance,
and Hessian sketch increased the bias.
Even when $\gamma$ is fine-tuned, the risks of classical and Hessian sketch can be
much higher than those of the optimal solution.
Our experiment also indicates that classical and Hessian sketch require 
setting $\gamma$ larger than the best regularization parameter for the optimal solution $\W^\star$.

Classical and Hessian sketch do not outperform each other in terms of the risk.
When variance dominates bias, Hessian sketch is better in terms of the risk;
when bias dominates variance, classical sketch is preferable.
In the experiment yielding Figure~\ref{fig:risk_nb2},
Hessian sketch delivers lower risks than classical sketch.
This is not generally true: if we use a smaller $\xi$ (the standard deviation of the Gaussian noise), so that the variance is dominated by bias,
then classical sketch results in lower risks than Hessian sketch.


\begin{figure}
	\begin{center}
		\includegraphics[width=0.7\textwidth]{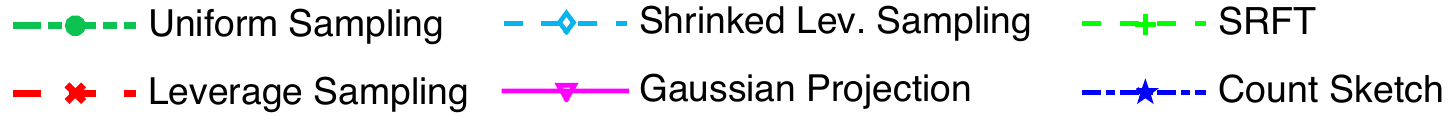}
    \subfigure[\textsf{Classical sketch with model averaging.}]{\includegraphics[width=0.49\textwidth]{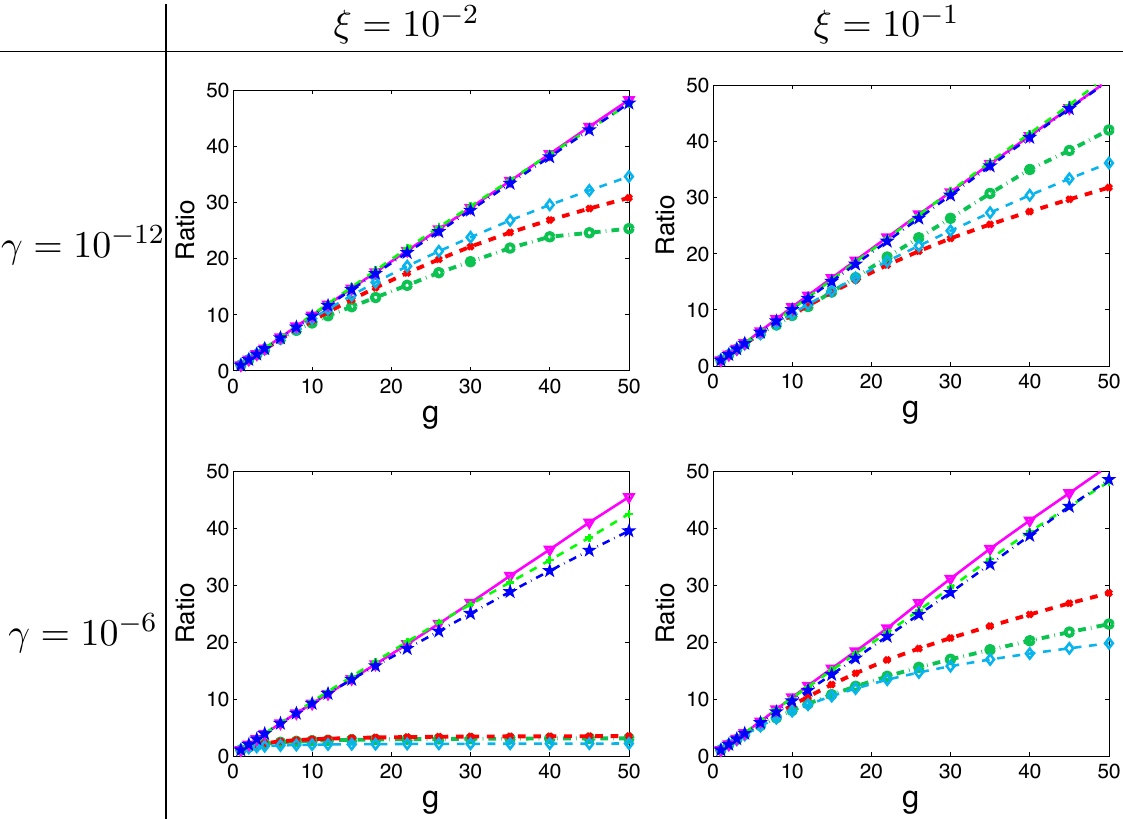}\label{fig:avg_obj_nb_classical}}
    \subfigure[\textsf{Hessian sketch with model averaging.}]{\includegraphics[width=0.49\textwidth]{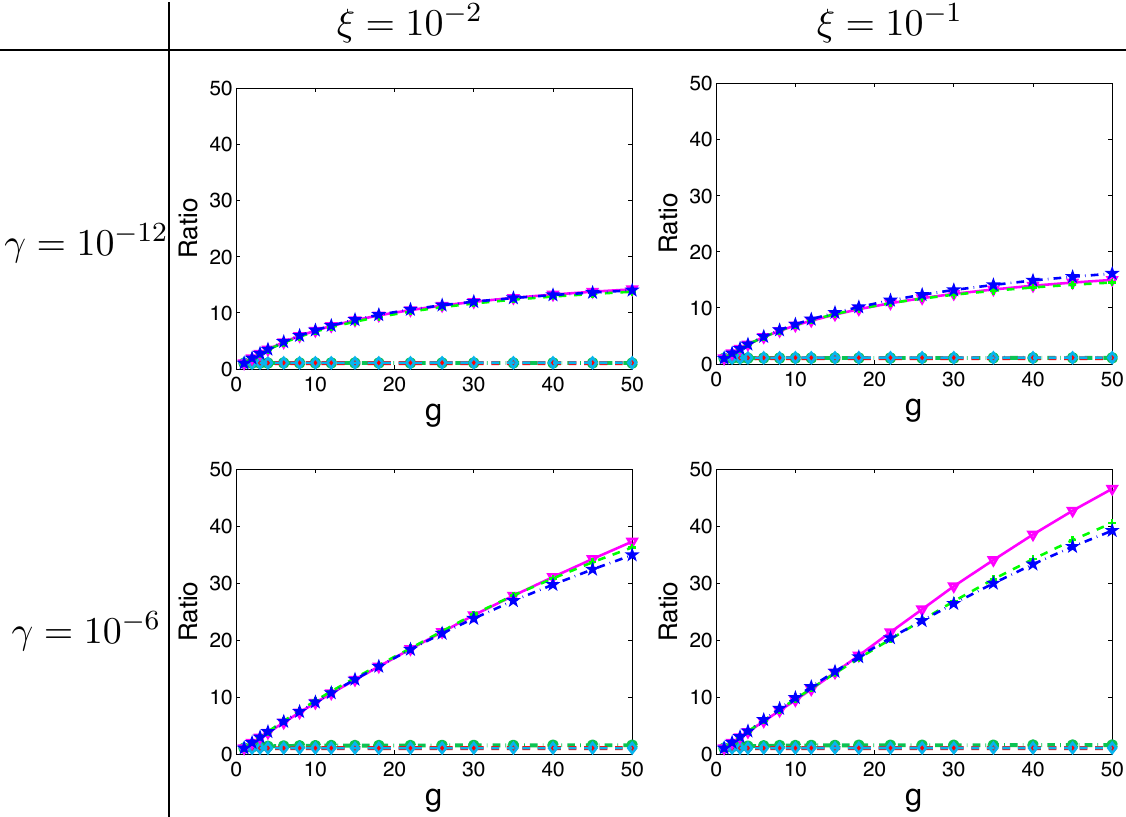}\label{fig:avg_obj_nb_hessian}}
	\end{center}
  \caption{An empirical study of model averaging from the optimization perspective.
    The $x$-axis is $g$, i.e., the number of models that are averaged.
		In \ref{fig:avg_obj_nb_classical}, 
		the $y$-axis is the ratio (log-scale) defined in \eqref{eq:ratio_obj_avg_tilde}.
		In \ref{fig:avg_obj_nb_hessian}, 
		the $y$-axis is the ratio (log-scale) defined in \eqref{eq:ratio_obj_avg_hat}.
    Here $\gamma$ is the regularization parameter and $\xi$ is the standard deviation of the Gaussian noise.}
\end{figure}

\subsection{Model Averaging: Optimization Objective} \label{sec:experiment1:opt_avg}

We consider different noise levels by setting $\xi = 10^{-2}$ or $10^{-1}$, where $\xi$ is defined in Section~\ref{sec:experiment1:setting}
as the standard deviation of the Gaussian noise in the response vector $\y$.
We calculate the objective function values $f ({\w}^{\textrm{c}}_{[g]})$ and $f ({\w}^{\textrm{h}}_{[g]})$
for different settings of $g$, $\gamma$.
We use different methods of sketching at the fixed sketch size $s=5,000$.

Theorem~\ref{thm:optimization:classical_avg} indicates that for large $s$,
e.g., Gaussian projection with $s = \tilde{\OM} \big( \frac{ d }{\epsilon}  \big)$,

\begin{eqnarray} \label{eq:experiment:avg:obj_classical1}
f \big({\w}^{\textrm{c}}_{[g]} \big) - f \big(\w^\star \big)
& \leq & \beta \big( \tfrac{\epsilon}{g} 
+ \beta^2 \epsilon^2 \big) \, f (\w^\star )  ,
\end{eqnarray}
where $\beta = \frac{\|\X\|_2^2}{\|\X\|_2^2 + n \gamma} \leq 1$.
In Figure~\ref{fig:avg_obj_nb_classical}
we plot the ratio
\begin{eqnarray} \label{eq:ratio_obj_avg_tilde}
\tfrac{f ({\w}^{\textrm{c}}_{[1]}) - f (\w^\star) }{f ({\w}^{\textrm{c}}_{[g]})  - f (\w^\star)} 
\end{eqnarray}
against $g$.
Rapid growth of this ratio indicates that model averaging is highly effective.
The results in Figure~\ref{fig:avg_obj_nb_classical}
indicate that model averaging significantly improves the accuracy
as measured by the objective function value.
For the three random projection methods, the growth rate of this ratio is almost linear in $g$.
In Figure~\ref{fig:avg_obj_nb_classical},
we observe that the regularization parameter $\gamma$ affects the ratio \eqref{eq:ratio_obj_avg_tilde}.
The ratio grows faster when $\gamma = 10^{-12}$ than when $\gamma = 10^{-6}$.
This phenomenon is not explained by our theory.

\begin{figure}
	\begin{center}
		\includegraphics[width=0.9\textwidth]{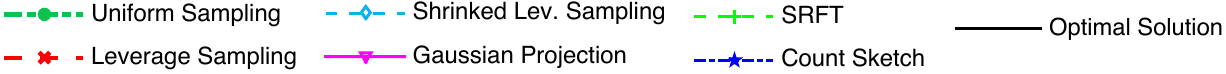}
		\subfigure[\textsf{The variance $\var ({\w}^{\textrm{c}}_{[g]} )$.}]{\includegraphics[width=0.49\textwidth]{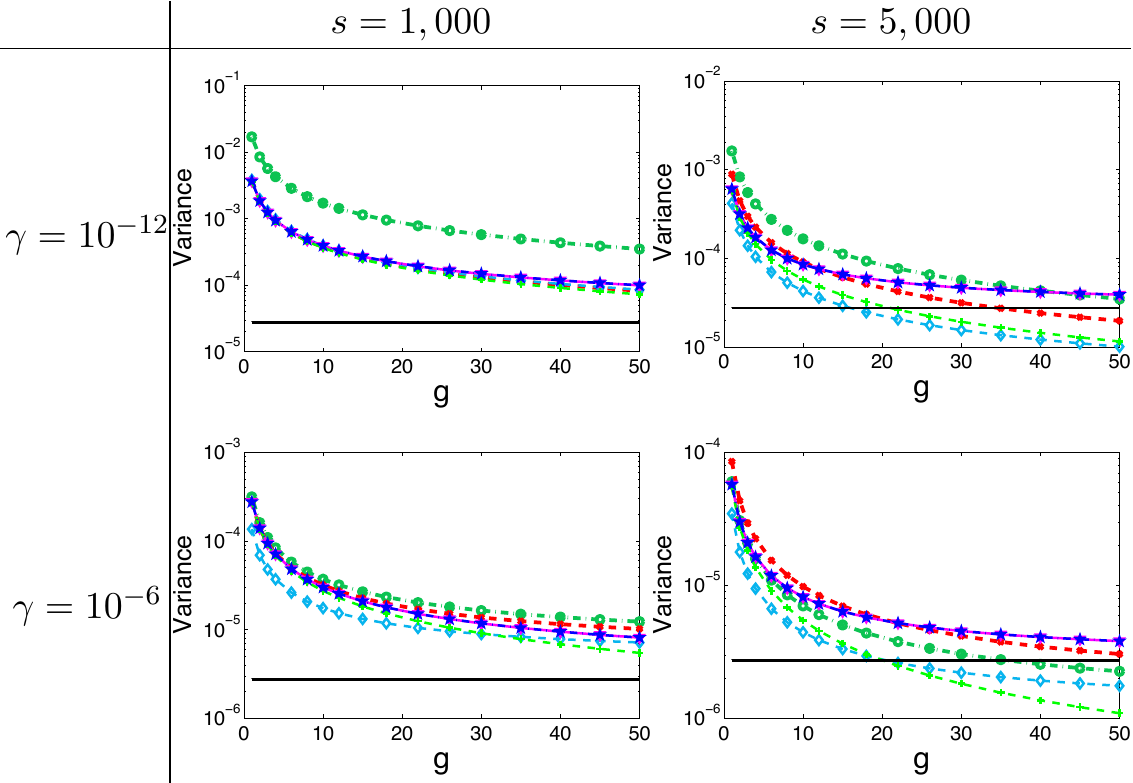}\label{fig:avg_var_nb_classical}}
		\subfigure[\textsf{The ratio 
			$\tfrac{\var ({\w}^{\textrm{c}}_{[1]} )}{\var ({\w}^{\textrm{c}}_{[g]} )}$.}]{\includegraphics[width=0.49\textwidth]{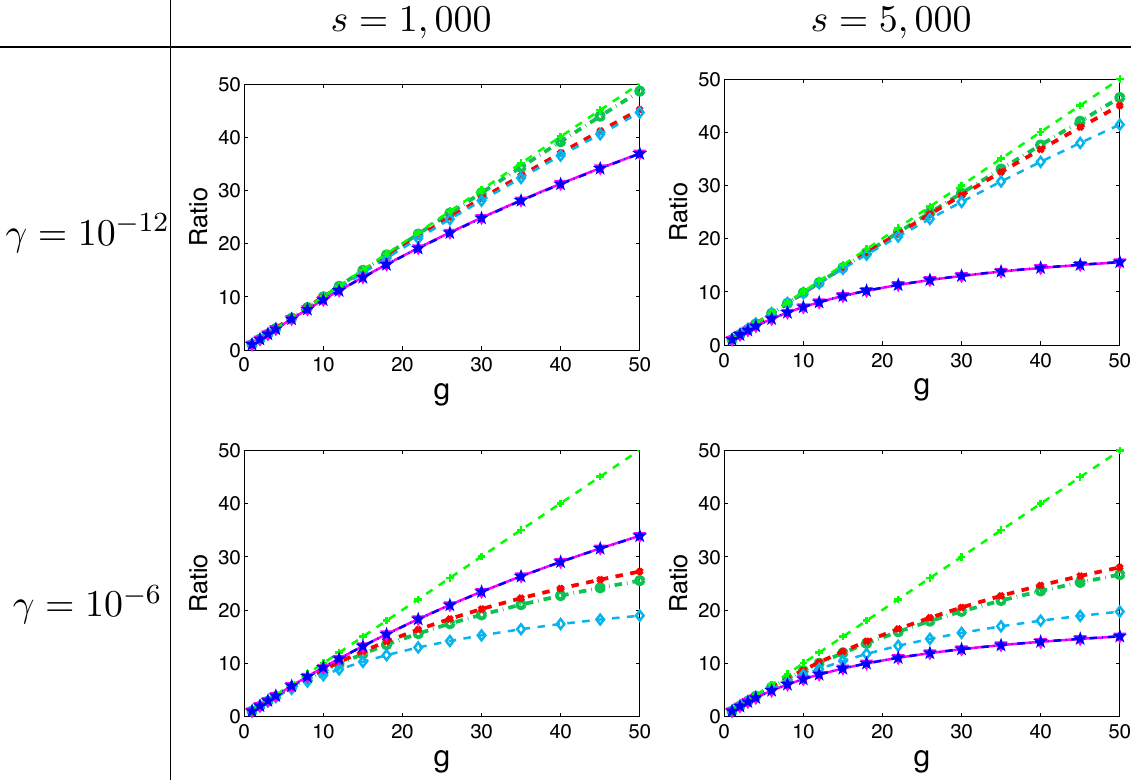}\label{fig:avg_varratio_nb_classical}}
	\end{center}
  \caption{An empirical study of the variance of classical sketch with model averaging.
    The $x$-axis is $g$, i.e., the number of models that are averaged.
		In \ref{fig:avg_var_nb_classical}, 
		the $y$-axis is the variance $\var ({\w}^{\textrm{c}}_{[g]} )$
    (log scale) defined in Theorem~\ref{thm:bias_var_decomp_avg}.
		In \ref{fig:avg_varratio_nb_classical}, 
		the $y$-axis is the ratio 
		$\tfrac{\var ({\w}^{\textrm{c}}_{[1]} )}{\var ({\w}^{\textrm{c}}_{[g]} )}$.
    Here $\gamma$ is the regularization parameter and $s$ is the sketch size.}
\end{figure}

Theorem~\ref{thm:optimization:hessian_avg} shows that
for large sketch size $s$,
e.g., Gaussian projection with $s = \tilde{\OM} \big( \frac{d }{\epsilon }  \big)$,

\begin{eqnarray*}
	f ({\w}^{\textrm{h}}) - f (\w^\star )
	& \leq & \beta^2 \, \Big( \tfrac{\epsilon}{g} + {\epsilon^2 } \Big) \,
	\Big( \tfrac{\|\y \|_2^2}{n}  -  f (\w^\star ) \Big)  ,
\end{eqnarray*}
where $\beta = \frac{\|\X\|_2^2}{\|\X\|_2^2 + n \gamma} \leq 1$.
In Figure~\ref{fig:avg_obj_nb_hessian},
we plot the ratio
\begin{eqnarray} \label{eq:ratio_obj_avg_hat}
\tfrac{f ({\w}^{\textrm{h}}_{[1]}) - f (\w^\star) }{f ( {\w}^{\textrm{h}}_{[g]})  - f (\w^\star)} 
\end{eqnarray}
against $g$.
Rapid growth of this ratio indicates that model averaging is highly effective.
Our empirical results indicate that the growth rate of this ratio is moderately rapid for very small $g$ and very slow for large $g$.


\subsection{Model Averaging: Statistical Perspective} \label{sec:experiment1:stats_avg}

We empirically study model averaging from the statistical perspective.
We calculate the bias and variance $\bias ( {\w}^{\star} )$, $\var ({\w}^{\star} )$ of the optimal MRR solution 
according to Theorem~\ref{thm:bias_var_decomp}
and the bias and variance $\bias ( {\w}^{\textrm{c}}_{[g]} )$, $\var ({\w}^{\textrm{c}}_{[g]} )$ 
and $\bias ( {\w}^{\textrm{h}}_{[g]} )$, $\var ({\w}^{\textrm{h}}_{[g]} )$ of, respectively, the model averaged classical
sketch solution and the model averaged Hessian sketch solution according to Theorem~\ref{thm:bias_var_decomp_avg}.

\begin{figure*}
	\begin{center}
		\includegraphics[width=0.7\textwidth]{figure/legend3.pdf}\\
		\includegraphics[width=0.95\textwidth]{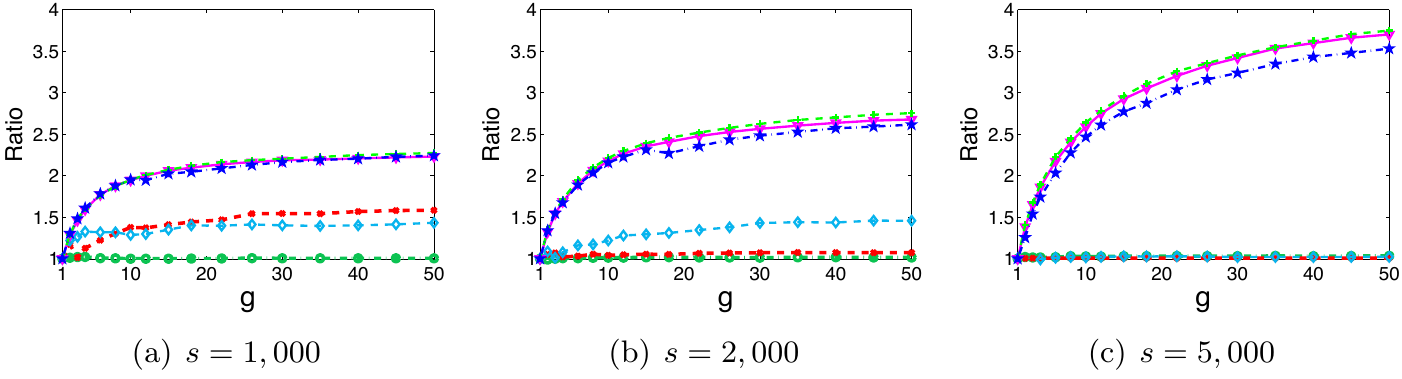}
	\end{center}
  \caption{An empirical study of the bias of Hessian sketch with model averaging.
    The $x$-axis is $g$, the number of models being averaged;
		the $y$-axis is the ratio \eqref{eq:experiment:avg:hessian:ratio}.}
	\label{fig:avg_biasratio_nb_hessian}
\end{figure*}

\subsubsection{Classical Sketch}

Theorem~\ref{thm:biasvariance:classical_avg} indicates that
for large enough $s$, e.g., Gaussian projection with $s = \tilde\OM \big( \frac{d}{\epsilon^2} \big)$,
with high probability
\begin{align*}
\tfrac{\bias ( {\w}^{\textrm{c}}_{[g]} )}{\bias (\w^\star )}
\leq 1 + \epsilon 
\qquad \textrm{and} \qquad
\tfrac{\var ( {\w}^{\textrm{c}}_{[g]} ) }{ \var (\w^\star ) }
\leq \tfrac{n}{s} \Big( \sqrt{\tfrac{1 + \epsilon }{h} } 
+ \epsilon \Big)^2,
\end{align*}
where $h = \min \{ g, \, \Theta (\frac{n}{s}) \}$.
This result implies that model averaging decreases the variance of classical sketch without significantly changing the bias. 
We conduct experiments to verify this point.

In Figure \ref{fig:avg_var_nb_classical} we plot the variance $\var ({\w}^{\textrm{c}}_{[g]} )$ against $g$;
the variance of the optimal solution $\w^\star$ is depicted for comparison.
Clearly, the variance drops as $g$ grows.
In particular, when $s$ is big ($s=5,000$) and $g$ exceeds $\frac{n}{s}$
($=\frac{100,000}{5,000} = 20$),
$\var ({\w}^{\textrm{c}}_{[g]} )$ can be even lower than $\var (\w^\star )$.

To more clearly decrease the impact of model averaging on the variance, 
in Figure~\ref{fig:avg_varratio_nb_classical} we plot the ratio $\frac{\var ({\w}^{\textrm{c}}_{[1]} )}{\var ({\w}^{\textrm{c}}_{[g]} )}$ against $g$.
According to Theorem~\ref{thm:biasvariance:classical_avg},
this ratio grows linearly in $g$ when $s$ is at least $\tilde{\OM} (d g)$, and otherwise is sublinear in $g$.
This claim is verified by the empirical results in Figure~\ref{fig:avg_varratio_nb_classical}.

When $\bias ({\w}^{\textrm{c}}_{[g]} )$ is plotted as a function of $g$, the curves are almost horizontal,
indicating that, as expected, {\it the bias is insensitive to the number of models $g$}.
We do not show such plots because these nearly horizontal curves are not interesting.

\subsubsection{Hessian Sketch}

Theorem~\ref{thm:biasvariance:hessian_avg} indicates that for large enough $s$, e.g., Gaussian projection with
$s = \tilde\OM \big( \frac{d}{\epsilon^2} \big)$,
the inequalities
\begin{align*}
\tfrac{\bias ( \w^{\textrm{h}}_{[g]} )}{\bias (\w^\star ) }
\; \leq \;
1 + \epsilon + \Big( \tfrac{\epsilon }{ \sqrt{g}} + \epsilon^2  \Big) \tfrac{\| \X \|_2^2 }{n \gamma } 
\qquad \textrm{and} \qquad
\tfrac{\var ( {\w}^{\textrm{h}}_{[g]} ) }{ \var (\w^\star ) }
\; \leq \;
1+\epsilon
\end{align*}
hold with high probability.
That is, model averaging improves the bias without affecting the variance.
The bound
\[
\tfrac{\bias ( {\w}^{\textrm{h}}_{[g]} ) - \bias (\w^\star ) }{\bias (\w^\star ) }
\; \leq \; \epsilon + \Big( \tfrac{\epsilon }{ \sqrt{g} } + \epsilon^2  \Big)
\tfrac{\| \X \|_2^2 }{n \gamma }
\]
indicates that if $n \gamma$ is much smaller than $\| \X \|_2^2$
and $\epsilon \leq \frac{1}{ \sqrt{g} }$, or equivalently, $s $ is at least $ \tilde{\OM} (d g)$,
then the ratio is proportional to $\frac{\epsilon}{ \sqrt{g} }$.

To verify Theorem~\ref{thm:biasvariance:hessian_avg}, we
set $\gamma$ very small---$\gamma = 10^{-12}$---and vary $s$ and $g$.
In Figure~\ref{fig:avg_biasratio_nb_hessian} we plot the ratio
\begin{small}
\begin{eqnarray} \label{eq:experiment:avg:hessian:ratio}
\tfrac{\bias ( {\w}^{\textrm{h}}_{[1]} ) - \bias (\w^\star )}{\bias ( {\w}^{\textrm{h}}_{[g]} ) - \bias (\w^\star ) } ,
\end{eqnarray}
\end{small}%
by fixing $\gamma = 10^{-12}$ and varying $s$ and $g$.
The theory indicates that for large sketch size $s = \tilde{\OM} (d g^2)$, this ratio should grow nearly linearly in $g$.
Figure~\ref{fig:avg_biasratio_nb_hessian} shows that only for large $s$ and very small $g$, the growth is near linear in $g$;
this verifies our theory.

When we similarly plot $\var \big( {\w}^{\textrm{h}}_{[g]} \big)$ against $g$, we
observe that $\var \big( {\w}^{\textrm{h}}_{[g]} \big)$ remains nearly unaffected as $g$ grows from 1 to 50.
Since the curves of the variance against $g$ are almost horizontal lines,
we do not show this plot in the paper.


\begin{figure}[t]
	\begin{center}
		\includegraphics[width=0.5\textwidth]{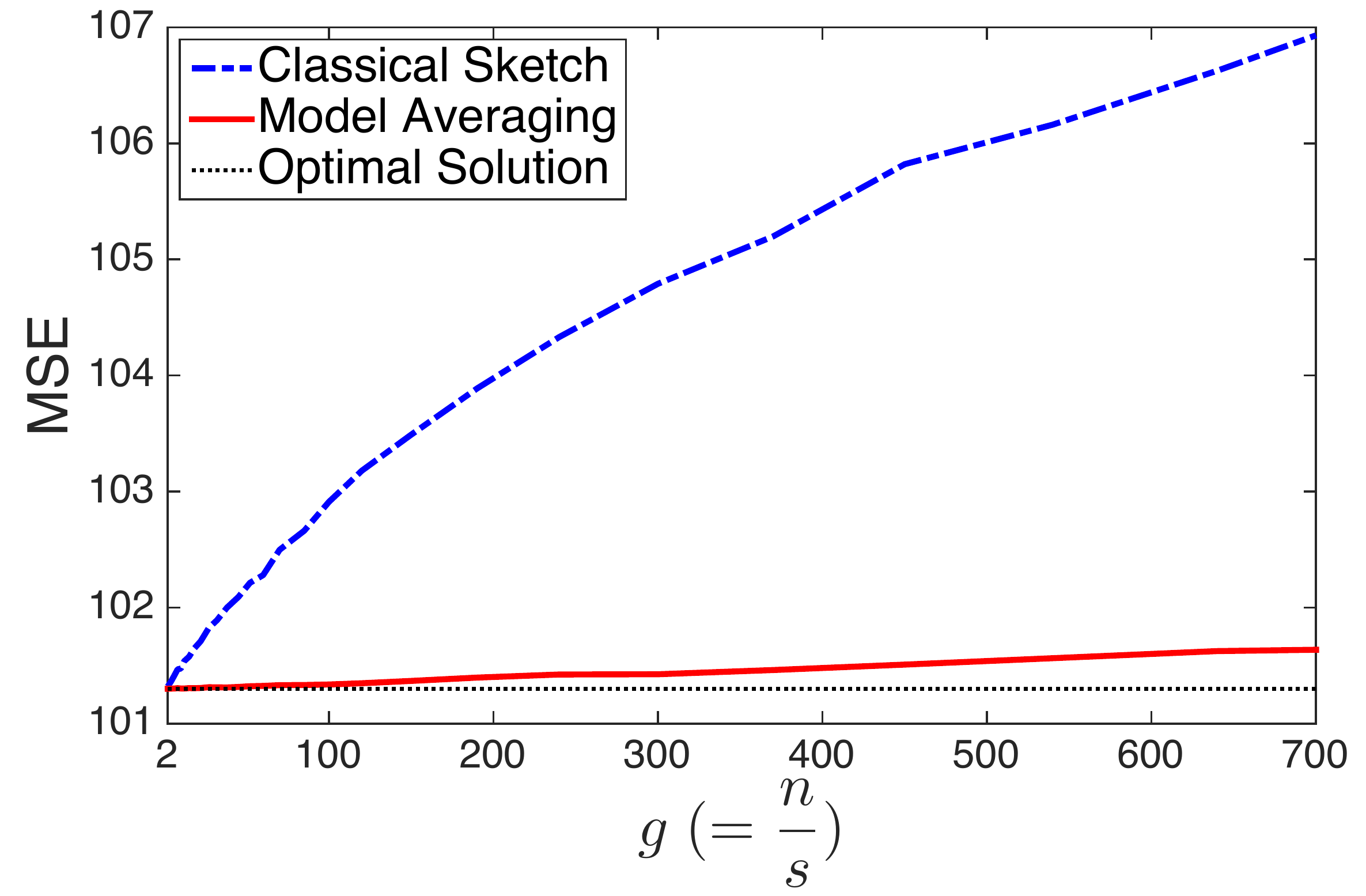}
		\vspace{-5mm}
	\end{center}
  \caption{Prediction performance of classical sketch with and without model averaging on the Year Prediction data set.
    The $x$-axis is $g$, the number of data partitions, and 
		the $y$-axis is the mean squared error (MSE) on the test set.}
	\label{fig:realdata_mse}
	\vspace{-5mm}
\end{figure}

\section{Model Averaging Experiments on Real-World Data} \label{sec:experiments_real}

In Section~\ref{sec:introduction} we mentioned that
in the distributed setting where 
the feature-response pairs $(\x_1, \y_1) , \cdots , (\x_n, \y_n ) \in \RB^{d\times m}$
are randomly and uniformly partitioned across $g$ machines,\footnote{If the samples are i.i.d., then any deterministic partition is essentially a uniformly randomly distributed partition. Otherwise, we can invoke a \textsf{Shuffle} operation, which is supported by systems such as Apache Spark~\citep{zaharia2010spark}, to make the partitioning uniformly randomly distributed.}
classical sketch with model averaging requires only one
round of communication, and is therefore a communication-efficient algorithm that can be 
used to: (1) obtain an approximate solution of the MRR problem with risk comparable to
a batch solution, and (2) obtain a low-precision solution of the MRR optimization problem
that can be used as an initializer for more communication-intensive optimization algorithms.
In this section, we demonstrate both applications.

We use the Million Song Year Prediction data set,
which has $ 463,715$ training samples and $ 51,630$ test samples
with $90$ features and one response.
We normalize the data by shifting the responses to have zero mean and scaling
the range of each feature to $[-1, 1]$.
We randomly partition the training data into $g$ parts, which amounts to
uniform row selection with sketch size $s = \frac{n}{g}$.

\subsection{Prediction Error}

We tested the prediction performance of sketched ridge regression
by implementing classical sketch with model averaging in PySpark \citep{zaharia2010spark}.\footnote{The code is 
	available at https://github.com/wangshusen/SketchedRidgeRegression.git}
We ran our experiments using PySpark in local mode; the experiments
proceeded in three steps:
(1) use five-fold cross-validation to determine the regularization parameter $\gamma$;
(2) learn the model $\w$ using the selected $\gamma$;
and (3) use $\w$ to predict on the test set and record the mean squared errors (MSEs).
These steps map cleanly onto the Map-Reduce programming model used by PySpark.

In Figure~\ref{fig:realdata_mse}, we plot the test MSE against $g=\frac{n}{s}$.
As $g$ grows, the sketch size $s = \frac{n}{g}$ decreases,
so the performance of classical sketch deteriorates.
However classical sketch with model averaging always has test MSE comparable to the optimal solution.

\subsection{Optimization Error}

We mentioned earlier that classical sketch with or without model averaging
can be used to initialize optimization algorithms for solving MRR problems.
If $\w$ is initialized with zero-mean random variables or deterministically with zeros, then
$\EB \big[ \|\w - \w^\star \|_2 / \|\w^\star\|_2 \big] \geq 1$.
Any $\w$ with the above ratio substantially smaller than $1$ provides a better initialization.
We implemented classical sketch with and without model averaging
in Python and calculated the above ratio on the training set of the Year Prediction data set; 
to estimate the expectation, we repeated the procedure $100$ times and report the average of the ratios.

\begin{figure}[t]
	\begin{center}
    \subfigure[Classical sketch]{
			\includegraphics[width=0.45\textwidth]{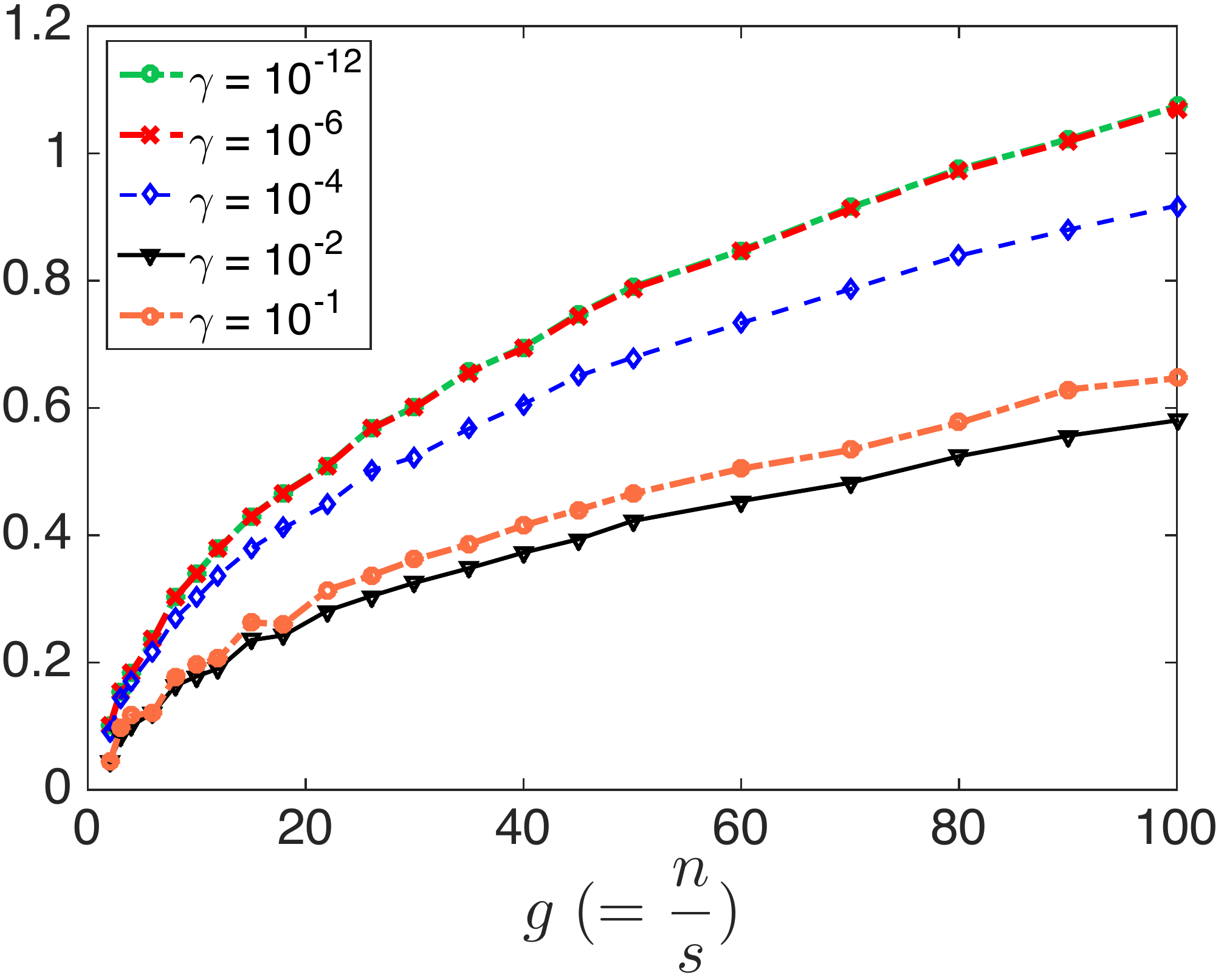}}~~~
    \subfigure[Classical sketch with model averaging]
		{\includegraphics[width=0.45\textwidth]{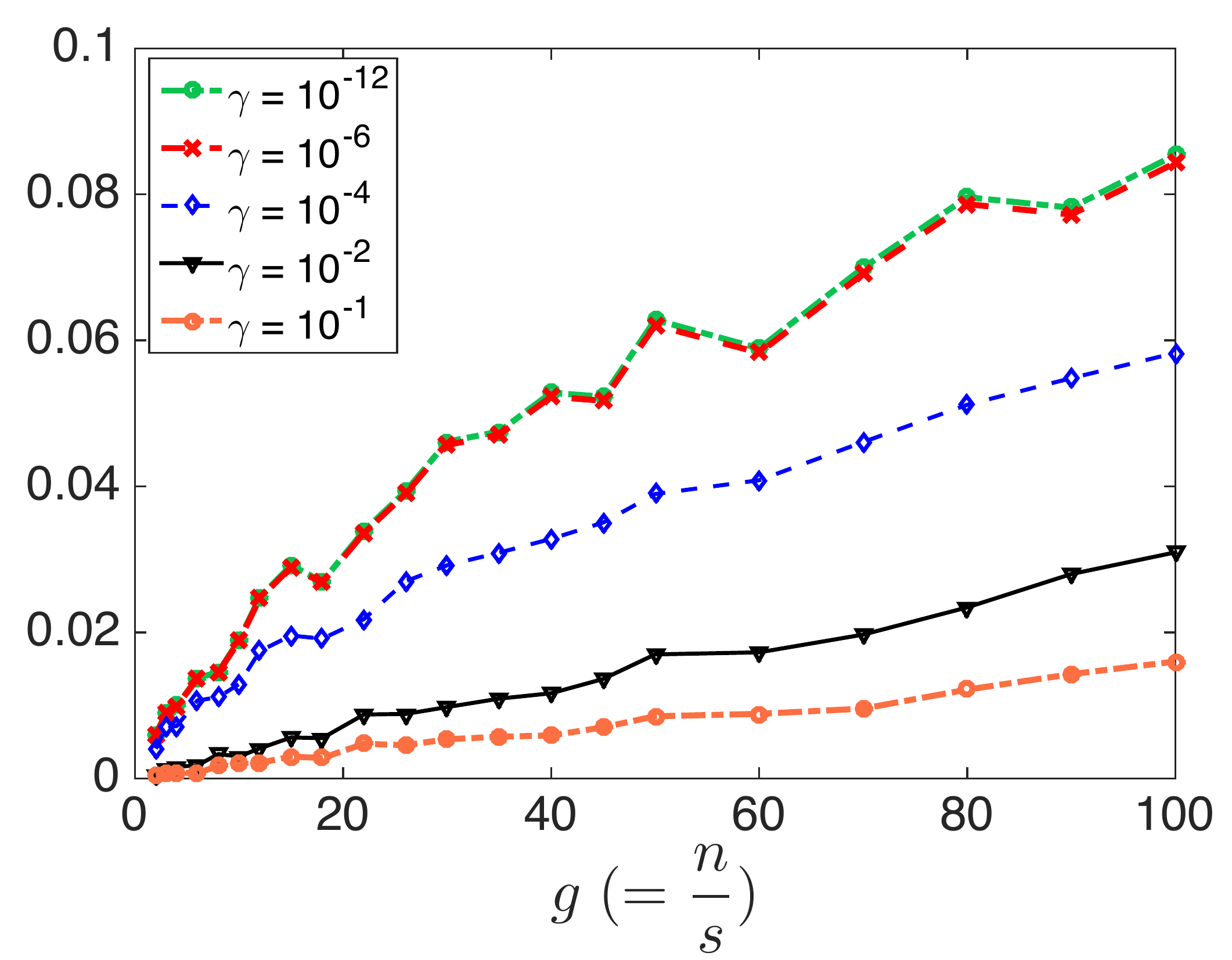}}
		\vspace{-5mm}
	\end{center}
  \caption{Optimization performance of classical sketch with and without model averaging.
    The $x$-axis is $g$, the number of data partitions, and 
		the $y$-axis is the ratio $\frac{\|\w - \w^\star \|_2}{\|\w^\star\|_2}$.}
	\label{fig:realdata_dist}
	\vspace{-5mm}
\end{figure}

In Figure \ref{fig:realdata_dist}, we plot the average of the ratio
$\frac{\|\w - \w^\star \|_2}{\|\w^\star\|_2}$ against $g$ for different settings of the regularization parameter $\gamma$.
Clearly, classical sketch does not give a good initialization unless $g$ is small
(equivalently, the sketch size $s = \frac{n}{g}$ is large).
In contrast, the averaged solution is always close to $\w^\star$.


\input{proofsketch}


\section{Conclusions}

We studied sketched matrix ridge regression (MRR) from the optimization and
statistical perspectives. Using classical sketch, by taking a large enough
sketch, one can obtain an $\epsilon$-accurate approximate solution.
Counterintuitively and in contrast to classical sketch, the relative
error of Hessian sketch increases as the responses $\Y$ are better
approximated by linear combinations of the columns of $\X$. 
Both classical and Hessian sketches can have statistical risks
that are worse than the risk of the optimal solution by an order of magnitude. 

We proposed the use of model averaging to attain better optimization and statistical properties. 
We have shown that model averaging leads to substantial
improvements in the theoretical error bounds, suggesting applications in
distributed optimization and machine learning.
We also empirically verified its practical benefits.

Our fixed-design statistical analysis has limitations.
We have shown that the classical sketch and Hessian sketch can significantly
increase the in-sample statistical risk, which implies large training error, and
that model averaging can alleviate such problems.
However, our statistical results are not directly applicable to an unseen test sample.
We conjecture that the generalization error can be bounded by following the random design analysis of \citet{hsu2014random}, which is left as future work.


\acks{We thank the anonymous reviewers and Serena Ng for their helpful suggestions.
	We thank the Army Research Office and the Defense Advanced Research Projects Agency for partial support of this work.}


\appendix

\input{proof}

\vskip 0.2in
\bibliography{matrix}

\end{document}

%% file: proofsketch.tex
\section{Sketch of Proof}\label{sec:proofsketch}

In this section, we outline the proofs of our main results.
The complete details are provided in the appendix.
Section~\ref{sec:proofsketch:sketch} recaps several relevant properties of matrix sketching.
Section~\ref{sec:proofsketch:sketch_avg} establishes certain properties of averages of sketches;
these results are used to analyze the application of model averaging to the MRR problem. 
Sections~\ref{sec:proofsketch:opt} to \ref{sec:proofsketch:stat_avg}
provide key structural results on sketched solutions to the MRR problem constructed with or without model averaging.

Our main results in Section~\ref{sec:main}
(Theorems~\ref{thm:optimization:classical}, \ref{thm:optimization:hessian},
\ref{thm:biasvariance:classical}, \ref{thm:biasvariance:hessian},
\ref{thm:optimization:classical_avg}, \ref{thm:optimization:hessian_avg},
\ref{thm:biasvariance:classical_avg}, and~\ref{thm:biasvariance:hessian_avg})
follow directly from the relevant properties of matrix sketching and the structural results for solutions to the sketched MRR problem.
Table~\ref{tab:proofsketch} summarizes the dependency relationships among these theorems.
For example, Theorem~\ref{thm:optimization:classical}, 
which studies classical sketching from the optimization perspective,
is one of our main theorems and is proven 
using Theorems~\ref{thm:properties} and \ref{thm:optimization:wtilde}.

\begin{table}[!h]\setlength{\tabcolsep}{0.3pt}
  \caption{An overview of our results and their dependency relationships.}
	\label{tab:proofsketch}
	\begin{center}
		\begin{small}
			\begin{tabular}{c c c c}
				\hline
				~~~{\bf Main Theorems}~~~ &~~~{\bf Solution}~~~
        &~~~{\bf Perspective}~~~&~~~{\bf Prerequisites}~~~ \\
				\hline
				~~~Theorem~\ref{thm:optimization:classical}~~~ 
				&~~~classical~~~ 
				&~~~optimization~~~
				& ~~~Theorems~\ref{thm:properties} and \ref{thm:optimization:wtilde}~~~ \\
				~~~Theorem~\ref{thm:optimization:hessian}~~~
				& Hessian
				& optimization
				& ~~~Theorems~\ref{thm:properties} and \ref{thm:optimization:what}~~~ \\
				~~~Theorem~\ref{thm:biasvariance:classical}~~~
				& classical
				& statistical
				& ~~~Theorems~\ref{thm:properties}, \ref{thm:biasvariance:wtilde_lower2}, 
				\ref{thm:biasvariance:wtilde}, \ref{thm:biasvariance:wtilde_lower1}~~~ \\
				~~~Theorem~\ref{thm:biasvariance:hessian}~~~
				& Hessian
				& statistical
				& ~~~Theorems~\ref{thm:properties} and \ref{thm:biasvariance:what}~~~ \\
				~~~Theorem~\ref{thm:optimization:classical_avg}~~~
				& ~~~classical, averaging~~~
				& optimization
				& ~~~Theorems~\ref{thm:properties_avg} and \ref{thm:optimization:wtilde_avg}~~~ \\
				~~~Theorem~\ref{thm:optimization:hessian_avg}~~~
				& ~~~Hessian, averaging~~~
				& optimization
				& ~~~Theorems~\ref{thm:properties_avg} and \ref{thm:optimization:what_avg}~~~ \\
				~~~Theorem~\ref{thm:biasvariance:classical_avg}~~~
				& ~~~classical, averaging~~~
				& statistical
				& ~~~Theorems~\ref{thm:properties_avg} and \ref{thm:biasvariance:wtilde_avg}~~~ \\
				~~~Theorem~\ref{thm:biasvariance:hessian_avg}~~~
				& ~~~Hessian, averaging~~~
				& statistical
				& ~~~Theorems~\ref{thm:properties_avg} and \ref{thm:biasvariance:what_avg}~~~ \\
				\hline
			\end{tabular}
		\end{small}
	\end{center}
\end{table}


\subsection{Properties of Matrix Sketching} \label{sec:proofsketch:sketch}

Our analysis of the performance of solutions to the sketched MRR problem draws heavily on the three key properties defined in
Assumption~\ref{assumption:sketching}.
Theorem~\ref{thm:properties} establishes that the six sketching methods considered in this paper
indeed enjoy the three key properties under certain conditions.
Finally, Theorem~\ref{thm:biasvariance:wtilde_lower2} establishes the lower bounds of $\|\S\|_2^2$ that are used to 
prove the lower bounds on the variance of sketched MRR solutions in Theorem~\ref{thm:biasvariance:classical}.

\begin{assumption}\label{assumption:sketching}
	Let $\eta , \epsilon \in (0, 1)$ be fixed parameters.
  Let $\B$ be any fixed matrix of conformal shape, $\rho = \rk (\X)$,
  and $\U \in \RB^{n\times \rho}$ be an orthonormal basis for the column span of $\X$.
  Let $\S \in \RB^{n\times s}$ be a sketching matrix, where $s$ depends on $\eta$ and/or $\epsilon$.
  Throughout this paper, we assume that $\S$ satisfies the following properties with a probability that depends on $s$:
	\begin{enumerate}[label=\textnormal{(\arabic*)}]
		\item[\mylabel{assumption:1:1}{1.1}]
		$\big\|\U^T \S \S^T \U - \I_\rho \big\|_2 \leq {\eta}$
		\quad(Subspace Embedding Property);
		\item[\mylabel{assumption:1:2}{1.2}] 
		$\big\|\U^T \S \S^T \B - \U^T \B \big\|_F^2 \leq {\epsilon} \|\B\|_F^2$
		\quad(Matrix Multiplication Property);
		\item[\mylabel{assumption:1:3}{1.3}] 
		When $s < n$, $\|\S \|_2^2  \leq \frac{\theta n}{s}$ \; for some constant $\theta$
		\quad(Bounded Spectral Norm Property).
	\end{enumerate}%
\end{assumption}

The subspace embedding property requires that sketching preserves the inner products between the columns of a matrix with orthonormal columns.
Equivalently, it ensures that the singular values of
any sketched column-orthonormal matrix are all close to one.
The subspace embedding property implies that, in particular, the squared norm of $\mathbf{S}\mathbf{x}$ is close to that of $\mathbf{x}$ for 
any $n$-dimensional vector in a fixed $\rho$-dimensional subspace. A dimension counting argument suggests that since $\mathbf{S}\mathbf{x}$ is an $s$-dimensional vector, its length must be scaled
by a factor of $\sqrt{\tfrac{n}{s}}$ to ensure that this consequence of the subspace embedding property holds. The bounded spectral norm property requires that the spectral
norm of $\S$ is not much larger than this rescaling factor of $\sqrt{\tfrac{n}{s}}$.

\begin{Remark}
  The first two assumptions were identified in \citep{mahoney2011ramdomized}
  and are the relevant structural conditions that allow 
  strong results from the optimization perspective.
  The third assumption is new, but \citet{ma2014statistical,raskutti2015statistical}
	demonstrated that some sort of additional condition is necessary to obtain strong results from the statistical perspective.
\end{Remark}

\begin{Remark}
	We note that  $\U^T \U = \I_{\rho}$, and thus Assumption~\ref{assumption:1:1} can be expressed in the form of an approximate matrix multiplication bound~\citep{drineas06fastmonte1}.  
	We call it the Subspace Embedding Property since, as first highlighted in~\citet{drineas2006sampling}, 
  this subspace embedding property is the key result necessary to obtain high-quality sketching algorithms for regression and related problems.
\end{Remark}

\begin{table}\setlength{\tabcolsep}{0.3pt}
  \caption{The two middle columns provide an upper bound on the sketch size $s$ needed to satisfy the 
    subspace embedding property and the matrix multiplication property, respectively, under the different 
    sketching modalities considered;
    the right column lists the parameter $\theta$ with which the bounded spectral norm property holds.
    These properties hold with constant probability for the indicated values of $s$.
		Here $\tau$ is defined in \eqref{eq:def:appro_lev} and reflects the quality
    of the approximation of the leverage scores of $\U$;
		$\mu $ is the row coherence of $\U$.
		For Gaussian projection and CountSketch,
    the small-$o$ notation is a consequence of $s = o(n)$.} 
	\label{tab:sketching}
	\begin{center}
		\begin{footnotesize}
			\begin{tabular}{c c c c}
				\hline
				{\bf Sketching}
				&~~{\bf Subspace Embedding}~~
				&~~{\bf Matrix Multiplication}~~
				&~{\bf Spectral Norm}~ \\
				\hline
				Leverage
				&~~~$s = \OM \big( \frac{\tau \rho }{\eta^2 } \log \frac{\rho}{\delta_1} \big)$~~~
				& ~~~$s = \OM \big( \frac{\tau \rho}{\epsilon \delta_2} \big)$~~~
				& ~~~$\theta=\infty$~~~ \\
				Uniform
				&~~~$s = \OM \big( \frac{\mu  \rho }{\eta^2 } \log \frac{\rho}{\delta_1} \big)$~~~
				& ~~~$s = \OM \big(  \frac{\mu \rho}{\epsilon \delta_2} \big)$~~~
				& ~~~$\theta = 1$~~~ \\
				Shrinked Leverage
				&~~~$s = \OM \big(  \frac{\tau \rho }{\eta^2 } \log \frac{\rho }{\delta_1} \big)$~~~
				& ~~~$s = \OM \big( \frac{\tau \rho }{\epsilon \delta_2} \big)$~~~
				& ~~~$\theta = 2$~~~ \\
				SRHT
				&~~~$s = \OM \big( \frac{ \rho + \log n }{\eta^2 }  \log \frac{\rho}{\delta_1 } \big)$~~~
				& ~~~$s = \OM \big(  \frac{\rho + \log n }{\epsilon \delta_2 } \big)$~~~
				& ~~~$\theta = 1$~~~ \\
				~Gaussian Projection~
				&~~~$s = \OM \big( \frac{\rho +  \log (1/\delta_1)}{\eta^2} \big)$~~~
				& ~~~$s = \OM \big( \frac{\rho}{\epsilon \delta_2} \big)$~~~
				& ~~~$\theta = 1+o(1)$ w.h.p.~~~ \\
				CountSketch
				&~~~$s = \OM \big( \frac{\rho^2}{\delta_1 \eta^2} \big)$~~~
				&  ~~~$s = \OM \big( \frac{\rho}{\epsilon \delta_2} \big)$~~~
				& ~~~$\theta = 1+o(1)$ w.h.p.~~~ \\
				\hline
			\end{tabular}
		\end{footnotesize}
	\end{center}
\end{table}

Theorem~\ref{thm:properties} shows that the six sketching methods satisfy the three properties when $s$ is sufficiently large.
In particular, Theorem~\ref{thm:properties} shows that for all the sketching methods
except leverage score sampling,\footnote{If one leverage score approaches zero, 
	then the corresponding sampling probability $p_i$ goes to zero.
  By the definition of $\S$, the scale factor $\frac{1}{\sqrt{s p_i}}$ goes to infinity,
	which makes $\|\S\|_2^2$ unbounded.
	The shinked leverage score sampling avoids this problem and is thus a better choice than the leverage score sampling.}
$\|\S\|_2^2$ has nontrivial upper bound.
This is why Theorems~\ref{thm:biasvariance:classical} and~\ref{thm:biasvariance:classical_avg}
do not apply to leverage score sampling.
This fact can also be viewed as a motivation to use shrinked leverage score sampling.
We prove Theorem~\ref{thm:properties} in Appendix~\ref{sec:sketch:proof}.

\begin{theorem} \label{thm:properties}
	Fix failure probability $\delta$ and error parameters $\eta$ and $\epsilon$;
	set the sketch size $s$ as Table~\ref{tab:sketching}.
	Assumption~\ref{assumption:1:1} is satisfied with probability at least $1-\delta_1$.
	Assumption~\ref{assumption:1:2} is satisfied with probability at least $1-\delta_2$.
  Assumption~\ref{assumption:1:3} is satisfied either surely or with high probability (w.h.p.);
  the parameter $\theta$ is indicated in Table~\ref{tab:sketching}.
\end{theorem}

Theorem~\ref{thm:biasvariance:wtilde_lower2} establishes lower bounds on $\|\S\|_2^2$, and will
be applied to prove the lower bound on the variance of the classical sketch.
From Table~\ref{tab:var_wtilde} we see that the lower bound for (shrinked) leverage score sampling
is not interesting, because $\mu$ can be very large.
This is why Theorem~\ref{thm:biasvariance:classical} does not provide a lower bound for  
shrinked leverage score sampling. 
We prove Theorem~\ref{thm:biasvariance:wtilde_lower2} in Appendix~\ref{sec:sketch:proof}.

\begin{table}[!h]\setlength{\tabcolsep}{0.3pt}
  \caption{ Lower bounds on $\vartheta$ for the sketching modalities ($\vartheta$ is defined in Theorem~\ref{thm:biasvariance:wtilde_lower2}).
    The shrinked leverage score sampling is performed using the row leverage scores of a matrix $\X \in \RB^{n \times d}$, and $\mu$
    is the row coherence of $\X$.}
	\label{tab:var_wtilde}
	\begin{center}
		\begin{small}
			\begin{tabular}{c c}
				\hline
				Uniform
				& ~~~$\vartheta = 1$~~~ \\
				Leverage
				& ~~~$\vartheta \geq \frac{1 }{\mu} $~~~ \\
				Shrinked Leverage
				& ~~~$\vartheta \geq \frac{2 }{1 + \mu } $~~~ \\
				SRHT
				& ~~~$\vartheta = 1$~~~ \\
				~~~Gaussian Projection~~~
				& ~~~$\vartheta \geq 1-o(1)$ w.h.p.~~~ \\
				CountSketch
				& ~~~$\vartheta \geq 1-o(1)$ w.h.p.~~~ \\
				\hline
			\end{tabular}
		\end{small}
	\end{center}
\end{table} 

\begin{theorem} [Semidefinite Lower Bound on the Sketching Matrix] \label{thm:biasvariance:wtilde_lower2}
  When $s < n$, $\S^T \S \succeq \frac{\vartheta n }{s} \I_s$ holds either surely or
  with high probability (w.h.p.), where Table~\ref{tab:var_wtilde} provides the applicable
  $\vartheta$ for each sketching method.
\end{theorem}

\begin{Remark} \label{remark:sketch:spectral}
Let $p_1 , \cdots , p_n$ be an arbitrary set of sampling probabilities.
By the definition of the associated sampling matrix $\S \in \RB^{n\times s}$, the non-zero entries of $\S$ can be any of $\frac{1}{\sqrt{s p_i}}$, for $i \in [n]$.

For leverage score sampling, since the smallest sampling probability can be zero or close, and the largest sampling probability can be close to one, $\|\S\|_2^2$ has no nontrivial
upper or lower bound.\footnote{In our application, nontrivial bound means 
	$\|\S\|_2^2$ is of order $\frac{n}{s}$.}
It is because $\min_{i} p_i $ can be close to zero
and $\max_{i} p_i $ can be large (close to one).

For shrinked leverage score sampling, 
because $\min_{i} p_i $ is at least $\frac{1}{2n}$, $\|\S\|_2^2$ has a nontrivial upper bound;
but as in the case of leverage score sampling, since $\max_{i} p_i $ can be large, 
there is no nontrivial lower bound on $\|\S\|_2^2$.
\end{Remark}


\subsection{Matrix Sketching with Averaging} \label{sec:proofsketch:sketch_avg}

Assumptions~\ref{assumption:1:1} and \ref{assumption:1:2} imply that sketching can be used to 
approximate certain matrix products, but what happens if we independently draw $g$ sketches, 
use them to approximate the same matrix product, and then average the $g$ results?
Intuitively, averaging should lower the variance of the approximation without affecting its bias,
and thus provide a better
approximation of the true product.

To justify this intuition formally, let
$\S_1 , \cdots , \S_g \in \RB^{n\times s}$ be sketching matrices
and $\A$ and $\B$ be fixed conformal matrices.
Then evidently
\begin{small}
\[
\frac{1}{g} \sum_{i=1}^g \A^T \S_i \S_i^T \B
\; = \; \A^T \S \S^T \B,
\]
\end{small}%
where $\S = \frac{1}{\sqrt{g}} [\S_1 , \cdots , \S_g ] \in \RB^{n\times gs}$ 
can be thought of as a sketching matrix formed by concatenating the $g$ smaller sketching matrices.
If $\S_1 , \cdots , \S_g$ are all instance of column selection, SRHT, or Gaussian projection sketching matrices, 
then $\S$ is a larger instance of the same type of sketching matrix.\footnote{CountSketch sketching matrices
  does not have this property. If $\S_i \in \RB^{n\times s}$ is a CountSketch matrix,
	then it has only one non-zero entry in each row. In contrast,
	$\S \in \RB^{n\times gs}$ has $g$ non-zero entries in each row.}

To analyze the effect of model averaging on the solution to the sketched MRR problem,
we make the following assumptions on the concatenated sketch matrix.
Assumption~\ref{assumption:2:1} is the subspace embedding property,
Assumption~\ref{assumption:2:2} is the matrix multiplication property, and
Assumption~\ref{assumption:2:3} is the bounded spectral norm property.

\begin{assumption}\label{assumption:avg}
	Let $\eta , \epsilon \in (0, 1)$ be fixed parameters.
	Let $\B$ be any fixed matrix of proper size, $\rho = \rk (\X)$,
  and $\U \in \RB^{n\times \rho}$ be an orthonormal basis for the column span of $\X$.
	Let $\S_1 , \cdots , \S_g \in \RB^{n\times s}$ be sketching matrices
	and $\S = \frac{1}{\sqrt{g}} [\S_1 , \cdots , \S_g ] \in \RB^{n\times gs}$;
	here $s$ depends on $\eta$ and/or $\epsilon$.
  Throughout this paper we assume that $\S$ and the $\S_i$ satisfy the following 
  properties with a probability that depends on $g$ and $s$:
	\begin{itemize}
		\item[\mylabel{assumption:2:1}{2.1}]
		$\big\|\U^T \S_i \S_i^T \U - \I_\rho \big\|_2 \leq \eta$ for all $i\in [g]$
		\quad \textrm{and} \quad
		$\big\|\U^T \S \S^T \U - \I_\rho \big\|_2 \leq \frac{\eta}{ \sqrt{g} }$;
		\item[\mylabel{assumption:2:2}{2.2}]
		$\big(\tfrac{1}{g} \sum_{i=1}^g \big\|\U^T \S_i \S_i^T \B - \U^T \B \big\|_F \big)^2
		\leq \epsilon \|\B\|_F^2$
		\quad \textrm{and} \quad
		$\big\|\U^T \S \S^T \B - \U^T \B \big\|_F^2 \leq \frac{\epsilon}{g} \|\B\|_F^2$;
		\item[\mylabel{assumption:2:3}{2.3}]
		For some constant $\theta$,
		$\|\S_i\|_2^2  \leq \frac{\theta n}{s}$ for all $i\in [g]$,  
		\, \textrm{and} \,
		$\|\S \|_2^2  \leq \frac{\theta n}{g s}$ for $gs < n$.
	\end{itemize}	
  Except in the case of leverage score sampling, when $gs$ is comparable to or larger than $n$, $\|\S \|_2^2 = \Theta (1)$.
\end{assumption}

Theorem~\ref{thm:properties_avg} establishes that random column selection, SRHT, and Gaussian projection matrices
satisfy Assumptions~\ref{assumption:2:1}, \ref{assumption:2:2}, and~\ref{assumption:2:3}.
We prove Theorem~\ref{thm:properties_avg} in Appendix~\ref{sec:sketch:proof}.

\begin{theorem} \label{thm:properties_avg}
  Let $\S_1 , \cdots , \S_g \in \RB^{n\times s}$ be independent and identically distributed random sketching matrices that are either
	column selection, SRHT, or Gaussian projection matrices.
  Fix a failure probability $\delta$ and error parameters $\eta$ and $\epsilon$, then
	set the sketch size $s$ as Table~\ref{tab:sketching}.
	
  Assumption~\ref{assumption:2:1} holds with probability at least $1-(g+1)\delta_1$.
	Assumption~\ref{assumption:2:2} holds with probability at least $1-2\delta_2$.
  Assumption~\ref{assumption:2:3} is satisfied either surely or with high probability, 
  with the parameter $\theta$ specified in Table~\ref{tab:sketching}.
\end{theorem}

In Theorem~\ref{thm:properties}, Assumption~\ref{assumption:1:1} fails with probability at most $\delta_1$.
In contrast, in Theorem~\ref{thm:properties_avg}, 
the counterpart assumption fails with probability at most $(g+1) \delta_1$.
However, this makes little difference in practice, because the dependence of $s$ on $\delta_1$ is logarithmic, so $\delta_1$ can be set very small (recall Table~\ref{tab:sketching}) without increasing $s$ significantly.

\begin{Remark} \label{remark:cs_avg}
	We do not know whether CountSketch enjoys the properties in Assumption~\ref{assumption:avg}.
  There are two difficulties in establishing this using the same route as is employed in our proof of Theorem~\ref{thm:properties} for other sketching methods.
  First, the concatenation of multiple CountSketch matrices is not a CountSketch matrix.
  Second, the probability that a CountSketch matrix does not have the subspace embedding property is constant, rather than exponentially small.
\end{Remark}


\subsection{Sketched MRR: Optimization Perspective} \label{sec:proofsketch:opt}

The randomness in the performance of the classical and Hessian sketch is entirely due to the choice of random sketching matrix. 
We now assume that the randomly sampled sketching matrices are ``nice" in that they satisfy the assumptions just introduced, and
state deterministic results on the optimization performance of the classical and Hessian sketches.

Theorem~\ref{thm:optimization:wtilde} holds under the subspace embedding property 
and the matrix multiplication property
(Assumptions~\ref{assumption:1:1} and \ref{assumption:1:2}), and quantifies the suboptimality of the classical sketch.
We prove this result in Appendix~\ref{sec:optimization:proof}.

\begin{theorem} [Classical Sketch] \label{thm:optimization:wtilde}
	Let Assumptions~\ref{assumption:1:1} and \ref{assumption:1:2} hold
	for the sketching matrix $\S \in \RB^{n\times s}$.
  Let $\eta$ and $\epsilon$ be defined in Assumption~\ref{assumption:sketching}, and let
	$\alpha = \frac{2 \max \{ {\epsilon} , \eta^2 \} }{ 1- \eta } $
  and $\beta = \frac{\| \X \|_2^2}{\| \X \|_2^2 + n \gamma} $, then
	\begin{small}
	\begin{eqnarray*}
		f ({\W}^{\textrm{c}}) - f (\W^\star )
		& \leq & \alpha \beta  f (\W^\star ) .
	\end{eqnarray*}
\end{small}%
\end{theorem}

Theorem~\ref{thm:optimization:what} holds under the subspace embedding property (Assumption~\ref{assumption:1:1}), and quantifies
the suboptimality of the Hessian sketch.
We prove this result in Appendix~\ref{sec:optimization:proof}.

\begin{theorem} [Hessian Sketch] \label{thm:optimization:what}
	Let Assumption~\ref{assumption:1:1} hold
	for the sketching matrix $\S \in \RB^{n\times s}$.
	Let $\eta$ be defined in Assumption~\ref{assumption:sketching}
  and $\beta = \frac{\| \X \|_2^2}{\| \X \|_2^2 + n \gamma} $, then
	\begin{small}
	\begin{eqnarray*}
		f ({\W}^{\textrm{h}}) - f (\W^\star )
		& \leq & \frac{\eta^2 \beta^2 }{(1-\eta)^2 }  \bigg( \frac{\|\Y\|_F^2}{n }-  f (\W^\star ) \bigg) .
	\end{eqnarray*}
	\end{small}%
\end{theorem}


\subsection{Sketched MRR: Statistical Perspective}\label{sec:proofsketch:stat}

Similarly, we assume that the randomly sampled sketching matrices are nice, and state deterministic
results on the bias and variance of the classical and Hessian sketches.

Theorem~\ref{thm:biasvariance:wtilde} holds under the subspace embedding property 
(Assumption~\ref{assumption:1:1})
and the bounded spectral norm property (Assumption~\ref{assumption:1:3}), and bounds the
bias and variance of the classical sketch.
Specifically, it shows that the bias of the classical sketch is close to that of the optimal solution,
but that the variance may be much larger.
We prove this result in Appendix \ref{sec:statistical:proof}.

\begin{theorem} [Classical Sketch] \label{thm:biasvariance:wtilde}
	Let $\eta$ and $\theta$ be defined in Assumption~\ref{assumption:sketching}.
	Under Assumption~\ref{assumption:1:1}, it holds that
	\begin{eqnarray*}
		\tfrac{1}{ 1 + \eta }
		\; \leq \; \tfrac{\bias ( \W^{\textrm{c}} )}{\bias (\W^\star ) }
		\; \leq \; \tfrac{1}{ 1 - \eta } .
	\end{eqnarray*}
	Further assume $s \leq n$;
	under Assumptions~\ref{assumption:1:1} and \ref{assumption:1:3}, it holds that
	\begin{eqnarray*}
		\tfrac{\var ( \W^{\textrm{c}} ) }{ \var (\W^\star ) }
		\; \leq \; \tfrac{ (1+\eta )}{ (1-\eta)^2} \tfrac{\theta n}{s} .
	\end{eqnarray*}
\end{theorem}

Theorem \ref{thm:biasvariance:wtilde_lower1} establishes a lower bound on the variance of the classical sketch.
We prove this result in Appendix \ref{sec:statistical:proof}.

\begin{theorem} [Lower Bound on the Variance] \label{thm:biasvariance:wtilde_lower1}
	Under Assumption~\ref{assumption:1:1} and the additional assumption that
	$\S^T \S \succeq \frac{\vartheta n }{s} \I_s$, it holds that
	\begin{eqnarray*}
		\tfrac{ \var ( \W^{\textrm{c}} ) }{\var (\W^\star ) }
		& \geq & \tfrac{  1-\eta  }{ (1+\eta )^2} \tfrac{\vartheta n }{s}  .
	\end{eqnarray*}
\end{theorem}

Theorem~\ref{thm:biasvariance:what} holds under the subspace embedding property (Assumption~\ref{assumption:1:1}), and
quantifies the bias and variance of the Hessian sketch.
We prove this result in Appendix \ref{sec:statistical:proof}.

\begin{theorem} [Hessian Sketch] \label{thm:biasvariance:what}
  Let $\eta$ be defined in Assumption~\ref{assumption:sketching}, take $\rho = \rk(\X)$, and let $\sigma_1 \geq \cdots \geq \sigma_\rho$ be the singular values of $\X$. 
	Under Assumption~\ref{assumption:1:1}, it holds that
	\begin{align*}
	&\tfrac{\bias ( {\W}^{\textrm{h}} )}{\bias (\W^\star ) }
	\; \leq \; \tfrac{1}{1 - \eta} \,
	\Big( 1 + \tfrac{\eta \sigma_{1}^2 }{n \gamma} \Big)  ,\\
	&\tfrac{1}{1+\eta }
	\; \leq \;
	\tfrac{\var ( {\W}^{\textrm{h}} ) }{ \var (\W^\star ) }
	\; \leq \;
	\tfrac{1}{1- \eta }  .
	\end{align*}
	Further assume that $\sigma_\rho^2 \geq \frac{n \gamma}{\eta }$.
	Then
	\begin{align*}
	\tfrac{\bias ( {\W}^{\textrm{h}} )}{\bias (\W^\star ) }
	\; \geq \; \tfrac{1}{1 + \eta} \,
	\Big(  \tfrac{\eta \sigma_{\rho}^2 }{n \gamma} - 1\Big)  .
	\end{align*}
\end{theorem}


\subsection{Model Averaging: Optimization Perspective}\label{sec:proofsketch:opt_avg}

Theorem~\ref{thm:optimization:wtilde_avg} holds under 
the subspace embedding property (Assumption~\ref{assumption:2:1})
and the matrix multiplication property
(Assumption~\ref{assumption:2:2}).
We prove this result in Appendix~\ref{sec:opt_avg:proof}.

\begin{theorem} [Classical Sketch with Model Averaging] \label{thm:optimization:wtilde_avg}
  Let $\eta$ and $\epsilon$ be defined in Assumption~\ref{assumption:avg}, and 
  let $\alpha = 2 \big( \tfrac{1}{\sqrt{g}} + 2 \beta \eta \big)^2 \max \big\{ \epsilon , \eta^2 \big\}  $
	and $\beta = \frac{\| \X \|_2^2}{\| \X \|_2^2 + n \gamma } \leq 1$.
	Under Assumption~\ref{assumption:2:1} and~\ref{assumption:2:2},
	we have that
	\begin{eqnarray*}
		f ({\W}^{\textrm{c}}) - f (\W^\star )
		& \leq & \alpha \beta f (\W^\star ) .
	\end{eqnarray*}
\end{theorem}

Theorem~\ref{thm:optimization:what_avg} holds 
under the subspace embedding property (Assumption~\ref{assumption:2:1}), and is proven in
Appendix~\ref{sec:opt_avg:proof}.

\begin{theorem} [Hessian Sketch with Model Averaging] \label{thm:optimization:what_avg}
	Let $\eta$ be defined in Assumption~\ref{assumption:avg}, 
  and let $\alpha = \big( \frac{\eta}{ \sqrt{g} } + \frac{ \eta^2 }{1-\eta } \big) $
	and $\beta = \frac{\| \X \|_2^2}{\| \X \|_2^2 + n \gamma } \leq 1$.
	Under Assumption~\ref{assumption:2:1},
	we have that
	\begin{eqnarray*}
		f ({\W}^{\textrm{h}}) - f (\W^\star )
		& \leq & \alpha^2 \beta^2 \Big(  \tfrac{1}{n} \|\Y\|_F^2 - f (\W^\star )\Big) .
	\end{eqnarray*}
\end{theorem}


\subsection{Model Averaging: Statistical Perspective}\label{sec:proofsketch:stat_avg}

Theorem~\ref{thm:biasvariance:wtilde_avg} requires the subspace embedding property
(Assumption~\ref{assumption:2:1}).
In addition, to bound the variance, the spectral norms of 
$\S_1, \cdots , \S_g$ and $\S = \frac{1}{\sqrt{g}} [\S_1 , \cdots , \S_g]$ must be bounded
(Assumption~\ref{assumption:2:3}).
This result shows that model averaging decreases the variance of the classical sketch without increasing its bias.
We prove this result in Appendix \ref{sec:stat_avg:proof}.

\begin{theorem} [Classical Sketch with Model Averaging] \label{thm:biasvariance:wtilde_avg}
	Under Assumption~\ref{assumption:2:1}, it holds that
	\begin{eqnarray*}
		\frac{\bias ( \W^{\textrm{c}} )}{\bias (\W^\star )}
		& \leq & \frac{1}{ 1 - \eta }  .
	\end{eqnarray*}
	Under Assumptions~\ref{assumption:2:1} and \ref{assumption:2:3}, it holds that
	\begin{eqnarray*}
		\frac{\var ( \W^{\textrm{c}} ) }{ \var (\W^\star ) }
		& \leq & \frac{\theta n}{s} \bigg( \tfrac{\sqrt{1+\eta /  \sqrt{g} }}{ \sqrt{h}}
		+ \tfrac{ \eta \sqrt{1+\eta}}{1- \eta} \bigg)^2 .
	\end{eqnarray*}
	Here $\eta$ and $\theta$ are defined in Assumption~\ref{assumption:avg}
	and $h = \min \{ g, \, \frac{n}{s}\big( 1 - o(1) \big)  \}$,
\end{theorem}

Theorem~\ref{thm:biasvariance:what_avg} requires the subspace embedding property
(Assumption~\ref{assumption:2:1}), and
shows that model averaging decreases the bias of the Hessian sketch without increasing its variance.
We prove this result in Appendix \ref{sec:stat_avg:proof}.

\begin{theorem} [Hessian Sketch with Model Averaging] \label{thm:biasvariance:what_avg}
	Under Assumption~\ref{assumption:2:1}, it holds that:
	\begin{small}
	\begin{align*}
	&\frac{\bias ( {\W}^{\textrm{h}} )}{\bias (\W^\star ) }
	\; \leq \;
	\frac{1 }{1-\eta} + \Big( \frac{\eta}{ \sqrt{g} } + \frac{ \eta^2 }{1-\eta} \Big)
	\frac{\|\X \|_2^2 }{n \gamma }  ,\\
	& \frac{\var ( {\W}^{\textrm{h}} ) }{ \var (\W^\star ) }
	\; \leq \;
	\frac{1}{1- \eta }  .
	\end{align*}
	\end{small}%
	Here $\eta$ is defined in Assumption~\ref{assumption:avg}.
\end{theorem}

%% file: proof.tex
\section{Properties of Matrix Sketching: Proofs} \label{sec:sketch:proof}

In Section~\ref{sec:sketch:proof:1} we prove Theorem~\ref{thm:properties}.
In Section~\ref{sec:statistical:proof:lower2}, we prove Theorem~\ref{thm:biasvariance:wtilde_lower2}.
In Section~\ref{sec:sketch:proof:2} we prove Theorem~\ref{thm:properties_avg}.

\subsection{Proof of Theorem~\ref{thm:properties}}  \label{sec:sketch:proof:1}

We prove that the six sketching methods considered in this paper satisfy the three key properties.
In Section~\ref{sec:sketch:proof:1A1A2} we show the
six sketching methods satisfy Assumptions~\ref{assumption:1:1} and \ref{assumption:1:2}.
In section~\ref{sec:sketch:proof:1A3} we show the
six sketching methods satisfy Assumption~\ref{assumption:1:3}.

\subsubsection{Proof of Assumptions \ref{assumption:1:1} and \ref{assumption:1:2}}  \label{sec:sketch:proof:1A1A2}

For uniform sampling, leverage score sampling, Gaussian projection, SRHT, and CountSketch,
the subspace embedding property and matrix multiplication property have been established by the previous works
\citep{drineas2008cur,drineas2011faster,meng2013low,nelson2013osnap,tropp2011improved,woodruff2014sketching}.
See also \citep{wang2015towards} for a summary.

In the following we prove only that {\bf shrinked leverage score sampling} satisfies assumptions~\ref{assumption:1:1} and~\ref{assumption:1:2}.
We cite the following lemma from \citep{wang2016spsd};
this lemma was first established in the works \citep{drineas2008cur,gittens2011spectral,woodruff2014sketching}.

\begin{lemma}[\citet{wang2016spsd}] 
	\label{lem:sampling_property}
	Let $\U\in \RB^{n\times \rho}$ be a fixed matrix with orthonormal columns.
  Let the column selection matrix $\S \in \RB^{n\times s}$ sample $s$
	columns according to probabilities $p_1 , p_2 , \cdots , p_n$.
	Assume $\alpha \geq \rho $ and
	\begin{equation*}
	\max_{i\in [n]} \frac{\| \u_{i:} \|_2^2}{p_{i}}
	\; \leq \; \alpha .
	\end{equation*}
	When
	$s \, \geq \, \alpha \frac{6 + 2\eta}{3 \eta^2} \log (\rho /\delta_1)$,
	it holds that
	\[
	\PB \Big\{ \big\| \I_\rho - \U^T \S \S^T \U \big\|_2 \; \geq \; \eta \Big\}
	\; \leq \; \delta_1 .
	\]
	When $s \, \geq \, \frac{\alpha }{\epsilon \delta_2}$,
	it holds that
	\begin{align*}
	& \EB \big\| \U \B - \U^T \S \S^T \B \big\|_F^2 \; \leq \; \delta_2 \epsilon \|\B\|_F^2 ;
	\end{align*}
	as a consequence of Markov's inequality, it holds that
	\begin{align*}
	\PB \Big\{ \big\| \U \B - \U^T \S \S^T \B \big\|_F^2 \; \geq \; \epsilon \|\B\|_F^2 \Big\}
	\; \leq \; \delta_2 .
	\end{align*}
  Here the expectation and probability are with respect to the randomness in $\S$.
\end{lemma}

Now we apply the above lemma to analyze {\bf shrinked leverage score sampling}.
Given the approximate shrinked leverage scores defined in \eqref{eq:def:appro_lev},
the sampling probabilities satisfy
\begin{eqnarray*}
p_i 
\; = \; \tfrac{1}{2} \big( \tfrac{1}{n} + \tfrac{\tilde{l}_i}{\sum_{q=1}^n \tilde{l}_q} \big)
\; \geq \; \tfrac{\|\u_{i:} \|_2^2}{2\tau \rho} .
\end{eqnarray*}
Here $\tilde{l}_i$ and $\tau$ are defined in \eqref{eq:def:appro_lev}.
Thus for all $i \in [n]$, $\frac{\|\u_{i:}\|_2^2}{p_i } \leq 2 \tau \rho$.
We can then apply Lemma~\ref{lem:sampling_property} to show
that Assumption~\ref{assumption:1:1} holds with probability at least $1-\delta_1$ when
$s \geq 2 \tau \rho \frac{6 + 2\eta}{3 \eta^2} \log \frac{\rho}{\delta_1}$
and that Assumption~\ref{assumption:1:2} holds with probability at least $1-\delta_2$ when
$s \geq \frac{2\tau\rho }{\epsilon \delta_2}$.

\subsubsection{Proof of Assumption \ref{assumption:1:3}} \label{sec:sketch:proof:1A3}

For {uniform sampling (without replacement) and SRHT}, when $s < n$,
it is easy to show that $\S^T \S = \frac{n}{s} \I_s$, 
and thus $\|\S\|_2^2 = \frac{n}{s}$.
Let $\{  p_i^{\textrm{s}}  \}$ and $\{  p_i^{\textrm{u}}  \}$
be the sampling probabilites of shrinked leverage score sampling and uniform sampling, respectively.
Obviously $ p_i^{\textrm{s}} \geq \frac{1}{2}  p_i^{\textrm{u}}$.
Thus for shrinked leverage score sampling, $\|\S\|_2^2 \leq \frac{2n}{s}$.

The greatest singular value of a standard Gaussian matrix $\G \in \RB^{n\times s}$
is at most $\sqrt{n} + \sqrt{s} + t$ with probability at least $1-2 e^{-t^2/2}$~\citep{vershynin2010introduction}.
Thus a Gaussian projection matrix $\S$ satisfies 
\[
\|\S\|_2^2
\; = \; \frac{1}{s} \|\G\|_2^2
\; \leq \; \frac{(\sqrt{n} + \sqrt{s} + t)^2 }{s}
\]
with probability at least $1-2 e^{-t^2/2}$.

If $\S$ is the CountSketch matrix, then each row of $\S$ has exactly one nonzero entry,
either $1$ or $-1$. Because the columns of $\S$ are orthogonal to each other, it holds that
\[
\|\S\|_2^2
\; = \; \max_{i \in [s]} \|\s_{:i}\|_2^2
\; = \; \max_{i \in [s]}  \, \nnz (\s_{:i}) .
\]
The problem of bounding $\nnz (\s_{:i})$ is equivalent to
assigning $n$ balls into $s$ bins uniformly at random
and bounding the number of balls in the bins.
\citet{patrascu2012power} showed that for $s \ll n$, the maximal number of balls in any bin is
at most $n/s + \OM \big( \sqrt{n/s} \log^c n \big)$
with probability at least $1-\frac{1}{n}$, where $c = \OM (1)$.
Thus
\[
\|\S\|_2^2
\; = \; \max_{i \in [s]}  \, \nnz (\s_{:i})
\; \leq \; \frac{n}{s} + \OM \bigg( \frac{\sqrt{n} \log^c n }{\sqrt{s}} \bigg)
\; = \; \frac{n}{s} \big( 1 + o (1) \big)
\]
holds with probability at least $1-\frac{1}{n}$.


\subsection{Proof of Theorem~\ref{thm:biasvariance:wtilde_lower2}}
\label{sec:statistical:proof:lower2}

For uniform sampling (without replacement) and SRHT, it holds that $\S^T \S = \frac{n}{s} \I_s$.

For non-uniform sampling with probabilities $p_1 , \cdots , p_n$, (with $\sum_{i} p_i = 1$),
let $p_{\max} = \max_i p_i$.
The smallest entry in $\S$ is $\tfrac{1}{\sqrt{s p_{\max}}}$,
and thus $\S^T \S \succeq \frac{1}{s p_{\max} } \I_s$.
For leverage score sampling, $p_{\max } = \frac{\mu}{n}$.
For shrinked leverage score sampling, $p_{\max } = \frac{1+\mu}{2n}$.
The lower bound on $\|\S \|_2^2$ is thus established.

The smallest singular value of any
$n\times s$ standard Gaussian matrix $\G$ is at least $\sqrt{n} - \sqrt{s} - t$
with probability at least $1- 2 e^{-t^2 / 2}$~\citep{vershynin2010introduction}.
Thus if $\S = \frac{1}{\sqrt{s}} \G$ is the Gaussian projection matrix,
the smallest eigenvalue of $\S^T \S$ is $(1-o(1)) \frac{n}{s}$ with probability very close to one.

If $\S$ is the CountSketch matrix, then each row of $\S$ has exactly one nonzero entry,
either $1$ or $-1$. Because the columns of $\S$ are orthogonal to each other, it holds that
\[
\sigma_{\min}^2 (\S)
\; = \; \min_{i \in [s]} \|\s_{:i}\|_2^2
\; = \; \min_{i \in [s]}  \, \nnz (\s_{:i}) .
\]
The problem of bounding $\nnz (\s_{:i})$ is equivalent to
assigning $n$ balls into $s$ bins uniformly at random
and bounding the number of balls in the bins.
Standard concentration arguments imply that each bin has at least
$\frac{n}{s} (1-o(1))$ balls w.h.p., and hence
$\sigma_{\min}^2 (\S) \geq \frac{n}{s} (1-o(1))$ w.h.p.


\subsection{Proof of Theorem~\ref{thm:properties_avg}}
 \label{sec:sketch:proof:2}

\paragraph{Assumption~\ref{assumption:2:1}.}
By Theorem~\ref{thm:properties} and the union bound,
we have that $\big\|\U^T \S_i \S_i^T - \I_\rho \big\|_2 \leq \eta$ 
hold simultaneously for all $i\in [g]$ with probability at least $1-g\delta_1$.
Because $\S \in \RB^{n\times gs}$ is the same type of sketching matrix,
it follows from Theorem~\ref{thm:properties} that
$\big\|\U^T \S \S^T \U - \I_\rho \big\|_2 \leq \frac{\eta}{ \sqrt{g} }$
holds with probability at least $1- \delta_1$.

\paragraph{Assumption~\ref{assumption:2:2}.}
By the same proof of Theorem~\ref{thm:properties}, 
we can easily show that	
\begin{eqnarray*}
	\EB \big\| \U^T \B - \U^T \S_i \S_i^T \B \big\|_F^2
	\; \leq \; \delta_2 \epsilon \, \|\B\|_F^2 ,
\end{eqnarray*}
where $\B$ is any fixed matrix and the expectation is taken w.r.t.\ $\S$.
It follows from Jensen's inequality that
\begin{align*}
\Big( \EB  \big\|\U^T \S_i \S_i^T \B - \U^T \B \big\|_F \Big)^2
\; \leq \; \EB  \big\|\U^T \S_i \S_i^T \B - \U^T \B \big\|_F^2
\; \leq \; {\delta_2 \epsilon } \big\|\B\big\|_F^2 .
\end{align*}
It follows that
\begin{align*}
\frac{1}{g} \sum_{i=1}^g \EB  \big\|\U^T \S_i \S_i^T \B - \U^T \B \big\|_F
\; \leq \; \sqrt{\delta_2 \epsilon } \big\|\B\big\|_F ,
\end{align*}
and thus
\begin{align*}
\Big( \frac{1}{g} \sum_{i=1}^g \EB  \big\|\U^T \S_i \S_i^T \B - \U^T \B \big\|_F \Big)^2
\; \leq \; {\delta_2 \epsilon } \big\|\B\big\|_F^2 .
\end{align*}
It follows from Markov's bound that
\begin{align*}
\PB \bigg\{ \Big( \frac{1}{g} \sum_{i=1}^g  \big\|\U^T \S_i \S_i^T \B - \U^T \B \big\|_F \Big)^2
\; \leq \; { \epsilon } \big\|\B\big\|_F^2 \bigg\}
\; \geq \; 1 - \delta_2 .
\end{align*}
Because $\S \in \RB^{n\times gs}$ is the same type of sketching matrix,
it follows from Theorem~\ref{thm:properties} that
$\big\|\U^T \S \S^T \B - \U^T \B \big\|_F^2 \leq \frac{\epsilon}{g} \|\B\|_F^2$
holds with probability at least $1-\delta_2$.

\paragraph{Assumption~\ref{assumption:2:3}.}
Theorem~\ref{thm:properties} shows that $\|\S_i\|_2^2$ can be bounded either surely or w.h.p.\ (assuming $n$ is large enough).
Because $g \ll n$, $\|\S_i\|_2^2$ can be bounded simultaneously for all $i\in [g]$
either surely or w.h.p.

Suppose $sg < n$.
Because $\S \in \RB^{n\times gs}$ is the same type of sketching matrix,
it follows from Theorem~\ref{thm:properties} that $\|\S\|_2^2 \leq \tfrac{\theta n}{gs}$ holds either surely or w.h.p.

Suppose $sg \geq n$.
It is not hard to show that uniform sampling, shrinked leverage score sampling, and SRHT satisfy $\| \S \|_2
= \Theta (1)$ w.h.p.
Previously we have shown that a random Gaussian projection matrix $\S \in \RB^{n\times sg}$ satisfies
\begin{small}
\[
\|\S\|_2^2
\; \leq \; \big( 1 + o(1) \big) \,  \frac{(\sqrt{n} + \sqrt{gs} )^2 }{gs} 
\]
\end{small}%
w.h.p. 
Hence for $sg \geq n$, $\|\S \|_2^2 \leq 4 + o(1)$ w.h.p.


\section{Sketched MRR from the Optimization Perspective: Proofs}
\label{sec:optimization:proof}

In Section~\ref{sec:optimization:proof:lemma} we establish a key lemma.
In Section~\ref{sec:optimization:proof:classical} we prove Theorem~\ref{thm:optimization:wtilde}.
In Section~\ref{sec:optimization:proof:hessian} we prove Theorem~\ref{thm:optimization:what}.

\subsection{Key Lemma}
\label{sec:optimization:proof:lemma}

Recall that the objective function of the matrix ridge regression (MRR) problem is
\begin{eqnarray*}
	f (\W )
	\; \triangleq \; \frac{1}{n} \big\| \X \W - \Y \big\|_F^2 + \gamma \|\W\|_F^2 .
\end{eqnarray*}
The optimal solution is $\W^\star = \argmin_\W f(\W)$.
The following is the key lemma for understanding the difference between the objective value at $\W^\star$ and any arbitrary $\W$.

\begin{lemma} \label{lem:optimization}
	For any matrix $\W$ and any nonsingular matrix $\M$ of proper size,
	it holds that
	\begin{eqnarray*}
		f (\W )
		& = & \frac{1}{n} \tr \Big[ \Y^T \Y - (2 \W^\star - \W)^T (\X^T \X + n \gamma \I_n) \W \Big] , \\
		f (\W^\star )
		& = & \frac{1}{n}  \Big[ \big\|\Y^\perp \big\|_F^2
		+ n \gamma \big\| \big( \Si^2 + n \gamma \I_\rho \big)^{-1/2} \U^T \Y \big\|_F^2 \Big] ,\\
		f (\W ) - f(\W^\star )
		& = & \frac{1}{n} \Big\| (\X^T \X + n \gamma \I_d)^{1/2} (\W - \W^\star ) \Big\|_F^2 ,\\
		\Big\| \M^{-1}({\W} - \W^\star ) \Big\|_F^2
		& \leq & \sigma_{\min}^{-2} \Big[ (\X^T   \X + n \gamma \I_d)^{1/2} \M \Big]
		\, \Big\| (\X^T \X + n \gamma \I_d)^{1/2} ({\W} - \W^\star ) \Big\|_F^2 .
	\end{eqnarray*}
	Here $\X = \U \Si \V^T$ is the SVD and $\Y^\perp = \Y - \X \X^\dag \Y$.
\end{lemma}

\begin{proof}
	Let $\U$ be the left singular vectors of $\X$.
  The objective value $f (\W)$ can be written as
	\begin{eqnarray*}
		f (\W ) &= &
		\frac{1}{n} \big\| \X \W - \Y \big\|_F^2 + \gamma \big\|\W \big\|_F^2 \\
		& = & \frac{1}{n} \tr \Big[ \Y^T \Y - (2 \W^\star - \W)^T (\X^T \X + n \gamma \I_n) \W \Big],
	\end{eqnarray*}
so	
	\begin{eqnarray*}
		f (\W^\star )
		& = & \frac{1}{n} \tr\Big[\Y^T \Big( \I_n - \X (\X^T \X + n \gamma \I_d )^{-1} \X^T \Big) \Y \Big] \nonumber\\
		& = & \frac{1}{n} \tr\Big[\Y^T \Big( \I_n - \U (\I_\rho + n \gamma \Si^{-2} )^{-1} \U^T \Big) \Y\Big] \nonumber \\
		& = & \frac{1}{n} \tr\Big[\Y^T \Y - \Y^T \U \U^T \Y + \Y^T \U \U^T \Y
		- \Y^T   \U (\I_\rho + n \gamma \Si^{-2} )^{-1} \U^T \Y\Big] \nonumber \\
		& = & \frac{1}{n} \bigg\{ \tr \Big[ \Y^T (\I_n - \U \U^T) \Y \Big]
		+ n \gamma \cdot \tr \Big[ \Y^T \U \big( \Si^2 + n \gamma \I_\rho \big)^{-1} \U^T \Y \Big] \bigg\} \nonumber \\
		& = & \frac{1}{n}  \Big[ \big\|\Y^\perp \big\|_F^2
		+ n \gamma \big\| \big( \Si^2 + n \gamma \I_\rho \big)^{-1/2} \U^T \Y \big\|_F^2 \Big].
	\end{eqnarray*}
  The difference in the objective values is therefore
	\begin{eqnarray*}
		f (\W ) - f(\W^\star ) &= &
		\frac{1}{n}  \tr \Big[ (\W - \W^\star )^T (\X^T \X + n \gamma \I_d) (\W - \W^\star) \Big] \nonumber \\
		& = & \frac{1}{n} \Big\| (\X^T \X + n \gamma \I_d)^{1/2} (\W - \W^\star ) \Big\|_F^2 .
	\end{eqnarray*}
	Because $\sigma_{\min} (\A) \|\B\|_F \leq \|\A \B \|_F$ holds for any nonsingular $\A$ and any $\B$,
	it holds for any nonsingular matrix $\M$ that
	\begin{small}
	\begin{eqnarray*}
	\sigma_{\min}^2 \Big[ (\X^T \X + n \gamma \I_d)^{1/2} \M \Big]
	\Big\| \M^{-1} (\W - \W^\star ) \Big\|_F^2
	& \leq &
	 \Big\| (\X^T \X + n \gamma \I_d)^{1/2} \M \M^{-1} (\W - \W^\star ) \Big\|_F^2 \\
	& = &
	\Big\| (\X^T \X + n \gamma \I_d)^{1/2} (\W - \W^\star ) \Big\|_F^2 .
	\end{eqnarray*}
	\end{small}%
	The last claim in the lemma follows from the above inequality.
\end{proof}


\subsection{Proof of Theorem~\ref{thm:optimization:wtilde}}
\label{sec:optimization:proof:classical}

\begin{proof}
	Let $\rho = \rk (\X)$,
	$\U \in \RB^{n\times \rho}$ be the left singular vectors of $\X$,
	and $\Y^\perp = \Y - \X \X^\dag \Y = \Y - \U \U^T \Y$.
	It follows from the definition of $\W^\star $ and $\W^{\textrm{c}} $ that
	\begin{eqnarray*}
		{\W}^{\textrm{c}} - \W^\star
		& = & (\X^T \S \S^T \X + n \gamma \I_d)^{-1} \X^T \S \S^T \Y - (\X^T \X + n \gamma \I_d)^{-1} \X^T \Y  .
	\end{eqnarray*}
	It follows that
	\begin{align*}
	& (\X^T \S \S^T \X + n \gamma \I_d) ({\W}^{\textrm{c}} - \W^\star ) \\
	& = \; \X^T \S \S^T \Y^\perp + \X^T \S \S^T \X \X^\dag \Y
	- (\X^T \S \S^T \X + n \gamma \I_d) (\X^T \X + n \gamma \I_d)^{-1} \X^T \Y  \\
	& = \; \X^T \S \S^T \Y^\perp - n \gamma \X^\dag \Y
	+ (\X^T \S \S^T \X + n \gamma \I_d) \big[\X^\dag - (\X^T \X + n \gamma \I_d)^{-1} \X^T  \big] \Y \\
	& = \; \X^T \S \S^T \Y^\perp - n \gamma \X^\dag \Y
	+ n \gamma (\X^T \S \S^T \X + n \gamma \I_d) (\X^T \X + n \gamma \I_d)^{-1} \X^\dag  \Y \\
	& = \; \X^T \S \S^T \Y^\perp
	+ n \gamma (\X^T \S \S^T \X - \X^T \X ) (\X^T \X + n \gamma \I_d)^{-1} \X^\dag  \Y .
	\end{align*}
	It follows that
	\begin{eqnarray}  \label{eq:optimization:wtilde:0}
	(\X^T \X + n \gamma \I_d)^{-1/2} (\X^T  \S \S^T \X + n \gamma \I_d)
	({\W}^{\textrm{c}} - \W^\star )
	\; = \; \A + \B ,
	\end{eqnarray}
	where
	\begin{eqnarray*}
		\A & = & \big[ (\X^T \X + n \gamma \I_d)^{1/2} \big]^\dag \X^T \S \S^T \Y^\perp
		\; = \; \V (\Si^2 + n \gamma \I_\rho)^{-1/2} \Si \U \S \S^T \Y^\perp , \\
		\B & = & n \gamma  \big[ (\X^T \X + n \gamma \I_d)^{1/2} \big]^{\dag}
		(\X^T \S \S^T \X - \X^T \X ) (\X^T \X + n \gamma \I_d)^{\dag} \X^\dag  \Y \\
		& = & n \gamma \V (\Si^2 + n \gamma \I_\rho)^{-1/2} \Si
		(\U^T \S \S^T \U - \I_\rho ) \Si (\Si^2 + n \gamma \I_\rho )^{-1} \Si^{-1} \U^T  \Y\\
		& = & n \gamma \V \Si (\Si^2 + n \gamma \I_\rho)^{-1/2}
		(\U^T \S \S^T \U - \I_\rho ) (\Si^2 + n \gamma \I_\rho )^{-1}  \U^T \Y .
	\end{eqnarray*}
	It follows from \eqref{eq:optimization:wtilde:0} that
	\begin{align*}
	&(\X^T \X + n \gamma \I_d)^{1/2}
	\big({\W}^{\textrm{c}} - \W^\star \big) \\
	& = \;
	\big[(\X^T \X + n \gamma \I_d)^{-1/2} (\X^T  \S \S^T \X + n \gamma \I_d)
	(\X^T \X + n \gamma \I_d)^{-1/2} \big]^{\dag}
	\big( \A + \B \big) .
	\end{align*}
	By Assumption~\ref{assumption:1:1}, we have that
	\[
	(1 - \eta ) (\X^T \X + n \gamma \I_d)
	\; \preceq \; (\X^T  \S \S^T \X + n \gamma \I_d)
	\; \preceq \; (1 + \eta ) (\X^T \X + n \gamma \I_d) .
	\]
	It follows that
	\[
	\Big\| \big[(\X^T \X + n \gamma \I_d)^{-1/2} (\X^T  \S \S^T \X + n \gamma \I_d)
	(\X^T \X + n \gamma \I_d)^{-1/2} \big]^{\dag} \Big\|_2
	\; \leq \; \frac{1}{1-\eta}.
	\]
	Thus
	\begin{align*}
	& \Big\| (\X^T \X + n \gamma \I_d)^{1/2}
	\big({\W}^{\textrm{c}} - \W^\star \big) \Big\|_F^2
	\; \leq \;
	\frac{1}{1-\eta}
	\Big\| \A + \B \Big\|_F^2
	\; \leq \;
	\frac{2}{1-\eta}
	\Big( \big\| \A \big\|_F^2 + \big\| \B \big\|_F^2 \Big) .
	\end{align*}
	Lemma~\ref{lem:optimization} shows
	\begin{eqnarray} \label{eq:optimization:wtilde:2}
	f \big(\W^{\textrm{c}} \big) - f \big(\W^\star \big)
	=
	\frac{1}{n} \Big\| (\X^T \X + n \gamma \I_d)^{1/2} ({\W}^{\textrm{c}} - \W^\star ) \Big\|_F^2
	\leq   \frac{2}{n(1- \eta )}  \Big( \big\|\A\big\|_F^2 + \big\|\B \big\|_F^2  \Big) .
	\end{eqnarray}
	
	We respectively bound $\|\A\|_F^2$ and $\|\B\|_F^2$ in the following.
	It follows from Assumption~\ref{assumption:1:2} and $\U^T \Y^\perp = \0$ that
	\begin{eqnarray*}
		\| \A \|_F^2
		& = & \Big\| \V (\Si^2 + n \gamma \I_\rho)^{-1/2} \Si \U \S \S^T \Y^\perp \Big\|_F^2 \\
		& \leq &  \big\| (\Si^2 + n \gamma \I_\rho)^{-1/2} \Si \big\|_2^2 \,
		\big\| \U^T \S \S^T \Y^\perp - \U^T \Y^\perp \big\|_F^2  \\
		& \leq & {\epsilon } \big\| (\Si^2 + n \gamma \I_\rho )^{-1/2} \Si \big\|_2^2 \,
		\big\| \Y^\perp \big\|_F^2  .
	\end{eqnarray*}
	By the definition of $\B$, we have
	\begin{align*}
	& \| \B \|_F^2
	\; \leq \; n^2 \gamma^2 \big\| \Si (\Si^2 + n \gamma \I_\rho)^{-1/2}
	(\U^T \S \S^T \U - \I_\rho ) (\Si^2 + n \gamma \I_\rho )^{-1} \U^T  \Y \big\|_F^2 \\
	& \leq \; n^2 \gamma^2 \big\| \Si (\Si^2 + n \gamma \I_\rho)^{-1/2}
	(\U^T \S \S^T \U - \I_\rho ) (\Si^2 + n \gamma \I_\rho )^{-1/2} \big\|_2^2
	\big\| (\Si^2 + n \gamma \I_\rho )^{-1/2} \U^T  \Y \big\|_F^2 \\
	& = \; n^2 \gamma^2 \big\| \Si \N \big\|_2^2
	\big\| (\Si^2 + n \gamma \I_\rho )^{-1/2} \U^T  \Y \big\|_F^2,
	\end{align*}
	where we define $\N = (\Si^2 + n \gamma \I_\rho)^{-1/2}
	(\U^T \S \S^T \U - \I_\rho ) (\Si^2 + n \gamma \I_\rho )^{-1/2}$.
	By Assumption~\ref{assumption:1:1},
	we have
	\[
	- \eta (\Si^2 + n \gamma \I_\rho)^{-1}
	\; \preceq \;
	\N \; \preceq \;  \eta (\Si^2 + n \gamma \I_\rho)^{-1} .
	\]
	It follows that
	\begin{align*}
	& \| \B \|_F^2
	\; \leq \;  n^2 \gamma^2 \big\| \Si \N^2 \Si \big\|_2
	\big\| (\Si^2 + n \gamma \I_\rho )^{-1/2} \U^T  \Y \big\|_F^2 \\
	& \leq \; \eta^2 n^2 \gamma^2 \big\| \Si (\Si^2 + n \gamma \I_\rho)^{-2} \Si \big\|_2
	\big\| (\Si^2 + n \gamma \I_\rho )^{-1/2} \U^T  \Y \big\|_F^2 \\
	& = \; \eta^2 n^2 \gamma^2 \big\| (\Si^2 + n \gamma \I_\rho)^{-1} \Si \big\|_2^2
	\big\| (\Si^2 + n \gamma \I_\rho )^{-1/2} \U^T  \Y \big\|_F^2 \\
	& = \; \eta^2 n \gamma \big\| (\Si^2 + n \gamma \I_\rho)^{-1/2} \Si \big\|_2^2
	\big\| (\Si^2 + n \gamma \I_\rho )^{-1/2} \U^T  \Y \big\|_F^2 .
	\end{align*}
  The last equality follows from the fact that $\| (\Si^2 + n \gamma \I_\rho )^{-1/2} \|_2 \leq (n \gamma)^{-1/2}$.
	It follows that
	\begin{small}
		\begin{eqnarray}  \label{eq:optimization:wtilde:1}
		\| \A \|_F^2 + \| \B \|_F^2
		& \leq & {\max \big\{\epsilon , \eta^2 \big\}} \,
		\Big\| (\Si^2 + n \gamma \I_d)^{-1} \Si  \Big\|_2
		\Big[ \big\| \Y^\perp \big\|_F^2 +
		n \gamma  \big\| (\Si^2 + n \gamma \I_d)^{-1/2}  \U^T   \Y \big\|_F^2 \Big] \nonumber \\
		& \leq & {\max \big\{\epsilon , \eta^2 \big\}} \,
		\frac{\sigma_{\max}^2}{\sigma_{\max}^2 + n \gamma }
		\Big[ \big\| \Y^\perp \big\|_F^2 +
		n \gamma  \big\| (\Si^2 + n \gamma \I_d)^{-1/2}  \U^T   \Y \big\|_F^2 \Big] \nonumber \\
		& \leq & {\max \big\{\epsilon , \eta^2 \big\}} \,
		\beta n
		f(\W^\star ).
		\end{eqnarray}
	\end{small}%
	The last inequality follows from Lemma~\ref{lem:optimization}.
  The claimed result now follows from \eqref{eq:optimization:wtilde:1} and \eqref{eq:optimization:wtilde:2}.
\end{proof}


\subsection{Proof of Theorem~\ref{thm:optimization:what}}
\label{sec:optimization:proof:hessian}

\begin{proof}
	By the definition of ${\W}^{\textrm{h}}$ and $\W^\star$, we have
	\begin{align*}
	& (\X^T \X + n \gamma \I_d)^{1/2} \big({\W}^{\textrm{h}} - \W^\star \big) \\
	& = \;  (\X^T \X + n \gamma \I_d)^{1/2}
	\Big[ (\X^T \S \S^T \X + n \gamma \I_d)^\dag - (\X^T \X + n \gamma \I_d)^\dag \Big]
	\X^T \Y   \\
	& = \; \V (\Si^2 + n \gamma \I_\rho )^{1/2}
	\Big[ (\Si \U^T \S \S^T \U \Si + n \gamma \I_\rho )^\dag
	- (\Si^2 + n \gamma \I_\rho )^{-1}  \Big]
	\Si \U^T \Y .
	\end{align*}
	It follows from Assumption~\ref{assumption:1:1} that
	$\U^T \S \S^T \U$ has full rank, and thus
	\begin{align*}
	& (\X^T \X + n \gamma \I_d)^{1/2} \big({\W}^{\textrm{h}} - \W^\star \big) \\
	& = \; \V (\Si^2 + n \gamma \I_\rho )^{1/2}
	\Big[ (\Si \U^T \S \S^T \U \Si + n \gamma \I_\rho )^{-1}
	- (\Si^2 + n \gamma \I_\rho )^{-1}  \Big]
	\Si \U^T \Y \\
	& = \; \V (\Si^2 + n \gamma \I_\rho )^{1/2} (\Si^2 + n \gamma \I_\rho )^{-1}
	( \Si^2 - \Si \U^T \S \S^T \U \Si ) (\Si \U^T \S \S^T \U \Si + n \gamma \I_\rho )^{-1}
	\Si \U^T \Y \\
	& = \; \V (\Si^2 + n \gamma \I_\rho )^{-1/2}\Si
	( \I_\rho - \U^T \S \S^T \U  )\Si (\Si \U^T \S \S^T \U \Si + n \gamma \I_\rho )^{-1}
	\Si \U^T \Y ,
	\end{align*}
	where the second equality follow from
	$\M^{-1} - \N^{-1} = \N^{-1} (\N - \M) \M^{-1} $.
	We define
	\begin{align*}
	& (\X^T \X + n \gamma \I_d)^{1/2} \big({\W}^{\textrm{h}} - \W^\star \big)
	\; = \; \V \A \B \C ,
	\end{align*}
	where
	\begin{eqnarray*}
		\A & = & (\Si^2 + n \gamma \I_\rho)^{-1/2} \Si (\I_\rho - \U^T \S \S^T \U)
		\Si (\Si^2 + n \gamma \I_\rho)^{-1/2} ,\\
		\B & = & (\Si^2 + n \gamma \I_\rho)^{1/2}
		(\Si \U^T \S \S^T \U \Si + n \gamma \I_\rho)^{-1} (\Si^2 + n \gamma \I_\rho)^{1/2} , \\
		\C & = & (\Si^2 + n \gamma \I_\rho)^{-1/2} \Si \U^T \Y .
	\end{eqnarray*}
	It follows from Assumption~\ref{assumption:1:1} that
	\begin{eqnarray*}
		\| \A \|_2
		& \leq & \eta \Big\| (\Si^2 + n \gamma \I_\rho)^{-1/2} \Si^2
		(\Si^2 + n \gamma \I_\rho)^{-1/2} \Big\|_2
		\; \leq \;  \eta \beta ,\\
		\|\B\|_2
		& \leq & (1-\eta)^{-1} .
	\end{eqnarray*}
	It holds that
	\begin{eqnarray*}
		\big\| \C \big\|_F^2
		& \leq &
		\Big\| (\Si^2 + n \gamma \I_\rho)^{-1/2} \Si \U^T \Y \Big\|_F^2 \\
		& = & \bigg[ \tr \Big( \Y^T \U \U^T \Y \Big)
		- n \gamma \, \tr \Big(\Y^T \U  (\Si^2 + n \gamma \I_d)^{-1 } \U^T \Y \Big)
		\bigg] \\
		& = &  \bigg[ - \tr \Big( \Y^T (\I_d - \U \U^T) \Y \Big)
		- n \gamma \, \tr \Big(\Y^T \U  (\Si^2 + n \gamma \I_d)^{\dag } \U^T \Y \Big)+ \tr \big( \Y^T \Y \big) \bigg]
		\\
		& = & \Big(- n f(\W^\star)  +   \, \big\| \Y \big\|_F^2 \Big) ,
	\end{eqnarray*}
	where the last equality follows from Lemma~\ref{lem:optimization}.
	It follows from Lemma~\ref{lem:optimization} that
	\begin{eqnarray*}
		f ({\W}^{\textrm{h}} ) - f(\W^\star )
		& = & \frac{1}{n}\big\| (\X^T \X + n \gamma \I_d)^{1/2} 
		\big({\W}^{\textrm{h}} - \W^\star \big)  \big\|_F^2 \\
		& = & \frac{1}{n}\big\| \A \B \C \big\|_F^2
		\; \leq  \; \frac{\eta^2 \beta^2 }{(1-\eta)^2 }
		\Big(\frac{1 }{n} \, \big\| \Y \big\|_F^2 -  f(\W^\star)  \Big).
	\end{eqnarray*}
\end{proof}


\section{Sketched MRR from the Statistical Perspective: Proofs}
\label{sec:statistical:proof}

In Section~\ref{sec:statistical:proof:bv} we prove Theorem~\ref{thm:bias_var_decomp}.
In Section~\ref{sec:statistical:proof:classical} we prove Theorem~\ref{thm:biasvariance:wtilde}.
In Section~\ref{sec:statistical:proof:lower1} we prove Theorem~\ref{thm:biasvariance:wtilde_lower1}.
In Section~\ref{sec:statistical:proof:lower2} we prove Theorem~\ref{thm:biasvariance:wtilde_lower2}.
In Section~\ref{sec:statistical:proof:hessian} we prove Theorem~\ref{thm:biasvariance:what}.
Recall that the fixed design model is
$\Y = \X \W_0 + \Xii$
where $\Xii$ is random, $\EB \Xii = 0$, and $\EB [\Xii \Xii^T] = \xi^2 \I_n$.


\subsection{Proofs of Theorem~\ref{thm:bias_var_decomp}}
\label{sec:statistical:proof:bv}

We prove Theorem~\ref{thm:bias_var_decomp} in the following.
In the proof we exploit several identities.
The Frobenius norm and matrix trace satisfy 
\[
\|\A - \B \|_F^2
\; = \; \tr \big[ (\A - \B) ( \A - \B)^T) \big]
\; = \; \tr (\A \A^T) + \tr (\B \B^T ) - 2 \tr (\A \B^T ) 
\]
for any conformal matrices $\A$ and $\B$.
The trace is linear, and thus
for any fixed $\A$ and $\B$ and conformal random matrix $\Ps$, 
\[
\EB \big[ \tr (\A \Ps \B) \big]
\; = \; \tr \big[ \A (\EB \Ps ) \B \big],
\]
where the expectation is taken with respect to $\Ps$.

\begin{proof}
	It follows from the definition of
	the optimal solution $\W^\star$ in \eqref{eq:def_W_optimal} that
	\begin{eqnarray*}
		\X \W^\star
		& = & \X (\X^T \X + n \gamma \I_d )^\dag \X^T (\X \W_0 + \Xii) \\
		& = & \U (\Si^2 + n \gamma \I_\rho )^{-1} \Si^3 \V^T \W_0
		+ \U (\Si^2 + n \gamma \I_\rho )^{-1} \Si^2 \U^T \Xii \\
		& = & \U \Big[ \I_\rho - n \gamma (\Si^2 + n \gamma \I_\rho)^{-1} \Big] \Si \V^T \W_0
		+ \U (\Si^2 + n \gamma \I_\rho )^{-1} \Si^2 \U^T \Xii\\
		& = & \X \W_0 -  n \gamma \U (\Si^2 + n \gamma \I_\rho)^{-1}  \Si \V^T \W_0
		+ \U (\Si^2 + n \gamma \I_\rho )^{-1} \Si^2 \U^T \Xii .
	\end{eqnarray*}
	Since $\EB [\Xii] = \0$ and $\EB [\Xii \Xii^T ] = \xi^2 \I_n$,
	it holds that
	\begin{eqnarray*}
		R (\W^\star )
		& = & \frac{1}{n} \EB \big\|\X \W^\star - \X \W_0 \big\|_F^2 \\
		& = & \frac{1}{n}  \Big\| -  n \gamma (\Si^2 + n \gamma \I_\rho)^{-1}  \Si \V^T \W_0
		+ (\Si^2 + n \gamma \I_\rho )^{-1} \Si^2 \U^T \Xii \Big\|_F^2 \\
		& = &  n \gamma^2 \Big\|  (\Si^2 + n \gamma \I_\rho)^{-1}  \Si \V^T \W_0 \Big\|_F^2
		+ \frac{\xi^2}{n} \Big\| (\Si^2 + n \gamma \I_\rho )^{-1} \Si^2 \Big\|_F^2 .
	\end{eqnarray*}
  This exposes expressions for the bias and variance of the optimal solution $\W^\star$.

  We now decompose the risk function $R \big({\W}^{\textrm{c}} \big)$.
	It follows from the definition of ${\W}^{\textrm{c}}$ in \eqref{eq:def_W_tilde} that
	\begin{align*}
	&\X \W^{\textrm{c}}
	\; = \;  \X (\X^T \S \S^T \X + n\gamma \I_d )^{\dag} \X^T \S \S^T (\X \W_0 + \Xii) \\
	& = \; \U \Si \big( \Si \U^T \S \S^T \U \Si + n \gamma \I_d \big)^\dag
	\Si \big( \U^T \S \S^T \U \Si \V^T \W_0 + \U^T \S \S^T \Xii \big)
	\\
	& = \; \U (\U^T \S \S^T \U + n \gamma \Si^{-2} )^{-1}
	\Big[ (\U^T \S \S^T \U + n \gamma \Si^{-2} ) \Si \V^T \W_0
	- n \gamma \Si^{-1} \V^T \W_0
	+ \U^T \S \S^T \Xii\Big] \\
	& = \; \X \W_0
	+ \U (\U^T \S \S^T \U + n \gamma \Si^{-2} )^{-1}
	\big( - n \gamma \Si^{-1} \V^T \W_0
	+ \U^T \S \S^T \Xii\big) .
	\end{align*}
	Since $\EB [\Xii] = \0$ and $\EB [\Xii \Xii^T ] = \xi^2 \I_n$,
	it follows that
	\begin{align*}
	& R \big({\W}^{\textrm{c}} \big)
	\; = \; \frac{1}{n} \EB \big\|\X {\W}^{\textrm{c}} - \X \W_0 \big\|_F^2 \\
	& = \; \frac{1}{n}  \Big\| -  n \gamma (\U^T \S \S^T \U + n \gamma \Si^{-2} )^{-1}  \Si^{-1} \V^T \W_0
	+ (\U^T \S \S^T \U + n \gamma \Si^{-2} )^{-1} \U^T \S \S^T \Xii \Big\|_F^2 \\
	& = \;  n \gamma^2 \Big\| (\U^T \S \S^T \U + n \gamma \Si^{-2} )^{-1}  \Si^{-1} \V^T \W_0 \Big\|_F^2
	+ \frac{\xi^2}{n} \Big\| (\U^T \S \S^T \U + n \gamma \Si^{-2} )^{-1} \U^T \S \S^T  \Big\|_F^2 .
	\end{align*}
  This exposes expressions for the bias and variance of the approximate solution $\W^{\textrm{c}}$.

  We now decompose the risk function $R \big({\W}^{\textrm{h}}\big)$.
	It follows from the definition of ${\W}^{\textrm{h}}$ in \eqref{eq:def_W_hat} that
	\begin{align*}
	&  \X {\W}^{\textrm{h}} - \X \W_0
	\; = \; \X (\X^T \S \S^T \X + n \gamma \I_n )^\dag \X^T (\X \W_0 + \Xii ) - \X \W_0 \\
	& = \; \X (\X^T \S \S^T \X + n \gamma \I_d)^\dag \X^T \X \W_0 - \X \W_0
	+ \X (\X^T \S \S^T \X + n \gamma \I_d)^\dag \X^T \Xii  \\
	& = \; \U \big[ (\U^T \S \S^T \U + n \gamma \Si^{-2})^{-1} - \I_\rho^{-1} \big] \U^T \X \W_0
	+ \U (\U^T \S \S^T \U^T + n \gamma \Si^{-2})^\dag \U^T \Xii  \\
	& = \; \U \big( \I_\rho - \U^T \S \S^T \U - n \gamma \Si^{-2} \big)
	\big(\U^T \S \S^T \U + n \gamma \Si^{-2} \big)^{-1} \Si \V^T \W_0 \\
	& \quad + \U (\U^T \S \S^T \U^T + n \gamma \Si^{-2})^\dag \U^T \Xii ,
	\end{align*}
  where the last equality follows from the fact that
	$\A^{-1} - \B^{-1} = \B^{-1} (\B - \A ) \A^{-1}$ for any
  conformal nonsingular matrices $\A$ and $\B$.
	Since $\EB [\Xii] = \0$ and $\EB [\Xii \Xii^T ] = \xi^2 \I_n$,
	it follows that
	\begin{align*}
	R \big({\W}^{\textrm{h}} \big)
	\; = \; \bias^2 \big({\W}^{\textrm{h}} \big) + \var \big({\W}^{\textrm{h}} \big) ,
	\end{align*}
	where
	\begin{eqnarray*}
		\bias^2 \big({\W}^{\textrm{h}} \big)
		& = & \frac{1 }{n} \Big\|
		\big( n \gamma \Si^{-2} + \U^T \S \S^T \U - \I_\rho \big)
		\big(\U^T \S \S^T \U + n \gamma \Si^{-2} \big)^{-1} \Si \V^T \W_0 \Big\|_F^2 ,\\
		\var \big({\W}^{\textrm{h}} \big)
		& = & \frac{\xi^2 }{n} \Big\| \big(\U^T \S \S^T \U + n \gamma \Si^{-2} \big)^{-1} \Big\|_F^2 .
	\end{eqnarray*}
  This exposes expressions for the bias and variance of ${\W}^{\textrm{h}}$.
\end{proof}


\subsection{Proof of Theorem~\ref{thm:biasvariance:wtilde}}
\label{sec:statistical:proof:classical}

\begin{proof}
	Assumption~\ref{assumption:1:1} ensures that 
	$(1-\eta ) \I_\rho \preceq \U^T \S \S^T \U \preceq (1+\eta ) \I_\rho$.
	It follows that
	\[
	(1-\eta) \big( \I_\rho + n \gamma \Si^{-2} \big)
	\; \preceq \; \U^T \S \S^T \U + n \gamma \Si^{-2}
	\; \preceq \; (1+\eta) \big( \I_\rho + n \gamma \Si^{-2} \big) .
	\]
	The bias term can be written as
	\begin{eqnarray*}
		\bias^2 \big( \W^{\textrm{c}} \big)
		& = & n \gamma^2 \big\|
		\big( \U^T \S \S^T \U + n \gamma \Si^{-2} \big)^{\dag} \Si^{-1} \V^T \W_0 \big\|_F^2 \\
		& = & n \gamma^2 \, \tr \Big(
		\W_0^T \V \Si^{-1} 
		\big[ (\U^T \S \S^T \U + n \gamma \Si^{-2} )^\dag \big]^{2}
		\Si^{-1} \V^T \W_0  \Big) \\
		& \leq & \tfrac{n \gamma^2}{(1- \eta)^2 } 
		\big\| \big( \I_\rho + n \gamma \Si^{-2} \big)^{-1} \Si^{-1} \V^T \W_0 \big\|_F^2 \\
		& = & \tfrac{n \gamma^2}{(1- \eta)^2 } 
		\big\| \big( \Si^{2} + n \gamma \I_\rho \big)^{-1} \Si \V^T \W_0 \big\|_F^2 \\
		& = & \tfrac{1}{(1- \eta)^2 } \: \bias^2 (\W^\star ).
	\end{eqnarray*}
	We can analogously show
	$\bias^2 ({\W}^{\textrm{c}}) \geq \frac{1}{(1+\eta )^2 } \bias^2 (\W^\star )$.

	Let
	$\B = \big( \U^T \S \S^T \U + n\gamma \Si^{-2} \big)^{\dag} \U^T \S \in \RB^{\rho \times s}$.
	By Assumption~\ref{assumption:1:1}, it holds that
	\begin{eqnarray*}
		(1-\eta ) \big[ \big( \U^T \S \S^T \U + n\gamma \Si^{-2}  \big)^{2} \big]^{\dag}
		\; \preceq \; \B \B^T
		\; \preceq \; (1+ \eta ) \big[ \big( \U^T \S \S^T \U + n\gamma \Si^{-2}  \big)^{2} \big]^{\dag} .
	\end{eqnarray*}
	Applying Assumption~\ref{assumption:1:1} again, we obtain
	\begin{eqnarray*}
		(1-\eta )^2 \big( \I_\rho + n\gamma \Si^{-2}  \big)^{2} 
		\; \preceq \;  \big( \U^T \S \S^T \U + n\gamma \Si^{-2}  \big)^{2} 
		\; \preceq \; (1+\eta )^2 \big( \I_\rho + n\gamma \Si^{-2}  \big)^{2} .
	\end{eqnarray*}
	Note that both sides are nonsingular.
	Combining the above two equations, we have
	\begin{eqnarray*}
		\tfrac{1-\eta}{(1+\eta )^2} \big( \I_\rho + n\gamma \Si^{-2} \big)^{-2}
		\; \preceq \; \B \B^T
		\; \preceq \; \tfrac{1+\eta}{(1-\eta)^2} \big( \I_\rho + n\eta \Si^{-2} \big)^{-2} .
	\end{eqnarray*}
	Taking the trace of all the terms, we obtain
	\[
	\tfrac{1-\eta}{(1+\eta)^2}
	\; \leq \; \tfrac{\| \B \|_F^2}{
		\| ( \I_\rho + n\gamma \mathbf{\Sigma}^{-2} )^{-1} \|_F^2}
	\; \leq \; \tfrac{1+\eta}{(1-\eta)^2} .
	\]
	The variance term can be written as
	\begin{eqnarray*}
		\var \big( \W^{\textrm{c}} \big)
		& = & \tfrac{\xi^2 }{n}   \big\| \B \S^T  \big\|_F^2
		\; \leq \; \tfrac{\xi^2}{n}   \big\|\B\big\|_F^2 \, \big\| \S \big\|_2^2 \\
		& \leq & \tfrac{\xi^2 (1+\eta )}{n (1-\eta)^2}   \big\| \big(  \I_\rho + n\gamma \Si^{-2} \big)^{-1} \big\|_F^2
		\, \big\| \S \big\|_2^2  \\
		& = & \tfrac{ (1+\eta ) \| \S \|_2^2}{ (1-\eta)^2} \var (\W^\star ) .
	\end{eqnarray*}
	The upper bound on the variance follows from Assumption~\ref{assumption:1:3}.
\end{proof}


\subsection{Proof of Theorem~\ref{thm:biasvariance:wtilde_lower1}}
\label{sec:statistical:proof:lower1}

\begin{proof}
	Let $\B = \big( \U^T \S \S^T \U + n\gamma \Si^{-2} \big)^{\dag} \U^T \S \in \RB^{\rho \times s}$.
	In the proof of Theorem~\ref{thm:biasvariance:classical} we show that
	\begin{eqnarray*}
		\var \big( \W^{\textrm{c}} \big)
		& = & \tfrac{\xi^2 }{n}   \big\| \B \S^T  \big\|_F^2  .
	\end{eqnarray*}
	If $\S^T \S \succeq \frac{\vartheta n}{s} \I_s$, then it holds that
	\begin{eqnarray*}
		\var \big( \W^{\textrm{c}} \big)
		& = & \tfrac{\xi^2}{n}   \big\| \B \S^T  \big\|_F^2
		\; \geq \; \tfrac{\vartheta n}{s} \tfrac{\xi^2}{n}  \big\| \B \big\|_F^2
		\; \geq \; \tfrac{\vartheta n}{s} \tfrac{  1- \eta }{ (1+\eta)^2} \var (\W^\star ) .
	\end{eqnarray*}
	This establishes the lower bounds on the variance.
\end{proof}


\subsection{Proof of Theorem~\ref{thm:biasvariance:what}}
\label{sec:statistical:proof:hessian}

\begin{proof}
	Theorem~\ref{thm:bias_var_decomp} shows that
	\begin{eqnarray*}
		\bias \big({\W}^{\textrm{h}}\big)
		& = & \gamma \sqrt{n} \bigg\|
		\Big(  \Si^{-2} + \tfrac{\U^T \S \S^T \U - \I_\rho}{n \gamma} \Big)
		\big(\U^T \S \S^T \U + n \gamma \Si^{-2} \big)^{\dag} \Si \V^T \W_0 \bigg\|_F   \\
		& = & \gamma \sqrt{n} \big\| \A \Si^2 \B \big\|_F
		\; \leq \; \gamma \sqrt{n} \big\| \A \Si^2 \big\|_2 \big\|  \B \big\|_F   ,\\
		\var \big( {\W}^{\textrm{h}} \big)
		& = & \frac{\xi^2 }{n} \Big\| \big(\U^T \S \S^T \U + n \gamma \Si^{-2} \big)^{\dag } \Big\|_F^2 ,
	\end{eqnarray*}
	where we define
	\begin{eqnarray*}
		\A & = & \Si^{-2} + \tfrac{\U^T \S \S^T \U - \I_\rho}{n \gamma}  , \\
		\B & = & \Si^{-2} \big(\U^T \S \S^T \U + n \gamma \Si^{-2} \big)^{\dag} \Si \V^T \W_0 .
	\end{eqnarray*}
	
	We first analyze the bias.
	It follows from Assumption~\ref{assumption:1:1} that
	\begin{eqnarray} \label{eq:biasvariance:what:1}
	\Si^{-2} \big( \I_\rho  - \tfrac{\eta }{n \gamma} \Si^2 \big)
	\; \preceq \; \A
	\; \preceq \; \Si^{-2} \big( \I_\rho  + \tfrac{\eta }{n \gamma} \Si^2 \big) .
	\end{eqnarray}
	Since $\big( \I_\rho  - \frac{\eta }{n \gamma} \Si^2 \big)^2 \preceq
	\big( \I_\rho  + \frac{\eta }{n \gamma} \Si^2 \big)^2
	\preceq \big( 1 + \frac{\eta \sigma_{1}^2 }{n \gamma} \big)^2 \I_\rho $,
	it follows that
	\begin{eqnarray*}
		\A^2
		\; \preceq \; \Si^{-4} \big( \I_\rho  + \tfrac{\eta }{n \gamma} \Si^2 \big)^2
		\; \preceq \; \big( 1 + \tfrac{\eta \sigma_{1}^2 }{n \gamma} \big)^2 \Si^{-4}  .
	\end{eqnarray*}
	Thus
	\[
	\big\| \A \Si^2 \big\|_2^2
	\; = \; \big\| \Si^2 \A^2 \Si^2 \big\|_2
	\; \leq \; \Big( 1 + \tfrac{\eta \sigma_{1}^2 }{n \gamma} \Big)^2.
	\]
	It follows from Assumption~\ref{assumption:1:1} that
	\begin{align*}
	& (1+\eta )^{-1} \big( \I_\rho + n \gamma \Si^{-2} \big)^{-1}
	\; \preceq \;  \big( (1+\eta )\I_\rho + n \gamma \Si^{-2} \big)^{-1} \\
	& \; \preceq \; \big(\U^T \S \S^T \U + n \gamma \Si^{-2} \big)^{\dag}
	\; \preceq \; \big( (1-\eta ) \I_\rho + n \gamma \Si^{-2} \big)^{-1}
	\; \preceq \; (1-\eta )^{-1} \big( \I_\rho + n \gamma \Si^{-2} \big)^{-1}.
	\end{align*}
	Thus
	\begin{eqnarray}\label{eq:biasvariance:what:2}
	\B^T \B & = &
	\W_0^T \V \Si^3 \big( \Si^{-2} (\U^T \S \S^T \U + n \gamma \Si^{-2})^\dag
	\Si^{-2} \big)^{2} \Si^3 \V^T \W_0  \nonumber\\
	& \preceq & (1-\eta )^{-2} \W_0^T \V \Si^3
	\big(\Si^{-2} (\I_\rho + n \gamma \Si^{-2})^{-1} \Si^{-2} \big)^{2}
	\Si^3 \V^T \W_0 \nonumber  \\
	& = & (1-\eta )^{-2} \W_0^T \V \Si
	\big(\Si^2 + n \gamma \I_\rho  \big)^{-2}
	\Si \V^T \W_0 . 
	\end{eqnarray}
	It follows that
	\begin{eqnarray*}
		\|\B \|_F^2
		\; = \; \tr \big( \B^T \B \big)
		\; \leq \; (1-\eta )^{-2} \big\| \big( \Si^{-2} + n \gamma \I_\rho \big)^{-1} \Si \V^T \W_0 \big\|_F^2
		\; = \; \tfrac{\bias^2 (\W^\star)}{n \gamma^2 (1-\eta )^2} ,
	\end{eqnarray*}
	where the last equality follows from the definition of $\bias (\W^\star )$.
	By the definition of $\A$ and $\B$, we have
	\begin{align*}
	& \bias^2 \big({\W}^{\textrm{h}}\big)
	\; \leq \; \gamma^2 n \, \big\|\A \Si^2 \big\|_2^2 \big\| \B \big\|_F^2
	\; = \;  \tfrac{1}{(1 - \eta)^{2} } \,
	\big( 1 + \tfrac{\eta \sigma_{1}^2 }{n \gamma} \big)^2 \, \bias^2 \big( \W^\star \big) .
	\end{align*}
  Thus, the upper bound on $\bias \big({\W}^{\textrm{h}} \big)$ is established.
	
  Using the same $\A$ and $\B$, we can also show that
	\begin{eqnarray*}
		\bias \big({\W}^{\textrm{h}}\big)
		& = & \gamma \sqrt{n} \big\| \A \Si^2 \B \big\|_F
		\; \geq \; \gamma \sqrt{n} \; \sigma_{\min} \big( \A \Si^2  \big) \, \big\|  \B \big\|_F .
	\end{eqnarray*}
	Assume that $\sigma_\rho^2 \geq \frac{n \gamma}{\eta }$.
	It follows from \eqref{eq:biasvariance:what:1} that
	\[
	\A^2 \; \succeq \;
	\big( \tfrac{\eta \sigma_{\rho}^2 }{n \gamma} - 1\big)^2 \Si^{-4} .
	\]
	Thus
	\[
	\sigma_{\min}^2 \big( \A \Si^2  \big)
	\; = \;
	\sigma_{\min} (\Si^2 \A^2 \Si^2)
	\; \geq \; \big( \tfrac{\eta \sigma_{\rho}^2 }{n \gamma} - 1 \big)^2 .
	\]
	It follows from \eqref{eq:biasvariance:what:2} that
	\begin{eqnarray*}
		\B^T \B & \succeq &
		(1+\eta )^{-2} \W_0^T \V \Si
		\big(\Si^2 + n \gamma \I_\rho  \big)^{-2}
		\Si \V^T \W_0 . \nonumber
	\end{eqnarray*}
	Thus
	\begin{eqnarray*}
		\|\B \|_F^2
		\; = \; \tr \big( \B^T \B \big)
		\; \geq \; (1+\eta )^{-2} \big\| \big( \Si^{-2} + n \gamma \I_\rho \big)^{-1} \Si \V^T \W_0 \big\|_F^2
		\; = \; \tfrac{1}{n \gamma^2 (1+\eta )^2} \, \bias^2 (\W^\star) .
	\end{eqnarray*}
	In sum, we obtain
	\begin{align*}
	& \bias^2 \big({\W}^{\textrm{h}}\big)
	\; \geq \; \gamma^2 {n} \; \sigma_{\min}^2 \big( \A \Si^2  \big) \, \big\|  \B \big\|_F^2
	\; = \;  (1 + \eta)^{-2} \,
	\big(  \tfrac{\eta \sigma_{\rho}^2}{n \gamma}  - 1 \big)^2 \, \bias^2 \big( \W^\star \big) .
	\end{align*}
  Thus, the lower bound on $\bias \big({\W}^{\textrm{h}} \big)$ is established.

	It follows from Assumption~\ref{assumption:1:1} that
	\[
	(1+\eta )^{-1} \, \big( \I_\rho + n \gamma \Si^{-2} \big)^{-1}
	\; \preceq \; \big(\U^T \S \S^T \U + n \gamma \Si^{-2} \big)^{-1}
	\; \preceq \; (1-\eta )^{-1} \, \big( \I_\rho + n \gamma \Si^{-2} \big)^{-1} .
	\]
	It follows from Theorem~\ref{thm:bias_var_decomp} that
	\begin{eqnarray*}
		\var \big( {\W}^{\textrm{h}} \big)
		& = & \tfrac{\xi^2 }{n} \big\| \big(\U^T \S \S^T \U + n \gamma \Si^{-2} \big)^{-1} \big\|_F^2 \\
		& \in & \tfrac{1}{1\mp \eta}  \tfrac{\xi^2 }{n}
		\big\| \big( \I_\rho + n \gamma \Si^{-2} \big)^{-1} \big\|_F^2 \\
		& = & \tfrac{1}{1\mp \eta} \var \big( \W^\star \big) .
	\end{eqnarray*}
	This concludes the proof.
\end{proof}


\section{Model Averaging from the Optimization Perspective: Proofs} 
\label{sec:opt_avg:proof}

In Section~\ref{sec:opt_avg:proof:classical} we prove Theorem~\ref{thm:optimization:wtilde_avg}.
In Section~\ref{sec:opt_avg:proof:hessian} we prove Theorem~\ref{thm:optimization:what_avg}.

\subsection{Proof of Theorem~\ref{thm:optimization:wtilde_avg}}
\label{sec:opt_avg:proof:classical}

\begin{proof}
	By Lemma~\ref{lem:optimization},
  we only need to show that
	$ \| (\X^T \X + n \gamma \I_d)^{1/2}  ({\W}^{\textrm{c}} - \W^\star ) \|_F^2
	\leq n \alpha \beta f (\W^\star )$.
  In the proof, we define $\rho = \rk (\X)$ and let $\sigma_1 \geq \cdots \geq \sigma_\rho$
	be the singular values of $\X$.
	
	In the proof of Theorem~\ref{thm:optimization:wtilde} we show that
	\begin{align*}
	&(\X^T \X + n \gamma \I_d)^{1/2}
	\big({\W}^{\textrm{c}}_i - \W^\star \big) \\
	& = \;
	\big[(\X^T \X + n \gamma \I_d)^{-1/2} (\X^T  \S_i \S_i^T \X + n \gamma \I_d)
	(\X^T \X + n \gamma \I_d)^{-1/2} \big]^{\dag}
	\big( \A_i + \B_i \big) \\
	& = \; \C_i^\dag \big( \A_i + \B_i \big) ,
	\end{align*}
	where
	\begin{eqnarray*}
		\A_i & = &
		\V (\Si^2 + n \gamma \I_\rho)^{-1/2} \Si \U \S_i \S_i^T \Y^\perp , \\
		\B_i & = &
		n \gamma \V \Si (\Si^2 + n \gamma \I_\rho)^{-1/2}
		(\U^T \S_i \S_i^T \U - \I_\rho ) (\Si^2 + n \gamma \I_\rho )^{-1}  \U^T \Y\\
		\C_i
		& = & \big[  (\X^T \X + n \gamma \I_d)^{1/2} \big]^\dag
		\big( \X^T \S_i \S_i^T \X + n \gamma \I_d \big)
		\big[  (\X^T \X + n \gamma \I_d)^{1/2} \big]^\dag \\
		& = &\V (\I_\rho + n \gamma \Si^{-2} )^{-1/2} ( \U^T  \S_i \S_i^T \U  + n \gamma \Si^{-2})
		(\I_\rho + n \gamma \Si^{-2} )^{-1/2} \V^T \\
		& = & \V \V^T + \V (\I_\rho + n \gamma \Si^{-2} )^{-1/2}
		( \U^T  \S_i \S_i^T \U  - \I_\rho )
		(\I_\rho + n \gamma \Si^{-2} )^{-1/2} \V^T.
	\end{eqnarray*}
	By Assumption~\ref{assumption:2:1},
	we have that
	$
	\C_i \, \succeq \, \big( 1 - \frac{\eta \: \sigma_{\max}^2 }{\sigma_{\max}^2  + n \gamma} \big) \V \V^T
	$.
	Since $\eta \leq 1/2$, it follows that
	$
	\C_i^\dag \, \preceq \, \big( 1 + \frac{2\eta \: \sigma_{\max}^2 }{\sigma_{\max}^2  + n \gamma} \big) \V \V^T
	$.
	Let $\C_i^\dag = \V \V^T + \V \De_i \V^T $.
	It holds that
	$
	\De_i
	\, \preceq \, \frac{2\eta \: \sigma_{\max}^2 }{\sigma_{\max}^2  + n \gamma} \, \V \V^T
	\, \preceq \, 2 \eta \beta \V \V^T
	$.
	By definition, ${\W}^{\textrm{c}} = \frac{1}{g} \sum_{i=1}^g {\W}^{\textrm{c}}_i$.
	It follows that
	\begin{align} \label{eq:avg:optimization:wtilde:2}
	& \Big\| (\X^T \X + n \gamma \I_d)^{1/2}  ({\W}^{\textrm{c}}_i - \W^\star ) \Big\|_F
	\; = \; \Big\| \frac{1}{g} \sum_{i=1}^g \C_i^\dag (\A_i + \B_i )  \Big\|_F \nonumber \\
	& \leq \; \Big\| \frac{1}{g} \sum_{i=1}^g (\A_i + \B_i ) \Big\|_F
	+ \Big\| \frac{1}{g} \sum_{i=1}^g \V \De_i \V^T (\A_i + \B_i ) \Big\|_F  \nonumber \\
	& \leq \; \Big\| \frac{1}{g} \sum_{i=1}^g \A_i  \Big\|_F
	+ \Big\| \frac{1}{g} \sum_{i=1}^g \B_i  \Big\|_F
	+  \frac{1}{g} \sum_{i=1}^g \big\|\De_i \big\|_2
	\Big( \big\| \A_i \big\|_F+\big\| \B_i \big\|_F \Big) \nonumber \\
	& \leq \; \Big\| \frac{1}{g} \sum_{i=1}^g \A_i  \Big\|_F
	+ \Big\| \frac{1}{g} \sum_{i=1}^g \B_i  \Big\|_F
	+  2\eta \beta
	\frac{1}{g} \sum_{i=1}^g \Big( \big\| \A_i \big\|_F +  \big\| \B_i \big\|_F \Big) .
	\end{align}
	By Assumption~\ref{assumption:2:3}, we have that
	\begin{eqnarray*}
		\frac{1}{g} \sum_{i=1}^g \big\| \A_i \big\|_F
		& = & \big\| (\Si^2 + n \gamma \I_d)^{-1/2} \Si \big\|_2 \cdot
		\frac{1}{g} \sum_{i=1}^g \big\| \U^T \S_i \S_i^T \Y^\perp \big\|_F
		\; \leq \; \sqrt{ \tfrac{\epsilon \, \sigma_{\max}^2}{\sigma_{\max}^2 + n \gamma}}
		\big\|  \Y^\perp \big\|_F .
	\end{eqnarray*}
	We apply Assumption~\ref{assumption:2:1} and follow the proof of Theorem~\ref{thm:optimization:wtilde}
	to show that
	\begin{eqnarray*}
		\big\| \B_i \big\|_F^2
		& \leq & {\eta^2 n \gamma} \tfrac{ \sigma_{\max}^2}{\sigma_{\max}^2 + n \gamma} \,
		\Big\| (\Si^2 + n \gamma \I_d)^{-1/2}  \U^T   \Y \Big\|_F^2 .
	\end{eqnarray*}
	It follows that
	\begin{align}\label{eq:avg:optimization:wtilde:3}
	& \frac{1}{g} \sum_{i=1}^g \Big( \big\| \A_i \big\|_F + \big\| \B_i \big\|_F \Big) \nonumber \\
	& \leq \;  \max \Big\{ \sqrt{\epsilon}, \eta \Big\}
	\sqrt{ \tfrac{\sigma_{\max}^2}{\sigma_{\max}^2 + n \gamma}}
	\Big( \big\|  \Y^\perp \big\|_F
	+ \sqrt{n\gamma } \big\| (\Si^2 + n \gamma \I_d)^{-1/2}  \U^T   \Y \big\|_F \Big) \nonumber \\
	& \leq \;  \max \Big\{ \sqrt{\epsilon}, \eta \Big\}
	\sqrt{ \beta}
	\sqrt{ 2\big\|  \Y^\perp \big\|_F^2
		+ 2 {n\gamma } \big\| (\Si^2 + n \gamma \I_d)^{-1/2}  \U^T   \Y \big\|_F^2  } \nonumber  \\
	& = \;  \max \big\{ \sqrt{\epsilon}, \eta \big\}
	\sqrt{ \beta} \,
	\sqrt{ 2n \, f (\W^\star)  } .
	\end{align}
	Here the equality follows from Lemma~\ref{lem:optimization}.
	Let $\S = \frac{1}{g} [\S_1 , \cdots , \S_g ] \in \RB^{n\times sg}$.
	We have that
	\begin{eqnarray*}
		\frac{1}{g} \sum_{i=1}^g \A_i
		& = & \V (\Si^2 + n \gamma \I_d)^{-1/2} \Si \U^T \S \S^T \Y^\perp ,\\
		\frac{1}{g} \sum_{i=1}^g \B_i
		& = & n \gamma \V \Si  (\Si^2 + n \gamma \I_d)^{-1/2}
		(\U^T \S \S^T \U - \I_\rho )
		(\Si^2 + n \gamma \I_\rho)^{-1} \U^T  \Y .
	\end{eqnarray*}
	Applying Assumptions~\ref{assumption:2:1} and \ref{assumption:2:2}, we use the same techniques as in the above to obtain
	\begin{align} \label{eq:avg:optimization:wtilde:4}
	&\Big\| \frac{1}{g} \sum_{i=1}^g \A_i  \Big\|_F
	+ \Big\| \frac{1}{g} \sum_{i=1}^g \B_i  \Big\|_F
	\; \leq \;
	\sqrt{ 2\Big\| \frac{1}{g} \sum_{i=1}^g \A_i  \Big\|_F^2
		+ 2\Big\| \frac{1}{g} \sum_{i=1}^g \B_i  \Big\|_F^2 } \nonumber \\
	& \leq \;
	{\max \Big\{ \tfrac{\sqrt{\epsilon }}{ \sqrt{g} } , \tfrac{\eta }{ \sqrt{g} } \Big\}} \,
	\sqrt{\tfrac{ \sigma_{\max}^2}{\sigma_{\max}^2 + n \gamma }}
	\sqrt{2 n \,f(\W^\star )}
	\; = \;
	{\max \big\{ { \sqrt{\epsilon }} ,  {\eta }\big\}} \,
	\tfrac{\sqrt{\beta }}{ \sqrt{g} }
	\sqrt{2 n \,f(\W^\star )} .
	\end{align}
	It follows from \eqref{eq:avg:optimization:wtilde:2}, \eqref{eq:avg:optimization:wtilde:3},
	and \eqref{eq:avg:optimization:wtilde:4} that
	\begin{align*}
	& \Big\| (\X^T \X + n \gamma \I_d)^{1/2}  ({\W}^{\textrm{c}}_i - \W^\star ) \Big\|_F \nonumber \\
	& \leq \; \Big\| \frac{1}{g} \sum_{i=1}^g \A_i  \Big\|_F
	+ \Big\| \frac{1}{g} \sum_{i=1}^g \B_i  \Big\|_F
	+  2\eta \beta \,
	\frac{1}{g} \sum_{i=1}^g \Big( \big\| \A_i \big\|_F +  \big\| \B_i \big\|_F \Big) \\
	& \leq \; \Big[ \tfrac{1}{\sqrt{g} } \max \big\{ \sqrt{{\epsilon}} , {\eta } \big\}
	+ 2\beta \eta
	\cdot \max \big\{  \sqrt{\epsilon}, \eta \big\} \Big] \,
	\sqrt{\beta}
	\sqrt{2 n \, f(\W^\star )} \\
	& = \;  \max \big\{ \sqrt{\epsilon } , \eta \big\}
	\, \cdot \, \big( \tfrac{1}{\sqrt{g}} + 2 \beta \eta \big) 
	\,\sqrt{\beta}
	\sqrt{2 n \, f(\W^\star )} 
	\\
	& = \sqrt{ \alpha \beta n \, f(\W^\star ) }.
	\end{align*}
	This concludes our proof.
\end{proof}

\subsection{Proof of Theorem~\ref{thm:optimization:what_avg}}
\label{sec:opt_avg:proof:hessian}

\begin{proof}
	By Lemma~\ref{lem:optimization},
  we only need to show that
	$ \big\| (\X^T \X + n \gamma \I_d)^{1/2}  \big({\W}^{\textrm{h}} - \W^\star \big) \big\|_F^2
	\leq \alpha^2 \beta^2 \big( -n  f (\W^\star ) + \|\Y\|_F^2 \big)$.
	
	In the proof of Theorem~\ref{thm:optimization:hessian} we show that
	\begin{align*}
	& (\X^T \X + n \gamma \I_d)^{1/2} \big({\W}^{\textrm{h}}_i - \W^\star \big)
	\; = \; \V \A_i \B_i \C ,
	\end{align*}
	where
	\begin{eqnarray*}
		\A_i & = & (\Si^2 + n \gamma \I_\rho)^{-1/2} \Si (\I_\rho - \U^T \S_i \S_i^T \U)
		\Si (\Si^2 + n \gamma \I_\rho)^{-1/2} ,\\
		\B_i & = & (\Si^2 + n \gamma \I_\rho)^{1/2}
		(\Si \U^T \S_i \S_i^T \U \Si + n \gamma \I_\rho)^{-1} (\Si^2 + n \gamma \I_\rho)^{1/2} , \\
		\C & = & (\Si^2 + n \gamma \I_\rho)^{-1/2} \Si \U^T \Y .
	\end{eqnarray*}
	It follows from Assumption~\ref{assumption:2:1} that for all $i\in [g]$,
	\[
	\tfrac{1}{1+\eta } (\Si^2 + n \gamma \I_\rho)^{-1} 
	\; \preceq \;
	(\Si \U^T \S_i \S_i^T \U \Si + n \gamma \I_\rho)^{-1}
	\; \preceq \; \tfrac{1}{1-\eta } (\Si^2 + n \gamma \I_\rho)^{-1} .
	\]
	We let $\B_i = \I_\rho + \De_i$.
	Thus $-\frac{\eta}{1+\eta} \I_\rho  \preceq \De_i \preceq \frac{\eta}{1-\eta} \I_\rho $.
	It follows that
	\begin{align*}
	& (\X^T \X + n \gamma \I_d)^{1/2} \big({\W}^{\textrm{h}} - \W^\star \big)
	\; = \;
	\frac{1}{g} \sum_{i=1}^g (\X^T \X + n \gamma \I_d)^{1/2} \big({\W}^{\textrm{h}}_i - \W^\star \big) \nonumber \\
	& = \; \frac{1}{g} \sum_{i=1}^g \V \A_i (\I_\rho + \De_i ) \C
	\; = \; \frac{1}{g} \sum_{i=1}^g \V \A_i  \C
	+ \frac{1}{g} \sum_{i=1}^g \V \A_i \De_i \C  .
	\end{align*}
	It follows that
	\begin{align} \label{eq:avg:optimization:what:1}
	& \Big\| (\X^T \X + n \gamma \I_d)^{1/2} \big({\W}^{\textrm{h}} - \W^\star \big)  \Big\|_F
	\; \leq \; \Big\| \frac{1}{g} \sum_{i=1}^g \A_i \Big\|_2 \, \Big\| \C  \Big\|_F
	+ \frac{1}{g} \sum_{i=1}^g \big\| \A_i \big\|_2 \big\| \De_i \big\|_2 \big\| \C \big\|_F \nonumber \\
	& \leq  \; \Big\| \frac{1}{g} \sum_{i=1}^g \A_i \Big\|_2 \, \Big\| \C  \Big\|_F
	+ \frac{\eta }{1-\eta } \Big(\frac{1}{g}  \sum_{i=1}^g \big\| \A_i \big\|_2 \Big) \big\| \C \big\|_F .
	\end{align}
	Let $\S = \frac{1}{g} [\S_1 , \cdots , \S_g] \in \RB^{n\times gs}$.
	It follows from the definition of $\A_i$ that
	\begin{eqnarray*}
		\big\| \A_i \big\|_2
		& = & \Big\| (\Si^2 + n \gamma \I_\rho)^{-1/2} \Si (\I_\rho - \U^T \S_i \S_i^T \U)
		\Si (\Si^2 + n \gamma \I_\rho)^{-1/2} \Big\|_2 \\
		& \leq & \eta \Big\| (\Si^2 + n \gamma \I_\rho)^{-1/2} \Si
		\Si (\Si^2 + n \gamma \I_\rho)^{-1/2} \Big\|_2
		\; = \; \eta \tfrac{ \sigma_{\max}^2}{\sigma_{\max}^2 + n \gamma}
		\; = \; \eta \beta ,\\
		\Big\| \frac{1}{g} \sum_{i=1}^g \A_i \Big\|_2
		& = & \Big\| (\Si^2 + n \gamma \I_\rho)^{-1/2} \Si (\I_\rho - \U^T \S \S^T \U)
		\Si (\Si^2 + n \gamma \I_\rho)^{-1/2} \Big\|_2 \\
		& \leq & \tfrac{\eta }{ \sqrt{g} } \big\| (\Si^2 + n \gamma \I_\rho)^{-1/2} \Si
		\Si (\Si^2 + n \gamma \I_\rho)^{-1/2} \big\|_2
		\; = \; \tfrac{\eta}{ \sqrt{g} } \tfrac{ \sigma_{\max}^2}{\sigma_{\max}^2 + n \gamma}
		\; = \; \tfrac{\eta \beta }{ \sqrt{g} }.
	\end{eqnarray*}
	It follows from \eqref{eq:avg:optimization:what:1} that
	\begin{align*}
	& \Big\| (\X^T \X + n \gamma \I_d)^{1/2} \big({\W}^{\textrm{h}} - \W^\star \big)  \Big\|_F  \\
	& \leq \; \Big( \tfrac{\eta}{ \sqrt{g} } + \tfrac{ \eta^2 }{1 - \eta } \Big)
	\beta \big\| \C \big\|_F \\
	& \leq \; \Big( \tfrac{\eta}{ \sqrt{g} } + \tfrac{ \eta^2 }{1 - \eta}  \Big)
	\beta \sqrt{ - n f (\W^\star ) + \| \Y \|_F^2 } ,
	\end{align*}
	where the latter inequality follows from the proof of Theorem~\ref{thm:optimization:what}.
	This concludes the proof.
\end{proof}


\section{Model Averaging from the Statistical Perspective: Proofs}
\label{sec:stat_avg:proof}

In Section~\ref{sec:stat_avg:proof:classical} we prove Theorem~\ref{thm:biasvariance:wtilde_avg}.
In Section~\ref{sec:stat_avg:proof:hessian} we prove Theorem~\ref{thm:biasvariance:what_avg}.


\subsection{Proof of Theorem~\ref{thm:biasvariance:wtilde_avg}}
\label{sec:stat_avg:proof:classical}

\begin{proof}
	The bound on $\bias \big( \W^{\textrm{c}} \big)$ can be shown in the same way as
	the proof of Theorem~\ref{thm:biasvariance:wtilde}.
	
	We prove the bound on $\var \big( \W^{\textrm{c}} \big)$ in the following.
	It follows from Assumption~\ref{assumption:2:1} that
	\begin{align*}
	(1+\eta )^{-1} ( \I_\rho + n\gamma \Si^{-2} )^{-1}
	\; \preceq \; (\U^T \S_i \S_i^T \U + n\gamma \Si^{-2} )^{\dag}
	\; \preceq \; (1-\eta )^{-1} ( \I_\rho + n\gamma \Si^{-2} )^{-1}.
	\end{align*}
	Let
	\[
	(\U^T \S_i \S_i^T \U + n\gamma \Si^{-2} )^{\dag}
	= ( \I_\rho + n\gamma \Si^{-2} )^{-1/2} (\I_\rho + \De_i ) ( \I_\rho + n\gamma \Si^{-2} )^{-1/2} .
	\]
	It holds that
	\[
	-\frac{\eta}{1+\eta} \I_\rho
	\; \preceq \; \De_i
	\; \preceq \; \frac{\eta}{1-\eta} \I_\rho .
	\]
	By the definition of $\var ( {\W}^{\textrm{c}} )$ in Theorem~\ref{thm:bias_var_decomp_avg},
	we have that
	\begin{small}
		\begin{align*}
		& \sqrt{\var \big( \W^{\textrm{c}} \big)}\\
		& = \; \frac{\xi}{\sqrt{n}}
		\bigg\| \frac{1}{g} \sum_{i=1}^g ( \I_\rho + n\gamma \Si^{-2} )^{-1} \U^T \S_i \S_i^T
		+ \frac{1}{g} \sum_{i=1}^g ( \I_\rho + n\gamma \Si^{-2} )^{-1/2} \De_i
		( \I_\rho + n\gamma \Si^{-2} )^{-1/2} \U^T \S_i \S_i^T  \bigg\|_F \\
		& \leq \; \frac{\xi}{\sqrt{n}}  \Big( \Big\| ( \I_\rho + n\gamma \Si^{-2} )^{-1} \U^T \S \S^T \Big\|_F
		+ \frac{1}{g} \sum_{i=1}^g \Big\| ( \I_\rho + n\gamma \Si^{-2} )^{-1/2} \De_i
		( \I_\rho + n\gamma \Si^{-2} )^{-1/2}  \U^T \S_i \S_i^T  \Big\|_F  \Big) \\
		& \leq \; \frac{\xi }{\sqrt{n}}  \big\|( \I_\rho + n\gamma \Si^{-2} )^{-1}   \big\|_F
		\Big(\big\| \U^T \S \big\|_2 \big\| \S \big\|_2
		+ \frac{ \eta }{1- \eta}
		\frac{1}{g} \sum_{i=1}^g \big\| \U^T \S_i \big\|_2
		\big\| \S_i \big\|_2 \Big) \\
		& = \; \sqrt{\var \big( \W^\star \big)}
		\Big(\big\| \U^T \S \big\|_2 \big\| \S \big\|_2
		+ \frac{ \eta }{1- \eta}
		\frac{1}{g} \sum_{i=1}^g \big\| \U^T \S_i \big\|_2
		\big\| \S_i \big\|_2 \Big) .
		\end{align*}
	\end{small}%
	Under Assumption~\ref{assumption:2:1}, we have that
	$\| \S_i^T \U \|_2^2 \leq 1+\eta$
	and $\| \S^T \U \|_2^2 \leq 1+ \frac{\eta}{  \sqrt{g}  }$.
	It follows that
	\begin{small}
		\begin{align*}
		& \sqrt{\frac{\var \big( \W^{\textrm{c}} \big)}{\var \big( \W^\star \big)}  }
		\; \leq \; \sqrt{1+\frac{\eta}{ \sqrt{g} }} \big\| \S \big\|_2
		+ \frac{ \eta \sqrt{1+\eta}}{1- \eta} \frac{1}{g} \sum_{i=1}^g \big\| \S_i \big\|_2 .
		\end{align*}
	\end{small}%
  Now the desired result follows from Assumption~\ref{assumption:2:3}.
\end{proof}


\subsection{Proof of Theorem~\ref{thm:biasvariance:what_avg}}
\label{sec:stat_avg:proof:hessian}

\begin{proof}
	The bound on $\var \big( {\W}^{\textrm{h}} \big)$ can be established in the
	same way as Theorem~\ref{thm:biasvariance:what}.
	
	We prove the bound on $\bias \big( {\W}^{\textrm{h}} \big)$ in the following.
	Let
	\[
	(\U^T \S_i \S_i^T \U + n\gamma \Si^{-2} )^{\dag}
	= ( \I_\rho + n\gamma \Si^{-2} )^{-1/2} (\I_\rho + \De_i ) ( \I_\rho + n\gamma \Si^{-2} )^{-1/2} .
	\]
	Under Assumption~\ref{assumption:2:1}, we have that
	$
	\De_i \preceq \frac{\eta}{1-\eta} \I_\rho
	$.
	It follows from Theorem~\ref{thm:bias_var_decomp_avg} that
	\begin{small}
		\begin{align*}
		& \bias \big({\W}^{\textrm{h}}\big)
		\; = \; \gamma \sqrt{n} \bigg\| \frac{1}{g} \sum_{i=1}^g
		\Big(  \Si^{-2} + \tfrac{\U^T \S_i \S_i^T \U - \I_\rho}{n \gamma} \Big)
		(\U^T \S_i \S_i^T \U + n\gamma \Si^{-2} )^{\dag}  \Si \V^T \W_0 \bigg\|_F  \\
		& \leq  \; \gamma \sqrt{n} \bigg\| \frac{1}{g} \sum_{i=1}^g
		\Big(  \Si^{-2} + \tfrac{\U^T \S_i \S_i^T \U - \I_\rho}{n \gamma} \Big)
		( \I_\rho + n\gamma \Si^{-2} )^{-1}  \Si \V^T \W_0 \bigg\|_F  \\
		& \; + \gamma \sqrt{n} \bigg\| \frac{1}{g} \sum_{i=1}^g
		\Big(  \Si^{-2} + \tfrac{\U^T \S_i \S_i^T \U - \I_\rho}{n \gamma} \Big)
		( \I_\rho + n\gamma \Si^{-2} )^{-1/2} \De_i
		( \I_\rho + n\gamma \Si^{-2} )^{-1/2}  \Si \V^T \W_0 \bigg\|_F  \\
		& \triangleq \; \gamma \sqrt{n}  \big( A + B \big) ,
		\end{align*}
	\end{small}%
	where
	\begin{small}
		\begin{eqnarray*}
			A & = & \Big\| \frac{1}{g} \sum_{i=1}^g
			\Big(  \Si^{-2} + \tfrac{\U^T \S_i \S_i^T \U - \I_\rho}{n \gamma} \Big)
			\big( \I_\rho + n\gamma \Si^{-2} \big)^{-1} \Si \V^T \W_0 \Big\|_F \\
			& = & \Big\|  \Big(  \Si^{-2} + \tfrac{\U^T \S \S^T \U - \I_\rho}{n \gamma} \Big)
			\big( \I_\rho + n\gamma \Si^{-2} \big)^{-1} \Si \V^T \W_0 \Big\|_F , \\
			B & = &\Big\| \frac{1}{g} \sum_{i=1}^g
			\Big(  \Si^{-2} + \tfrac{\U^T \S_i \S_i^T \U - \I_\rho}{n \gamma} \Big)
			\big( \I_\rho + n\gamma \Si^{-2} \big)^{-1/2}
			\De_i  \big( \I_\rho + n\gamma \Si^{-2} \big)^{-1/2} \Si \V^T \W_0 \Big\|_F \\
			& \leq & \frac{1}{g} \sum_{i=1}^g \Big\|
			\Big(  \Si^{-2} + \tfrac{\U^T \S_i \S_i^T \U - \I_\rho}{n \gamma} \Big)
			\big( \I_\rho + n\gamma \Si^{-2} \big)^{-1/2}
			\De_i  \big( \I_\rho + n\gamma \Si^{-2} \big)^{-1/2} \Si \V^T \W_0 \Big\|_F .
		\end{eqnarray*}
	\end{small}%
	It follows from Assumption~\ref{assumption:2:1} that $\U^T \S \S^T \U - \I_\rho$ is semidefinitely bounded between $\pm \frac{\eta}{\sqrt{g}} \I_\rho $. Thus
	\[
	\Big(1 - \tfrac{\eta \sigma_{\max}^2 }{n\gamma \sqrt{g} } \Big) \Si^{-2}
	\; \preceq \;
	\Si^{-2} + \tfrac{\U^T \S \S^T \U - \I_\rho}{n \gamma}
	\; \preceq \;
	\Big(1+\tfrac{\eta \sigma_{\max}^2 }{n\gamma \sqrt{g} } \Big) \Si^{-2} .
	\]
	It follows that
	\begin{eqnarray*}
		A
		& = & \Big\|  \Big(  \Si^{-2} + \tfrac{\U^T \S \S^T \U - \I_\rho}{n \gamma} \Big)
		\big( \I_\rho + n\gamma \Si^{-2} \big)^{-1} \Si \V^T \W_0 \Big\|_F \\
		& \leq & \Big(1+\tfrac{\eta \sigma_{\max}^2 }{n\gamma \sqrt{g} } \Big)
		\Big\| \big( \Si^2 + n\gamma \I_\rho \big)^{-1} \Si \V^T \W_0 \Big\|_F .
	\end{eqnarray*}
	Similar to the proof of Theorem~\ref{thm:biasvariance:what}, we can show that
	\begin{eqnarray*}
		B & \leq & \Big(1+\frac{\eta \sigma_{\max}^2 }{n\gamma} \Big)  \cdot
		\frac{1}{g} \sum_{i=1}^g \Big\|
		\Si^{-2} \big( \I_\rho + n\gamma \Si^{-2} \big)^{-1/2}
		\De_i  \big( \I_\rho + n\gamma \Si^{-2} \big)^{-1/2} \Si \V^T \W_0 \Big\|_F \\
		& \leq &
		\frac{ \eta }{1-\eta}
		\Big(1+\frac{\eta \sigma_{\max}^2 }{n\gamma} \Big)  \cdot
		\Big\| \big( \Si^2 + n\gamma \I_\rho \big)^{-1} \Si \V^T \W_0 \Big\|_F .
	\end{eqnarray*}
	Hence
	\begin{align*}
	&\bias \big({\W}^{\textrm{h}}\big)
	\; \leq \; \gamma \sqrt{n}  \big( A + B \big) \\
	& \leq \; \Big[ \tfrac{ 1 }{1-\eta} + \Big( \tfrac{\eta}{ \sqrt{g} } + \tfrac{ \eta^2 }{1-\eta } \Big)
	\tfrac{\sigma_{\max}^2 }{n \gamma } \Big] \:
	\gamma \sqrt{n} \Big\| \big( \Si^2 + n\gamma \I_\rho \big)^{-1} \Si \V^T \W_0 \Big\|_F
	\\
	& = \; \Big[ \tfrac{ 1 }{1-\eta} + \Big( \tfrac{\eta}{ \sqrt{g} } + \tfrac{ \eta^2 }{1-\eta } \Big)
	\tfrac{\sigma_{\max}^2 }{n \gamma } \Big]
	\, \bias \big(\W^\star \big).
	\end{align*}
	Here the equality follows from Theorem~\ref{thm:bias_var_decomp}.
\end{proof}